\documentclass[twoside]{article}

\usepackage[accepted]{aistats2025}

\usepackage{amsfonts}
\usepackage{amssymb}
\usepackage{amsthm}

\usepackage{nicefrac}
\usepackage{hyperref}
\usepackage{graphicx}
\usepackage{xcolor}

\usepackage{algpseudocode}
\usepackage{algorithm}
\usepackage{algorithmicx}
\usepackage{subcaption}
\usepackage{caption}
\usepackage{placeins}
\usepackage{afterpage}
\usepackage{dblfloatfix}

\newtheorem{theorem}{Theorem}[section]
\newtheorem{Proposition}[theorem]{Proposition}
\newtheorem{lemma}[theorem]{Lemma}
\newtheorem{corollary}[theorem]{Corollary}
\newtheorem{assumption}{Assumption}
\theoremstyle{definition}
\newtheorem{definition}[theorem]{Definition}

\usepackage{enumitem}
\usepackage{stmaryrd}
\usepackage[round]{natbib}

\definecolor{nblue}{RGB}{28,130,185}
\definecolor{nred}{RGB}{203, 0, 68} 
\definecolor{ngreen}{RGB}{0,150,0}

\usepackage[toc,page]{appendix}
\usepackage{fancyhdr}
\usepackage{nicefrac}
\usepackage{booktabs}

\usepackage{tikz}
\usetikzlibrary{arrows.meta, positioning, calc, decorations.pathreplacing}

\usepackage{algorithm}
\usepackage{algpseudocode}

\newcommand{\eqsp}{\;}
\newcommand{\rmd}{\mathrm{d}}
\newcommand{\rset}{\Rset}
\newcommand{\E}{\mathbb{E}}
\newcommand{\Xset}{\mathsf{X}}
\newcommand{\Rset}{\mathbb{R}}
\newcommand{\Zset}{\mathsf{Z}}
\newcommand{\rme}{\mathrm{e}}
\newcommand{\Xsigma}{\mathcal{X}}
\newcommand{\Fb}{\mathcal{F}_{\text{b}}}
\newcommand{\NN}{\mathsf{NN}}
\newcommand{\idx}{m}
\newcommand{\PW}{\mathsf{P}} 
\newcommand{\SF}{\mathsf{S}}

\usepackage{minitoc}

\DeclareUnicodeCharacter{03B2}{beta}

\begin{document}

\twocolumn[

\aistatstitle{Theoretical Convergence Guarantees for Variational Autoencoders}

\aistatsauthor{ Sobihan Surendran$^{1,2}$ \And Antoine Godichon-Baggioni$^{1}$ \And  Sylvain Le Corff$^{1}$ } 
\aistatsaddress{ $^1$Sorbonne Université, CNRS, 
    Laboratoire de Probabilités, Statistique et Modélisation, Paris, France \\
    $^2$LOPF, Califrais’ Machine Learning Lab, Paris, France } ]

\begin{abstract}
Variational Autoencoders (VAE) are popular generative models used to sample from complex data distributions. Despite their empirical success in various machine learning tasks, significant gaps remain in understanding their theoretical properties, particularly regarding convergence guarantees. This paper aims to bridge that gap by providing non-asymptotic convergence guarantees for VAE trained using both Stochastic Gradient Descent and Adam algorithms.
We derive a convergence rate of $\mathcal{O}(\log n / \sqrt{n})$, where $n$ is the number of iterations of the optimization algorithm, with explicit dependencies on the batch size, the number of variational samples, and other key hyperparameters. Our theoretical analysis applies to both Linear VAE and Deep Gaussian VAE, as well as several VAE variants, including $\beta$-VAE and IWAE. Additionally, we empirically illustrate the impact of hyperparameters on convergence, offering new insights into the theoretical understanding of VAE training.
\end{abstract}

\section{INTRODUCTION}

Probabilistic inference in generative models is a long-standing problem in particular for complex and high dimensional distributions. Many solutions to estimate the likelihood of data or to infer parameters are based on Importance Sampling, Sequential Monte Carlo (SMC) or Markov Chain Monte Carlo (MCMC)-based algorithms.
Although their convergence properties have been analyzed in numerous research works, they face convergence problems and lead to slow training procedures in high-dimensional settings. Recently, many advances have been proposed in generative modeling, leading to the emergence of novel approaches such as Variational Autoencoders (VAE) \citep{kingma2013auto, rezende2014stochastic, ranganath2014black}, Generative Adversarial Networks (GAN) \citep{goodfellow2014generative}, and Probabilistic and Score-Based Diffusion Models \citep{ho2020denoising,song2021scorebased}. These models offer alternative training paradigms that address many of the limitations associated with MCMC-based approaches.
Among these, VAE are notable for their ability to learn low dimensional latent variables,  while maintaining a clear probabilistic interpretation, enabling both scalable training and meaningful generative modeling.

VAE have been successfully applied in many different contexts such as text generation \citep{bowman2015generating}, image generation \citep{vahdat2020nvae}, image segmentation \citep{kohl2018probabilistic}, representation learning \citep{chen2016variational}, music generation \citep{roberts2018hierarchical}, dimensionality reduction \citep{kaur2021variational}, anomaly detection \citep{park2022interpreting} and state estimation and image reconstruction \citep{cohen2022diffusion}. 
There are several popular variants of VAE, including IWAE \citep{burda2015importance}, $\beta$-VAE \citep{higgins2017beta}, VQ-VAE \citep{van2017neural}, and Conditional VAE \citep{sohn2015learning}.

Theoretical properties of Variational Inference have only recently been analyzed. For instance, \cite{cherief2022pac} provides generalization bounds on the theoretical reconstruction error using PAC-Bayes theory, while \cite{huggins2020validated} offers variational error bounds for posterior mean and uncertainty estimates. 
\cite{mbacke2024statistical} extend these analyses by offering statistical guarantees for reconstruction, generation, and regeneration. 
Furthermore, \cite{tang2021empirical} highlights the importance of covariance matrices in Gaussian encoders and decoders, providing variational excess risk bounds for Empirical Bayes Variational Autoencoders (EBVAE) applied to Gaussian models. The authors provide an oracle
result for the EBVAE estimator which is crucial to prove the consistency of the estimator. 
Additionally, theoretical insights into posterior collapse in VAE are provided by \citep{razavi2019preventing, lucas2019don, wang2022posterior}, particularly in the context of Linear VAE. More recently, \cite{domke2024provable, kim2024linear, kim2024convergence} established convergence results for Black-Box Variational Inference (BBVI) under location-scale parameterization.

Most theoretical guarantees for Variational Inference procedures have been established for independent data. However, for sequential data, such as time series, Sequential VAE have been developed. Various model architectures for Sequential VAE have been proposed, including those by \citep{chung2015recurrent, fraccaro2016sequential, marino2018general, kim2020variational, campbell2021online, bayer2021mind}.
However, there are relatively few theoretical results dedicated to dependent or structured data. Notably, \citep{chagneux2024additive} and \citep{gassiat2024variational} derive variational excess risk bounds for general state space models.

Therefore, obtaining convergence guarantees and non-asymptotic convergence rates for VAE remains mainly an open problem. 
In this paper, we address this gap by providing non-asymptotic convergence guarantees for VAE, applicable to both independent and sequential data, when using standard Stochastic Gradient Descent (SGD) \citep{robbins1951stochastic} and Adam \citep{kingma2014adam} optimization algorithms. More precisely, our contributions are summarized as follows.
\begin{itemize}
    \item We first establish the smoothness of the expected
    ELBO which is a crucial step to obtain a convergence rate of $\mathcal{O}(\log n / \sqrt{n})$, where $n$ is the number of iterations of the optimization algorithm, with explicit dependency on the batch size $B$ and the number of variational samples $K$ used at each gradient step (Theorem~\ref{th:rate_general_SGD}).
    \item We demonstrate that our results apply to both Linear VAE (Section~\ref{sec:linearVAE}) and Deep Gaussian VAE (Section~\ref{sec:deepVAE}), as well as to several variants, including $\beta$-VAE, IWAE, and Sequential VAE.
    \item We extend our analysis to BBVI (Section~\ref{sec:BBVI}), deriving new convergence guarantees under weaker assumptions, notably without requiring location-scale parameterization.
    \item We illustrate our convergence results by analyzing empirically the impact of hyperparameters, particularly the regularization parameter $\beta$ in $\beta$-VAE, the number of variational samples in IWAE, as well as the number of layers and activation functions.
\end{itemize}

\section{NOTATION AND BACKGROUND}

\paragraph{Notations.} In the following, for all distribution $\mu$ (resp. probability density $p$) we write $\E_\mu$ (resp. $\E_p$) the expectation under $\mu$ (resp. under $p$).
Given a measurable space $(\Xset, \Xsigma)$, where $\Xsigma$ is a countably generated $\sigma$-algebra, let $\mathsf{M}(\Xset)$ denote the set of all measurable functions defined on $(\Xset, \Xsigma)$.
We define $\mathcal{F}_{\text{SL}}$ as the set of measurable functions that are both smooth and Lipschitz continuous, and $\Fb$ as the set of bounded measurable functions.
The Hadamard product of vectors $u$ and $v$ is denoted by $u \odot v$. For $d \geq 1$, let $\mathrm{I}_{d}$ denote the $d \times d$ identity matrix.
The probability density function of the normal distribution with mean $\mu$ and covariance matrix $\Sigma$ is denoted by $\mathcal{N}(\cdot;\mu, \Sigma)$.
For probability measures $P$ and $Q$ defined on the same probability space, the Kullback-Leibler (KL) divergence is defined as $\text{D}_{\text{KL}}(Q \| P)=\mathbb{E}_Q[\log(\rmd Q/\rmd P)]$.
Given a fully connected neural network with $N$ layers, input $z$, parameters $\theta = \{(W_i, b_i)\}_{i=1}^N$, and activation functions $f = \{f_i\}_{i=1}^N$, the output of the neural network is written:
$$ \NN(z; \theta, f, N) = f_{N} \left( \cdots f_1 \left( W_1 z + b_1 \right) \cdots \right) \eqsp.$$
We define $\| \theta \|_{\infty} := \max \{ \max_{i} \| W_i \|, \max_{i} \| b_i \| \}$, where $\|\cdot\|$ represents the Euclidean norm for vectors and the spectral norm for matrices.

\paragraph{Background. }
Let $\Xset \subseteq \Rset^{d_x}$ and $\Zset \subseteq \Rset^{d_z}$ denote the data space and the latent space, respectively. We consider a dataset $\mathcal{D} = \{x_1, \ldots, x_B\}$ of independent copies of a random variable $x \in \Xset$ sampled from an unknown probability distribution $\pi$. 
In generative models, particularly those involving latent variables, a classical objective is to maximize the marginal likelihood of the observed data. This marginal likelihood for a given observation $x$ is typically expressed as:
$$
\log p_{\theta}(x)  = \log \mathbb{E}_{p_{\theta}(\cdot|x)} \left[ \frac{p_{\theta}(x, Z)}{p_{\theta}(Z | x)} \right]\eqsp,
$$
where $(x,z)\mapsto p_{\theta}(x, z)$ is the joint likelihood of the observation $x \in \Xset$ and the latent variable $z \in \Zset$.
The model is composed of a conditional likelihood $(x,z)\mapsto p_\theta(x | z)$ from a parametric family indexed by $\theta \in \Theta \subseteq \Rset^{d_{\theta}}$ (e.g., the weights of a neural network) and of a prior $z\mapsto p_\theta(z)$ (e.g., a standard Gaussian density).
Under some simple technical assumptions, by Fisher's identity, we have:
\begin{equation} \label{grad_log}
\nabla_{\theta} \log p_{\theta}(x)=\int \nabla_{\theta} \log p_{\theta}(x, z) p_{\theta}(z | x) \mathrm{d} z\eqsp.
\end{equation}
However, in most cases, the conditional density $z\mapsto p_{\theta}(z | x)$ can only be sampled from approximately using Markov Chain or Sequential Monte Carlo methods, see for instance \citep{neal1993probabilistic, andrieu2010particle,thin2021neo,neklyudov2022orbital}.  Variational Autoencoders introduce an additional parameter $\phi \in \Phi \subseteq \Rset^{d_{\phi}}$ and a family of variational distributions $(x,z)\mapsto q_{\phi}(z | x)$ to approximate the posterior distribution.
Parameters are estimated by maximizing the Evidence Lower Bound (ELBO):
\begin{equation*}
\log p_{\theta}(x) 
\geq \mathbb{E}_{q_{\phi}(\cdot|x)} \left[ \log \frac{p_{\theta}(x, Z)}{q_{\phi}(Z | x)} \right] =: \mathcal{L}(\theta, \phi; x)\eqsp.
\end{equation*}
The ELBO can be further rewritten as:
\begin{equation*}
\mathcal{L}(\theta, \phi; x) = \mathbb{E}_{q_{\phi}(\cdot|x)} \left[ \log p_{\theta}(x | Z) \right] - \text{D}_{\text{KL}}\left( q_{\phi}(\cdot | x) \| p \right)\eqsp.
\end{equation*} 

The ELBO consists of two terms: $(i)$ the reconstruction term, which quantifies the capability to accurately reconstruct the original input data from its latent representation, and $(ii)$  the regularization term, expressed as the KL divergence, which encourages the latent space of the VAE to follow the prior distribution.

\section{THEORETICAL PROPERTIES OF VAE}

The expected Evidence Lower Bound over the data distribution $\pi$ is defined as follows:
\begin{equation*}
\mathcal{L}(\theta, \phi) = \mathbb{E}_{\pi, \phi} \left[ \log \frac{p_{\theta}(X, Z)}{q_{\phi}(Z | X)} \right] = \mathbb{E}_{\pi} \left[ \mathcal{L}(\theta, \phi; X) \right]\eqsp,
\end{equation*}
where $\mathbb{E}_{\pi, \phi}$ denotes the expectation under the distribution $\pi(\rmd x) q_\phi(\rmd z|x)$.
In order to optimize $\theta$ and $\phi$, one needs to compute the gradients of the ELBO with respect to these parameters. Under classical regularity assumptions, the gradient with respect to $\theta$ is given by:
$$
\nabla_{\theta} \mathcal{L}(\theta, \phi) = \mathbb{E}_{\pi, \phi} \left[ \nabla_{\theta} \log p_{\theta}(X, Z)\right]\eqsp.
$$
Computing the gradient with respect to the variational parameters $\phi$ is more challenging since the inner expectation depends on  $q_{\phi}$. There are two common methods for computing this gradient. 

\paragraph{The Pathwise Gradient. }
 
The reparametrization trick involves expressing the random variable $z$ as a deterministic transform $z = g(\varepsilon, \phi)$, where $\varepsilon$ is an auxiliary independent random variable drawn from a known distribution $p_\varepsilon$.
Using this trick, the ELBO can be expressed as:
$$ \mathcal{L}(\theta, \phi; x) = \mathbb{E}_{p_\varepsilon}\left[ \log w_{\theta, \phi}(x, g(\varepsilon,\phi)) \right]\eqsp,$$
where $w_{\theta, \phi}(x,z) = p_{\theta}(x, z)/q_{\phi}(z|x)$ the unnormalized importance weights and $\mathbb{E}_{p_\varepsilon}$ is the expectation under the law of $\varepsilon$ when $\varepsilon \sim p_\varepsilon$.
The pathwise gradient \citep{kingma2013auto, rezende2014stochastic} of the ELBO is given by:
\begin{align*}
\nabla^{\PW}_{\phi} \mathcal{L}(\theta, \phi; x) &=  \mathbb{E}_{p_\varepsilon}\left[ \nabla_z \log w_{\theta, \phi}(x,z) \nabla_{\phi} g(\varepsilon, \phi) \right] \\
&\quad - \mathbb{E}_{p_\varepsilon}\left[\nabla_{\phi} \log q_{\phi}(g(\varepsilon, \phi) | x) \right]\eqsp.
\end{align*}
\paragraph{The Score Function Gradient. }
Alternatively, the score function gradient, also known as the Reinforce gradient \citep{glynn1990likelihood, williams1992simple,paisley2012variational}, can be utilized. Unlike the reparameterization trick, this method does not necessitate reparameterization and is applicable to a wider range of variational distributions. 
Proposition~\ref{prop:score_estimator} provides the form of the score function gradient with respect to $\phi$:
$$\nabla^{\SF}_{\phi} \mathcal{L}(\theta, \phi) = \mathbb{E}_{\pi, \phi}\left[ \log \frac{p_{\theta}(X, Z)}{q_{\phi}(Z | X)} \nabla_{\phi} \log q_{\phi}(Z | X) \right]\eqsp.$$
The gradient estimator of the ELBO for a given batch of observations $\{x_i\}_{i=1}^B$, where $B$ is the batch size, with respect to $\theta$ and $\phi$ can then be computed using Monte Carlo sampling as follows:
\begin{align*} \label{eq:grad_estimator}
\widehat{\nabla}_{\theta, \phi} \mathcal{L}(\theta, \phi; \{x_i\}_{i=1}^B) = \frac{1}{B} \sum_{i=1}^{B} \frac{1}{K} \sum_{\ell=1}^{K} \tilde{g}_{i,\ell}\eqsp,
\end{align*}
where $\tilde{g}_{i,\ell}$ is either the gradient estimator via the score function $\tilde{g}^{\SF}_{i,\ell}$ as defined in \eqref{eq:grad_estimator_SF} or the pathwise gradient estimator $\tilde{g}^{\PW}_{i,\ell}$ as described in \eqref{eq:grad_estimator_PW}.

The pathwise gradient estimator often yields lower-variance estimates than the score function estimator \citep{miller2017reducing, buchholz2018quasi}, but its variance can sometimes exceed that of the score function estimator, especially when the score function correlates with other components of the pathwise estimator.
Several methods have been proposed to further reduce variance, such as the Rao-Blackwellization estimator \citep{ranganath2014black}, Control Variates \citep{lievin2020optimal}, Stop Gradient estimator \citep{roeder2017sticking}, Quasi-Monte Carlo VAE \citep{buchholz2018quasi}, and Multi-Level Monte Carlo estimator \citep{fujisawa2021multilevel, he2022unbiased}. While our analysis focuses on the convergence rate of score function and pathwise gradient estimators, our convergence results also apply to most of these other methods.

\subsection{Convergence Analysis of VAE in the General Setting}

In this section, we derive the convergence rates of VAE with both the score function estimator and the pathwise gradient estimator.
SGD is a widely used method for training statistical models based on deep architectures.
It produces a sequence of parameter estimates as follows: $\left(\theta_{0}, \phi_{0}\right) \in \Theta \times \Phi$ and for all $k \in \mathbb{N}$,
\begin{equation} \label{eq:SGD}
\left(\theta_{k+1}, \phi_{k+1}\right) = \left(\theta_{k}, \phi_{k}\right) + \gamma_{k+1} \widehat{\nabla}_{\theta, \phi} \mathcal{L}(\theta_{k}, \phi_{k}; \mathcal{D}_{k+1})\eqsp, 
\end{equation}
where $\widehat{\nabla}_{\theta, \phi} \mathcal{L}(\theta_{k}, \phi_{k}; \mathcal{D}_{k+1})$ denotes an estimator of the gradient, defined in \eqref{eq:grad_estimator}, $\mathcal{D}_{k+1}$ is the mini-batch of data used at iteration $k+1$ and for all $k\geq 1$, $\gamma_{k}>0$ is the learning rate.
In recent years, several adaptive methods have been proposed, which leverage past gradients to avoid saddle points and handle ill-conditioned problems. Popular adaptive methods include Adagrad \citep{duchi2011adaptive}, RMSProp \citep{tieleman2012lecture}, Adadelta \citep{zeiler2012adadelta}, and Adam \citep{kingma2014adam}.
Recent studies have highlighted the superior performance of Adam, a method that iteratively updates the parameters $\theta$ and $\phi$ to effectively maximize the ELBO, as detailed in Algorithm \ref{alg:adam}.

Consider the following assumptions.

\begin{assumption}\label{ass:A1}
There exists $\alpha \in \mathsf{M}(\Xset \times \Zset)$ such that for all $\theta \in \Theta$, $\phi \in \Phi$, $x \in \Xset$ and $z \in \Zset$,
$$
\max\{ \left| \log p_{\theta}(x , z) \right|, \left|\log q_{\phi}(z|x)\right| \} \leq \alpha(x,z)\eqsp.
$$
\end{assumption}
Assumption \ref{ass:A1} corresponds to bounding the logarithm of the joint probability density $p_{\theta}(x, z)$ and variational log density $q_{\phi}(z|x)$. In other words, the probability densities are bounded in log space, which in turn guarantees that the score remains bounded. Although this assumption is typically satisfied in models with a compact state space, compactness is not a necessary condition in this context. This assumption is analyzed in detail in Lemma \ref{lem:ELBO_A1_gauss} for Gaussian distributions.

\begin{assumption}\label{ass:A2}

\begin{itemize}[left=0pt, label={}, itemindent=5pt, itemsep=0pt]
\item $(i)$ Score Function: there exist $M$, $L_1$, and $L_2 \in \mathsf{M}(\Xset \times \Zset)$ such that for all $\theta, \theta' \in \Theta$, $\phi, \phi' \in \Phi$, $x \in \Xset$ and $z \in \Zset$,
$$
\max\{ \left\|\nabla_{\theta} \log p_{\theta}(x, z)\right\|, \left\| \nabla_{\phi} \log q_{\phi}(z|x)\right\| \} \leq M(x,z)\eqsp,
$$
$$
\left\|\nabla_{\theta} \log p_{\theta}(x,z) - \nabla_{\theta} \log p_{\theta'}(x,z)\right\| \leq L_1(x,z) \left\|\theta - \theta'\right\|,
$$
$$
\left\|\nabla_{\phi} \log q_{\phi}(z|x) - \nabla_{\phi} \log q_{\phi'}(z|x)\right\| \leq L_2(x,z) \left\|\phi - \phi'\right\|.
$$
    \item $(ii)$ Pathwise Gradient: there exist $M$, $L_p$, and $L_q \in \mathsf{M}(\Xset \times \Zset)$ such that for all $\theta, \theta' \in \Theta$, $\phi, \phi' \in \Phi$, $x \in \Xset$,  $\varepsilon \in \Zset$, writing $z = g(\varepsilon, \phi)$ and $z' = g(\varepsilon, \phi')$,
$$
\max\{ \left\|\nabla_{z, \theta} \log p_{\theta}(x, z)\right\|, \left\| \nabla_{z} \log q_{\phi}(z|x)\right\| \} \leq M(x,\varepsilon)\eqsp,
$$
\vspace{-3em}
\begin{multline*}
\left\|\nabla_{z, \theta} \log p_{\theta}(x, z) - \nabla_{z, \theta} \log p_{\theta'}(x, z')\right\| \\\leq L_{p}(x,\varepsilon) \left( \left\|\theta - \theta'\right\| +  \left\|z - z'\right\| \right)\eqsp,
\end{multline*}
\vspace{-3em}
\begin{multline*}
\left\|\nabla_{z} \log q_{\phi}(z|x) - \nabla_{z} \log q_{\phi'}(z'|x)\right\| \\
\leq L_{q}(x,\varepsilon) \left( \left\|\phi - \phi'\right\| +  \left\|z - z'\right\| \right)\eqsp.
\end{multline*}
\end{itemize}
\end{assumption}
Assumption \ref{ass:A2} is divided into two parts: (i) presents the regularity conditions needed for convergence when the gradient is computed via the score function, and (ii) specifies the assumptions required for convergence when using the pathwise gradient.
Assumption \ref{ass:A2}(i) concerns the boundedness and Lipschitz continuity of the score functions associated with the distributions $p_{\theta}(x, z)$ and $q_{\phi}(z|x)$. 
This assumption is standard in the literature on Reinforcement Learning (RL) and Maximum Likelihood Estimation (MLE). 
In RL, assumptions about the boundedness and Lipschitz continuity of the score functions of the policy are commonly used to prove the convergence rates of policy gradient algorithms \citep{papini2018stochastic, shen2019hessian, fallah2021convergence, liu2020improved} and actor-critic algorithms \citep{castro2010convergent, qiu2021finite}.
Importantly, the bounds on the score function in prior works are independent of state variables, requiring a compact state space. In contrast, our approach considers variable dependence, relaxing this assumption and allowing for unbounded state spaces.
Furthermore, we impose additional assumptions on the score functions associated with the variational distribution.
In MLE, these assumptions are used to establish the convergence of recursive MLE in Non-Linear State-Space Models \citep{tadic2020asymptotic}. However, these assumptions are typically formulated in the original space rather than log space.
Assumption \ref{ass:A2}(ii) is similar to Assumption \ref{ass:A2}(i), with the key difference being that it considers the gradient with respect to $z$ instead of $\phi$ for the variational log density. Additionally, it accounts for the joint gradient with respect to both $z$ and $\theta$ for the conditional decoder log density.

\begin{assumption}\label{ass:A2_pathwise}
There exist $M_{g}$ and $L_{g} \in \mathsf{M}(\Xset \times \Zset)$ such that for all $\theta, \theta' \in \Theta$, $\phi, \phi' \in \Phi$, $x \in \Xset$ and $\varepsilon \in \Zset$,
$$
\left\|\nabla_{\phi} g(\varepsilon, \phi)\right\| \leq M_{g}(x,\varepsilon)\eqsp,
$$
$$
\left\|\nabla_{\phi} g(\varepsilon, \phi) - \nabla_{\phi} g(\varepsilon, \phi')\right\| \leq L_{g}(x,\varepsilon) \left\|\phi - \phi'\right\|\eqsp.
$$
\end{assumption}

Assumption \ref{ass:A2_pathwise} concerns the boundedness and smoothness of the reparameterization trick function $g$. It is important to note that under the boundedness of the gradients, the smoothness of $g$ and the Lipchitz condition of $\nabla_{z} \log q_{\phi}(z|x)$ are equivalent to the Lipchitz condition of the score function associated with the variational density.
We show that these assumptions hold in both Linear and Deep Gaussian VAE.
Under these assumptions, we first establish the smoothness of the expected ELBO, in both the score function and pathwise gradient cases, which is a critical step to prove the convergence rate.

\begin{Proposition} \label{prop:ELBO_smooth_score_pathwise}
For all $\theta, \theta' \in \Theta$ and $\phi, \phi' \in \Phi$,
\begin{itemize}[left=0pt, label={}, itemindent=5pt, itemsep=0pt]
    \item $(i)$ Score Function: under Assumptions \ref{ass:A1} and \ref{ass:A2}(i),
        $$
        \begin{aligned}
        \left\| \nabla^{\SF}_{\theta, \phi} \mathcal{L}(\theta, \phi) - \nabla^{\SF}_{\theta, \phi} \mathcal{L}(\theta', \phi') \right\| &\leq L^{\SF} \left\| \left(\theta, \phi\right) - \left(\theta', \phi'\right) \right\|,
        \end{aligned}
        $$

    \item $(ii)$ Pathwise Gradient: under Assumptions \ref{ass:A2}(ii) and \ref{ass:A2_pathwise},
    $$
    \begin{aligned}
    \| \nabla^{\PW}_{\theta, \phi} \mathcal{L}(\theta, \phi) - \nabla^{\PW}_{\theta, \phi} \mathcal{L}(\theta', \phi') \| &\leq L^{\PW} \| (\theta, \phi) - (\theta', \phi') \|,
    \end{aligned}
    $$
\end{itemize}
    where $L^{\SF}$ and $L^{\PW}$ are defined in Lemma \ref{lemma:ELBO_smooth_score} and \ref{lemma:ELBO_smooth_pathwise} respectively.
\end{Proposition}
We establish in Lemma \ref{lem:ELBO_smooth_gauss_finite} that these smoothness constants are well-defined and finite in the Gaussian case. With these results established, we now derive the convergence rates for VAE using both gradient estimators.

\begin{theorem}
\label{th:rate_general_SGD}
Let $\left(\theta_{n},\phi_{n}\right) \in \Theta \times \Phi$ be the $n$-th iterate of the recursion \eqref{eq:SGD}, where $\gamma_{n} = C_{\gamma}n^{-1/2}$ with $C_{\gamma} > 0$.
Assume that for $\idx \in \{\SF, \PW\}$, $\sigma_{\idx}^2 = \sup_{\theta \in \Theta, \phi \in \Phi} \mathbb{E}_{\pi}[ \mathbb{E}_{q_{\phi}(\cdot|X_i)}[\| \tilde{g}^{\idx}_{i,\ell} - \nabla^{\idx}_{\theta, \phi} \mathcal{L}(\theta, \phi)\|^2]] < +\infty$.
For all $n \geq 1$, let $R \in \{0, \ldots, n\}$ be a uniformly distributed random variable. Then,
\begin{itemize}[left=0pt, label={}, itemindent=5pt, itemsep=0pt]
    \item $(i)$ Score Function: under Assumptions \ref{ass:A1} and \ref{ass:A2}(i), and for $C_{\gamma} \leq 1/L^{\SF}$,
        $$
        \mathbb{E}\left[\left\| \nabla^{\SF}_{\theta, \phi} \mathcal{L}\left(\theta_{R}, \phi_{R}\right)\right\|^{2}\right] \leq \frac{2\mathcal{L}^* +  L^{\SF} C_{\gamma} \sigma_{\SF}^{2} \log n/(BK)}{C_{\gamma} \sqrt{n}}\eqsp,
        $$

    \item $(ii)$ Pathwise Gradient: under Assumptions \ref{ass:A2}(ii) and \ref{ass:A2_pathwise}, and for $C_{\gamma} \leq 1/L^{\PW}$,
        $$
        \mathbb{E}\left[\left\| \nabla^{\PW}_{\theta, \phi} \mathcal{L}\left(\theta_{R}, \phi_{R}\right)\right\|^{2}\right] \leq \frac{2\mathcal{L}^* +  L^{\PW} C_{\gamma} \sigma_{\PW}^{2} \log n/(BK)}{C_{\gamma} \sqrt{n}}\eqsp,
        $$
\end{itemize}
where $\mathcal{L}^* = \mathcal{L}\left(\theta^{*}, \phi^{*}\right) - \mathcal{L}\left(\theta_{0}, \phi_{0}\right)$.
\end{theorem}

Theorem \ref{th:rate_general_SGD} provides the classical convergence rate of $\mathcal{O}\left(\log n / \sqrt{n}\right)$ for non-convex problems. This rate indicates that increasing the batch size $B$ and the number of samples $K$ from the variational distribution can improve convergence by reducing the second term. However, larger values of $B$ and $K$ also increase computational costs. Therefore, it is crucial to balance the convergence rate with computational efficiency by choosing appropriate values for $B$ and $K$. In practice, it is common to use a larger batch size $B$ while setting $K = 1$. Additionally, we observe that high variance results in slower convergence, making the pathwise gradient estimator more favorable than the score function estimator.
Theorem~\ref{th:rate_general_Adam} provides the convergence rate of Adam, as defined in Algorithm~\ref{alg:adam}.

\begin{theorem}
\label{th:rate_general_Adam}
Let $\left(\theta_{n},\phi_{n}\right) \in \Theta \times \Phi$ be the $n$-th iterate of the recursion in Algorithm \ref{alg:adam}, where $\gamma_{n} = C_{\gamma}n^{-1/2}$ with $C_{\gamma}>0$. Suppose that $\beta_1 < \sqrt{\beta_2} < 1$ and that the assumptions of Theorem \ref{th:rate_general_SGD} hold. Then, for $\idx \in \{\SF, \PW\}$,
$$
\mathbb{E}\left[\left\| \nabla^{\idx}_{\theta, \phi} \mathcal{L}\left(\theta_{R}, \phi_{R}\right)\right\|^{2}\right] = \mathcal{O}\left(\frac{\mathcal{L}^*}{\sqrt{n}} + L^{\idx} \frac{d^{*} \log n}{\left(1-\beta_{1}\right) \sqrt{n}}\right),
$$
where $\mathcal{L}^* = \mathcal{L}\left(\theta^{*}, \phi^{*}\right) - \mathcal{L}\left(\theta_{0}, \phi_{0}\right)$, $d^{*} = d_{\theta} + d_{\phi}$ is the total dimension of the parameters, and $L^{\SF}$ and $L^{\PW}$ are the smoothness constants for the score function and pathwise gradient cases, respectively.
\end{theorem}

Theorem \ref{th:rate_general_Adam} provides a convergence rate similar to that of SGD, $\mathcal{O}\left(\log n / \sqrt{n}\right)$ but with an additional factor of the total dimension $d^*$, reflecting the impact of the adaptive step sizes.
In practice, Adam typically uses $\beta_1 = 0.9$ and $\beta_2 = 0.999$ \citep{kingma2014adam, zaheer2018adaptive, reddi2019convergence}, which satisfy the condition $\beta_1 < \sqrt{\beta_2} < 1$.

\subsection{Linear Gaussian VAE}
\label{sec:linearVAE}
We consider the following VAE model: for all $x\in\Xset$ and $z\in\Zset$,
\begin{align}\label{eq:linear_VAE}
\begin{split}
p_{\theta}(x | z) = \mathcal{N} (x;W_{1}z + b_{1}, c^2 \mathrm{I}_{d_x}), \\
q_{\phi}(z | x) = \mathcal{N} (z;W_{2}x + b_{2}, D), 
\end{split}
\end{align}
where $\theta = (W_{1}, b_{1},c^2)\in\rset^{d_x\times d_z}\times \rset^{d_x}\times \rset^*_+$ and $\phi = (W_{2}, b_{2},D)\in\rset^{d_z\times d_x}\times \rset^{d_z} \times \rset^{d_z\times d_z}$. The matrix $D$ is a diagonal covariance matrix and serves as an amortized variance for each input point. It is sufficient to achieve the global optimum of this model \citep{lucas2019don}.
Conditionally on $x$, the output of the Linear VAE follows a Gaussian distribution with mean $W_{1}(W_{2}x + b_{2}) + b_{1}$ and variance $W_{1}DW_{1}^\top$.

While analytic solutions for deep latent models are generally not available, the Linear VAE provides analytic solutions for optimal parameters, allowing us to gain insights into various phenomena associated with VAE training. Proposition \ref{prop:closed_form_Linear_VAE} provides an analytical form of the expected ELBO for the Linear VAE, which can be used to analyze the convergence rate. The following corollary derives the convergence rate for the Linear VAE.

\begin{corollary}\label{cor:Linear_VAE_rate}
Consider the Linear Gaussian VAE defined in \eqref{eq:linear_VAE} with $\theta = (W_{1}, b_{1})$ and $\phi = (W_{2}, b_{2},D)$ and let $c_{D} > 0$ such that $\lambda_{\min}(D) \geq c_{D}$.
Assume that the inputs have bounded second moments and there exists some constant $a$ such that for all $\theta \in \Theta$ and $\phi \in \Phi$,
$$ \lVert \theta \rVert_{\infty} + \lVert \phi \rVert_{\infty} \leq a. $$
Let $\left(\theta_{n},\phi_{n}\right) \in \Theta \times \Phi$ be the $n$-th iterate of the recursion in Algorithm \ref{alg:adam}, where $\gamma_{n} = C_{\gamma}n^{-1/2}$ with $C_{\gamma}>0$. Assume that $\beta_1 < \sqrt{\beta_2} < 1$.
For all $n \geq 1$, let $R \in \{0, \ldots, n\}$ be a uniformly distributed random variable. Then,
$$
\mathbb{E}\left[\left\| \nabla_{\theta, \phi} \mathcal{L}\left(\theta_{R}, \phi_{R}\right)\right\|^{2}\right] = \mathcal{O}\left(\frac{\mathcal{L}^*}{\sqrt{n}} + \frac{d_{x} d_{z} \log n}{\sqrt{n}}\right),
$$
where $\mathcal{L}^* = \mathcal{L}\left(\theta^{*}, \phi^{*}\right) - \mathcal{L}\left(\theta_{0}, \phi_{0}\right)$.
\end{corollary} 
In \cite{lucas2019don}, it is shown that the ELBO objective for a Linear VAE does not introduce any local maxima beyond the marginal loglikelihood.  Corollary \ref{cor:Linear_VAE_rate} indicates a convergence rate of $\mathcal{O}(\log n / \sqrt{n})$ with respect to the marginal loglikelihood.

\subsection{Deep Gaussian VAE}
\label{sec:deepVAE}
We show that the assumptions of our main results hold for classical architectures of neural networks and discuss the choice of hyperparameters for this architecture.
The deep Gaussian VAE consists of a decoder and an encoder such that $p_{\theta}(x | z) = \mathcal{N}(x; G_{\theta}(z), c^2 \mathrm{I}_{d_x})$ and $q_{\phi}(z | x) = \mathcal{N}(z; \mu_{\phi}(x), \Sigma_{\phi}(x))$.
Consider the following neural network formulations for the encoder and decoder:
\begin{multline*}
\mathcal{F}_G = \Big\{G_{\theta}(z) = \NN(z; \theta, f, N_{dd}), \, \theta = \{W_{\ell}, b_{\ell}\}_{\ell=1}^{N_{dd}}, \\
\sigma_{\ell} \in \mathcal{F}_{\text{SL}}, \ell = 1, \ldots, N_{dd}, \; \| G_{\theta}(z) \| \leq C_{G} \Big\},  
\end{multline*}
and
\begin{multline*}
\mathcal{F}_{\mu,\Sigma} = \Big\{ \left(\mu_{\phi}(x), \Sigma_{\phi}(x)\right) = \NN(x; \phi, f, N_{ed}), \\
\qquad \phi = \{W_{\ell}, b_{\ell}\}_{\ell=1}^{N_{ed}}, \, \, f_{\ell} \in \mathcal{F}_{\text{SL}}, \ell = 1, \ldots, N_{ed},\\
\qquad \| \mu_{\phi}(x) \| \leq C_{\mu}, \, \lambda_{\min}(\Sigma_{\phi}(x)) \geq c_{\Sigma}, \, \| \Sigma_{\phi}(x) \| \leq C_{\Sigma} \Big\}.
\end{multline*}
In $\mathcal{F}_{G}$ and $\mathcal{F}_{\mu,\Sigma}$, $f$ represents an activation function. Specifically, Proposition \ref{prop:smooth_act} suggests that $f$ can be chosen from a set that includes the sigmoid, Tanh, softplus \citep{glorot2011deep}, or CELU \citep{clevert2015fast, barron2017continuously} defined in Appendix \ref{supp:subsec:activation_function}.

A common approach to satisfy the condition on the minimum eigenvalue of the covariance matrix is to add a regularization term. However, in practice, since the encoder outputs the log variance, it is more efficient to impose a lower bound directly on the log variance. This ensures that the eigenvalues of the covariance matrix remain bounded from below, thus maintaining numerical stability. Furthermore, to ensure that $\| \mu_{\phi}(x) \| \leq C_{\mu}$ and $\| \Sigma_{\phi}(x) \| \leq C_{\Sigma}$, it is essential for the activation function in the final layer to be bounded. While sigmoid and tanh are commonly used for their bounded nature, they may not be suitable when linearity is required, as often needed for the encoder. An alternative is Hard Tanh \citep{tang2021empirical}, which clips the identity function, achieving both linearity and boundedness, though it lacks differentiability. To address this, we propose a generalized soft-clipping activation function, an extension of \citep{klimek2020neural}, designed to be approximately linear within a specified interval. We also provide a smoothness analysis in Proposition~\ref{prop:soft_clip}.

\begin{Proposition}\label{prop:soft_clip}
For $s_1, s_2 \in \rset$ with $s_1 \leq s_2$ and $s \in \rset^*_+$, the generalized soft-clipping activation function defined by
$$
f(x) = \frac{1}{s} \log \left( \frac{1 + \rme^{s(x-s_1)}}{1 + \rme^{s(x-s_2)}} \right) + s_1,
$$
is bounded between $s_1$ and $s_2$, and is Lipschitz continuous and smooth.
\end{Proposition}

The generalized soft-clipping activation function is approximately linear within the interval $(s_1,s_2)$. The parameter $s$ plays a crucial role in determining the shape and sharpness of the transition between $s_1$ and $s_2$.

\begin{theorem}\label{thm:con_gauss}
Let $G_{\theta}(z) \in \mathcal{F}_G$ and $(\mu_{\phi}(x), \Sigma_{\phi}(x)) \in \mathcal{F}_{\mu, \Sigma}$ for all $(x, z) \in \Xset \times \Zset$.
Assume that there exists $C_{rec} \in \mathsf{M}(\Xset \times \Zset)$ such that $\left\| x - G_{\theta}(z) \right\| \leq C_{rec}(x, z)$ for all $\theta \in \Theta$ and $(x, z) \in \Xset \times \Zset$.
Assume also that the data distribution $\pi$ has a finite fourth moment, and that there exists some constant $a$ such that for all $\theta \in \Theta$ and $\phi \in \Phi$,
$$ 
\lVert \theta \rVert_{\infty} + \lVert \phi \rVert_{\infty} \leq a\eqsp. 
$$
Let $\left(\theta_{n},\phi_{n}\right) \in \Theta \times \Phi$ be the $n$-th iterate of the recursion in Algorithm \ref{alg:adam}, where $\gamma_{n} = C_{\gamma}n^{-1/2}$ with $C_{\gamma}>0$. Assume that $\beta_1 < \sqrt{\beta_2} < 1$.
For all $n \geq 1$, let $R \in \{0, \ldots, n\}$ be a uniformly distributed random variable. Then, for $\idx \in \{\SF, \PW\}$,
$$
\mathbb{E}\left[\left\| \nabla^{\idx}_{\theta, \phi} \mathcal{L}\left(\theta_{R}, \phi_{R}\right)\right\|^{2}\right] = \mathcal{O}\left(\frac{\mathcal{L}^*}{\sqrt{n}} + C^{\idx}\frac{d^{*} \log n}{\left(1-\beta_{1}\right)\sqrt{n}}\right)\eqsp,
$$
where $\mathcal{L}^* = \mathcal{L}\left(\theta^{*}, \phi^{*}\right) - \mathcal{L}\left(\theta_{0}, \phi_{0}\right)$, $d^{*} = d_{\theta} + d_{\phi}$ is the total dimension of the parameters, and $C^{\SF} = d_{z}^2 N_{\text{max}} a^{2(N_{\text{max}}-1)}$ and $C^{\PW} = d_{z} N_{\text{total}} a^{2(N_{\text{total}}-1)}$.
Here, $N_{\text{max}} = \max\{N_{ed}, N_{dd}\}$ denotes the maximum number of layers, while $N_{\text{total}} = N_{ed} + N_{dd}$ represents the total number of layers in the encoder and decoder.
\end{theorem}

Theorem \ref{thm:con_gauss} provides the convergence rate of $\mathcal{O}\left(\log n / \sqrt{n}\right)$ for deep Gaussian VAE, which is commonly used in practice.
In deep learning, it is standard practice to initialize weights using a distribution scaled by $\mathcal{O}(1/\sqrt{d})$, such as $\mathcal{N}(0, 1/\sqrt{d})$ or $\mathcal{U}(-1/\sqrt{d}, 1/\sqrt{d})$ \citep{glorot2010understanding, he2015delving, li2017convergence}. This initialization ensures that the spectral norm of the resulting weights matrix is typically $\mathcal{O}(1)$ \citep{rudelson2010non}. Consequently, assuming the compactness of the parameter space is well-justified. This assumption is also used to derive the excess risk of VAE \citep{tang2021empirical}.

Our choice of activation functions to achieve the convergence rate is reasonable and does not deviate significantly from commonly used activation functions. Although our results do not directly apply to the ReLU activation function, experimental results demonstrate that similar convergence rates can still be achieved using ReLU.
Furthermore, since $G_{\theta}$ is bounded, the assumption regarding the reconstruction error $C_{rec}$ can be easily verified if the inputs $x$ are also bounded.

In our convergence rate analysis, the smoothness constant depends on the number of layers $N$ and grows exponentially with $N$. Specifically, the leading term in the smoothness constant is of the form $N \times a^{2(N-1)}$. This growing exponential factor also appears in the Lipschitz constant of neural networks, where the term is $\sqrt{N} \times a^{N-1}$ \citep{virmaux2018lipschitz, combettes2020lipschitz, tang2021empirical}.
As the number of hidden layers increases, both the smoothness constant and the total parameter dimension $d^*$ grow, leading to a greater number of iterations required for convergence.
\citet{damm2023elbo} shows that for Deep Gaussian VAE, the ELBO at stationary points is equal to the sum of the negative entropy of the prior distribution, the expected negative entropy of the observable distribution, and the average entropy of the variational distributions. Consequently, the ELBO in Deep Gaussian VAE converges to a sum of entropies at a rate of $\mathcal{O}\left(\log n / \sqrt{n}\right)$.

\subsection{Some Variants of VAE}

\subsubsection{\texorpdfstring{$\beta$-VAE}{beta-VAE}}

$\beta$-VAE \citep{higgins2017beta} is a variant of VAE introducing a parameter $\beta$ to control the trade-off between the reconstruction term and the regularization of the latent space. The ELBO for $\beta$-VAE is given by:
$$
\mathcal{L}_{\beta}(\theta, \phi; x) = \mathbb{E}_{q_\phi(\cdot|x)}\left[\log p_\theta(x|Z)\right] - \beta D_{\text{KL}}(q_\phi(\cdot|x)||p)\eqsp,
$$
where the Lagrangian multiplier $\beta$ is considered as a hyperparameter.
The small values of $\beta$ force decoders to use the latent variables, but this comes at the cost of a poor ELBO.
The role of $\beta$ invites a natural comparison to the parameter $c^2$ in the objective of the standard Gaussian VAE. Setting $c^2$ small in the standard VAE corresponds to setting $\beta$ small in $\beta$-VAE. For a given $\beta$, one can find a corresponding $c^2$ (and a learning rate) such that the gradient updates to the network parameters are identical \citep{lucas2019don}.

Given that $\beta$ plays a role analogous to $1/c^2$, when applying Theorems \ref{th:rate_general_Adam} and \ref{thm:con_gauss}, we observe that if $\beta < \infty$, $\beta$-VAE with the same architecture as in Theorem \ref{thm:con_gauss}, converges to a critical point of the expected ELBO at a rate of $\mathcal{O}\left(\log n / \sqrt{n}\right)$.
The smaller the value of $\beta$, the faster the convergence, due to the analogous role of $\beta$ and $1/c^2$. Selecting a sufficiently small $\beta$ leads to faster convergence and can help prevent posterior collapse \citep{wang2022posterior}.

\subsubsection{IWAE}
\label{sec:iwae}
The Importance Weighted Autoencoder (IWAE) \citep{burda2015importance} is another extension of the VAE that incorporates importance weighting to obtain a tighter ELBO.
The IWAE objective function is:
\begin{equation*}
\mathcal{L}^{\mathsf{IS}}_K(\theta, \phi) = \mathbb{E}_{\pi} \left[ \mathbb{E}_{q^{\otimes K}_{\phi}(\cdot|X)} \left[ \log \frac{1}{K} \sum_{\ell=1}^{K} \frac{p_{\theta}(X, Z^{(\ell)})}{q_{\phi}(Z^{(\ell)}|X)} \right] \right]\eqsp,
\end{equation*}
where $K$ corresponds to the number of samples drawn from the variational posterior distribution.

\begin{assumption}\label{ass:A3}
There exist $\alpha^-, \alpha^+ \in \mathsf{M}(\Xset \times \Zset)$ such that for all $\theta \in \Theta$, $\phi \in \Phi$, $x \in \Xset$, $\varepsilon, z \in \Zset$ where $z = g(\varepsilon, \phi)$,
$$\alpha^-(x,\varepsilon) \leq \max\{ p_{\theta}(x, z) , q_{\phi}(z|x) \} \leq \alpha^+(x,\varepsilon)\eqsp.$$
\end{assumption}
Assumption \ref{ass:A3} states the boundedness of both the joint probability density function and the variational density function. Given the existence of the reparametrization trick, Assumptions \ref{ass:A1} and \ref{ass:A3} are equivalent, and the bound is verified with $\alpha(x,z) = \max\{ \left| \log \alpha^-(x,\varepsilon) \right|, \left|\log \alpha^+(x,\varepsilon)\right| \}$.

\begin{theorem} \label{th:con_IWAE}
Let Assumptions \ref{ass:A2}(ii)-\ref{ass:A3} hold. Let $\left(\theta_{n},\phi_{n}\right) \in \Theta \times \Phi$ be the $n$-th iterate of the recursion in Algorithm \ref{alg:adam} where $\mathcal{L}$ is the IWAE objective, and $\gamma_{n} = C_{\gamma}n^{-1/2}$ with $C_{\gamma}>0$. Assume that $\beta_1 < \sqrt{\beta_2} < 1$. 
For all $n \geq 1$, let $R \in \{0, \ldots, n\}$ be a uniformly distributed random variable. 
Then,
$$
\mathbb{E}\left[\left\| \nabla_{\theta, \phi} \mathcal{L}^{\mathsf{IS}}_K\left(\theta_{R}, \phi_{R}\right)\right\|^{2}\right] = \mathcal{O}\left(d^{*} L_{K} \frac{\log n}{\sqrt{n}}\right)\eqsp,
$$
where $d^{*} = d_{\theta} + d_{\phi}$, and $L_{K}$ is as defined in \eqref{eq:smooth_iwae}.
\end{theorem}

We achieve a convergence rate similar to that of VAE and $\beta$-VAE. In particular, as $K$ increases, the convergence rate improves and becomes nearly inversely proportional to $K$.

\textbf{Link with Signal to Noise ratio. }
In \cite{rainforth2018tighter}, the authors propose to measure the relative accuracy of the gradient estimates using the Signal to Noise ratio (SNR), i.e. the absolute value of
the expected estimate of the gradient scaled by its standard deviation. 
They highlight that a low SNR is problematic, as it indicates that gradient estimates are dominated by noise. Our results align with \cite[Theorem 1]{rainforth2018tighter}, which establishes that the SNR scales as $\sqrt{BK}$ for $\theta$ and  $\sqrt{B/K}$ for $\phi$, where $B$ is the batch size. This means that increasing $K$ independently of $B$ might lead to vanishing SNR and poor gradient estimates for $\phi$ and motivate adaptive choices of $K$ with respect to $B$.

\begin{figure*}[b]
    \vskip -0.2in
    \begin{center}
    \centerline{
        \includegraphics[width=0.5\textwidth]{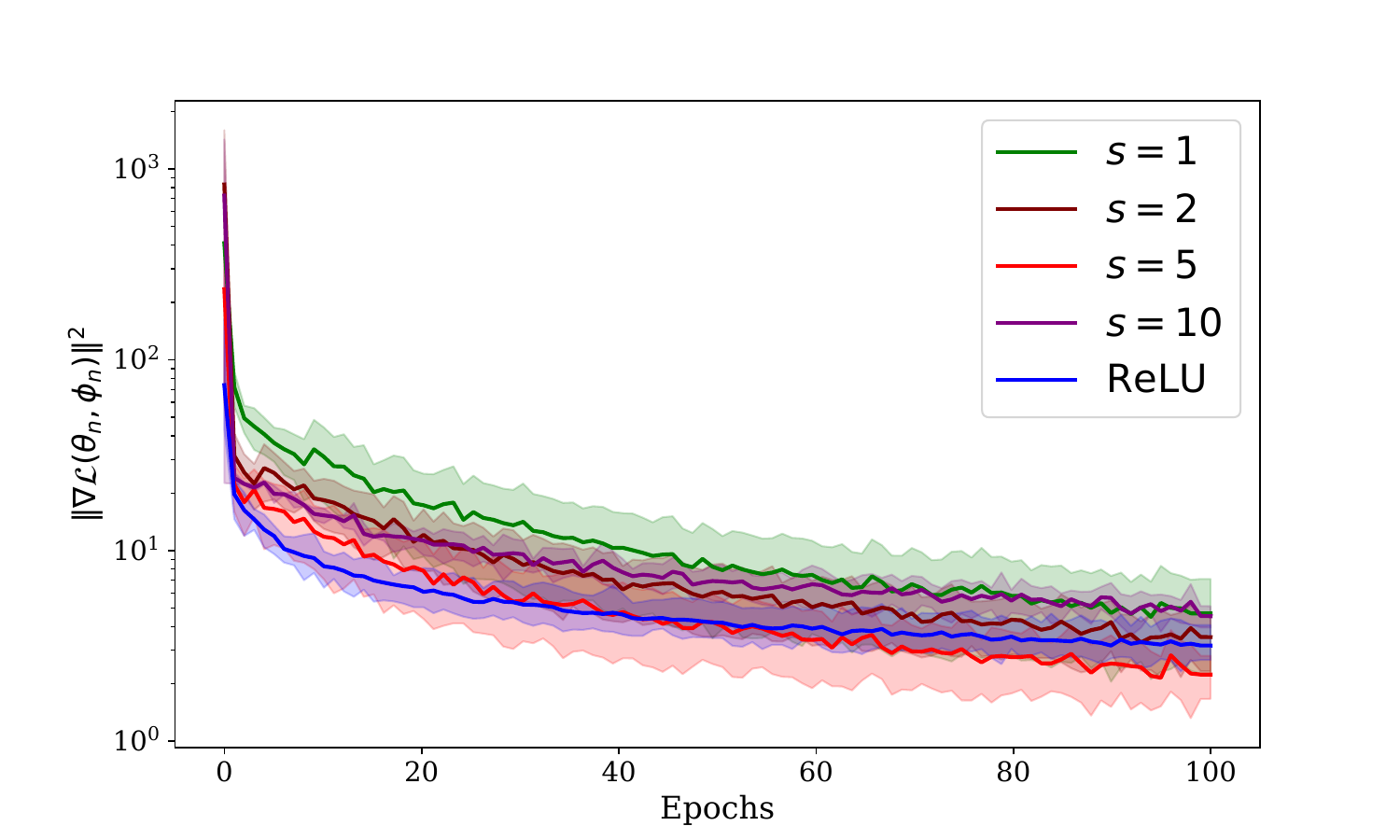}
        \includegraphics[width=0.5\textwidth]{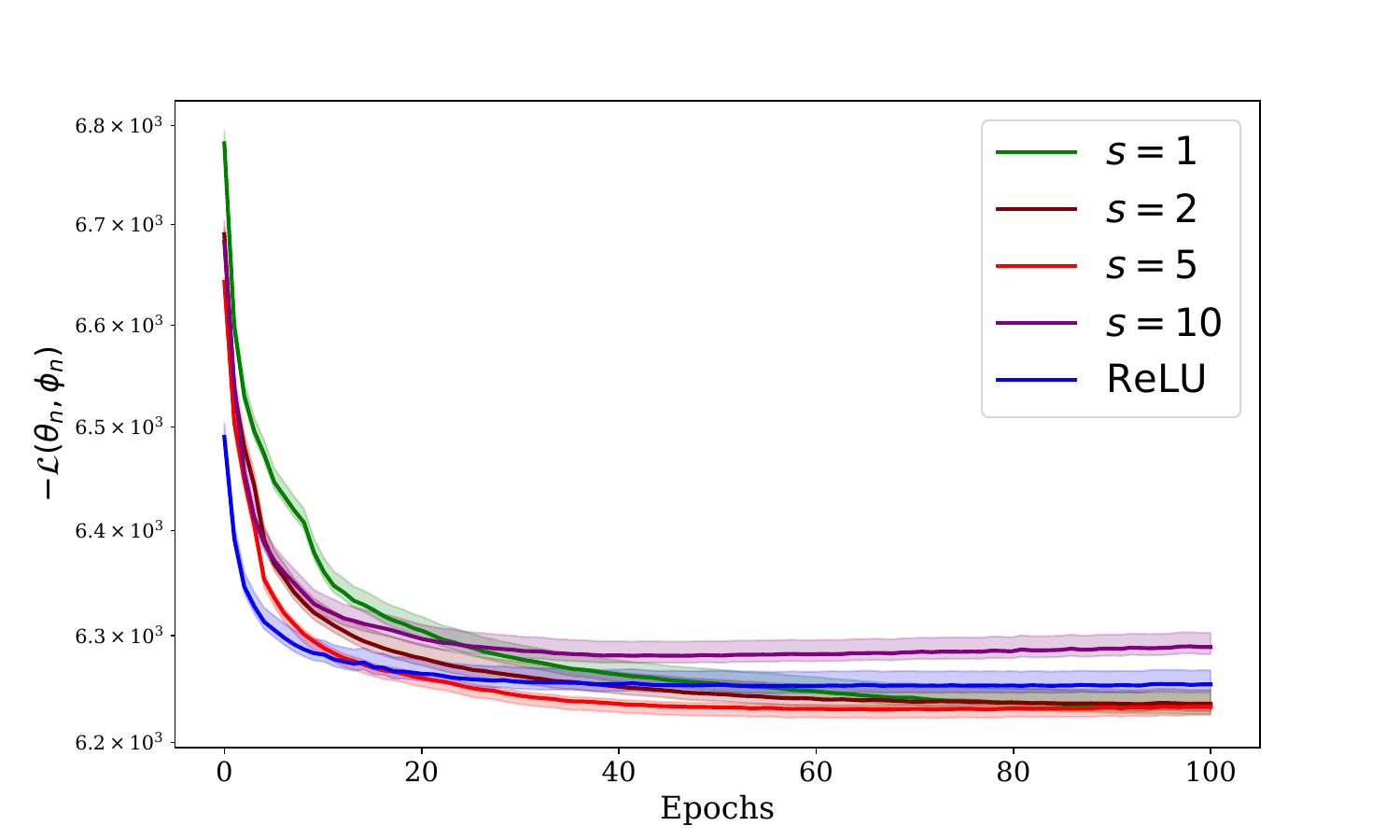}
    }
    \caption{Squared norm of gradients and Negative ELBO on the test set of the CelebA for VAE trained with Adam and generalized soft-clipping activation function. Bold lines represent the mean over 5 independent runs.}
    \label{all_CelebA}
    \end{center}
    \vskip -0.2in
\end{figure*}

\subsection{Extension to Variational Inference}
\label{sec:BBVI}

Black-Box Variational Inference (BBVI) is typically formulated as the maximization of the following objective function \citep{ranganath2014black}:
$$\mathcal{L}^{\mathsf{BBVI}}(\phi; x) = \mathbb{E}_{q_{\phi}(\cdot|x)} \left[ \log p(x, z) - \log q_{\phi}(z|x) \right]\eqsp,$$
where $q_{\phi}$ is the variational distribution with parameter $\phi$. As a special case of VAE, BBVI optimizes only $\phi$, not $\theta$. 
Existing convergence results for BBVI have been established in both convex \citep{domke2020provable, kim2024linear} and non-convex settings \citep{domke2024provable, kim2024convergence}. These results typically rely on smoothness assumptions about the ELBO, which are often derived under linear parameterization or the location-scale family.
The following corollary extends these results by removing the location-scale assumption, making them applicable to a broader class of reparameterization families.

\begin{corollary}\label{cor:ELBO_BBVI}
Assume the following conditions hold. 
There exist $M$, $M_g$, $L_g$, $L_p$, and $L_q \in \mathsf{M}(\Xset \times \Zset)$ such that for all $\phi \in \Phi$, $x \in \Xset$ and $\varepsilon \in \Zset$ with $z = g(\varepsilon, \phi)$,
\vspace{-4mm}
\begin{itemize}
    \item[(i)] $z \mapsto \log p(x,z)$ is $L_{p}(x, \varepsilon)$-smooth.
    \vspace{-2mm}
    \item[(ii)] $z \mapsto \log q_{\phi}(z|x)$ is $M(x, \varepsilon)$-Lipchitz and $L_{q}(x, \varepsilon)$-smooth.
    \vspace{-2mm}
    \item[(iii)] $\phi \mapsto g(\phi, \varepsilon)$ is $M_{g}(x, \varepsilon)$-Lipchitz and $L_{g}(x, \varepsilon)$-smooth.
\end{itemize}
\vspace{-4mm}
Then, $\phi \mapsto \mathcal{L}^{\mathsf{BBVI}}(\phi)$ is $L^{\mathsf{BBVI}}$-smooth, where the smoothness constant $L^{\mathsf{BBVI}}$ is given by \eqref{eq:elbo_bbvi}.
\end{corollary} 

Corollary \ref{cor:ELBO_BBVI} establishes the smoothness of the ELBO, which is crucial for achieving the convergence rate. Our assumptions are less restrictive than those in \citet{domke2020provable, kim2024convergence} and do not depend on a specific reparameterization trick, unlike prior works that often assume a location-scale parameterization. A detailed comparison is provided in Appendix \ref{app:BBVI}. This allows us to attain a convergence rate of $\mathcal{O}\left(\log n / \sqrt{n}\right)$, as in Theorem \ref{th:rate_general_SGD}.

\section{EXPERIMENTS}

\begin{figure*}[ht!]
\vskip -0.2in
\begin{center}
\centerline{
    \includegraphics[width=0.5\textwidth]
    {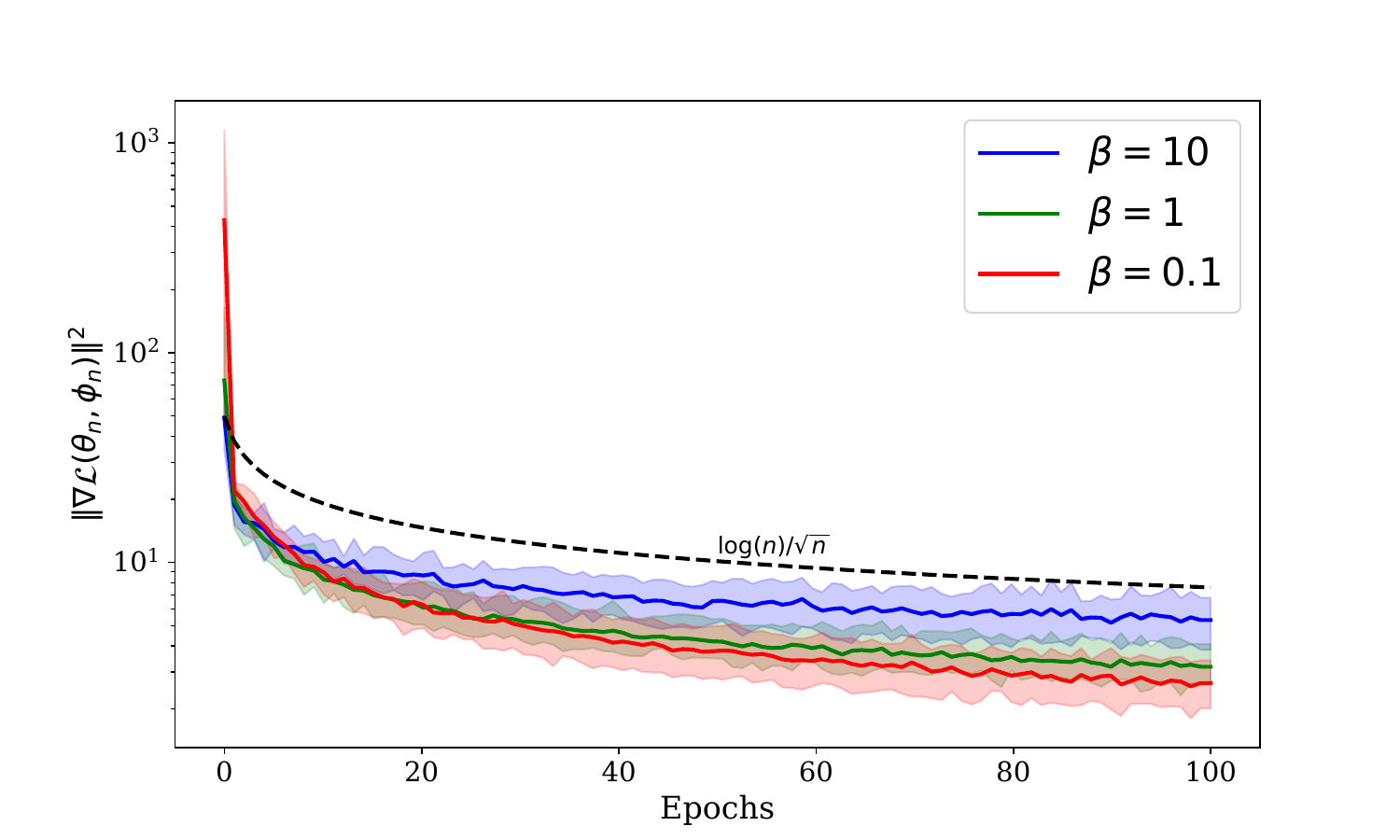}
    \includegraphics[width=0.5\textwidth]{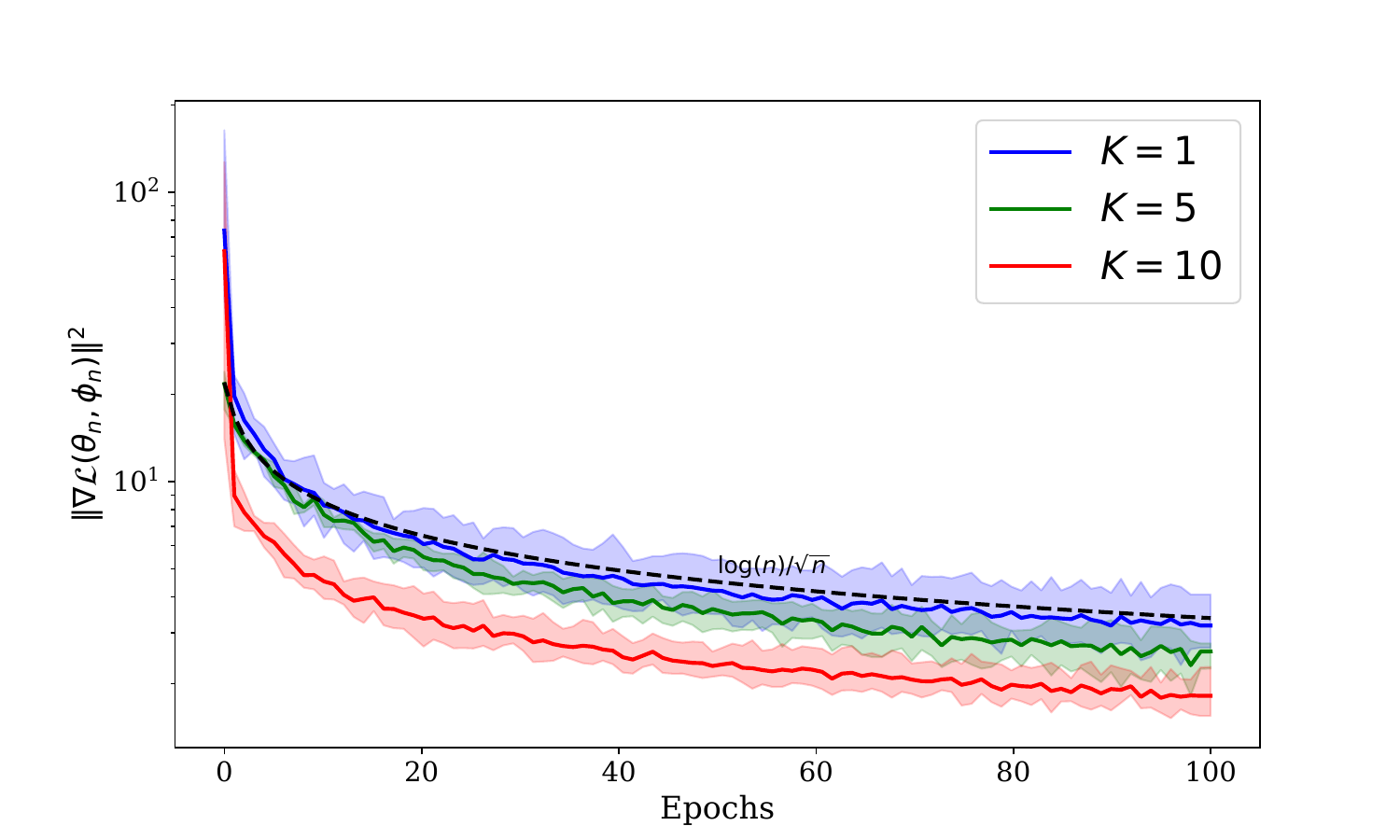}
    }
\caption{$\| \nabla \mathcal{L}(\theta_n, \phi_n) \|^{2}$ in $\beta$-VAE (on the left) and IWAE (on the right) trained with Adam. Bold lines represent the mean over 5 independent runs.
The dashed curves correspond to the expected convergence rate $\mathcal{O}(\log n/\sqrt{n})$.}
\label{beta_iwae_CelebA}
\end{center}
\vskip -0.2in
\end{figure*}

\begin{figure*}[ht!]
\vskip -0.1in
\begin{center}
\centerline{
    \includegraphics[width=0.5\textwidth]
    {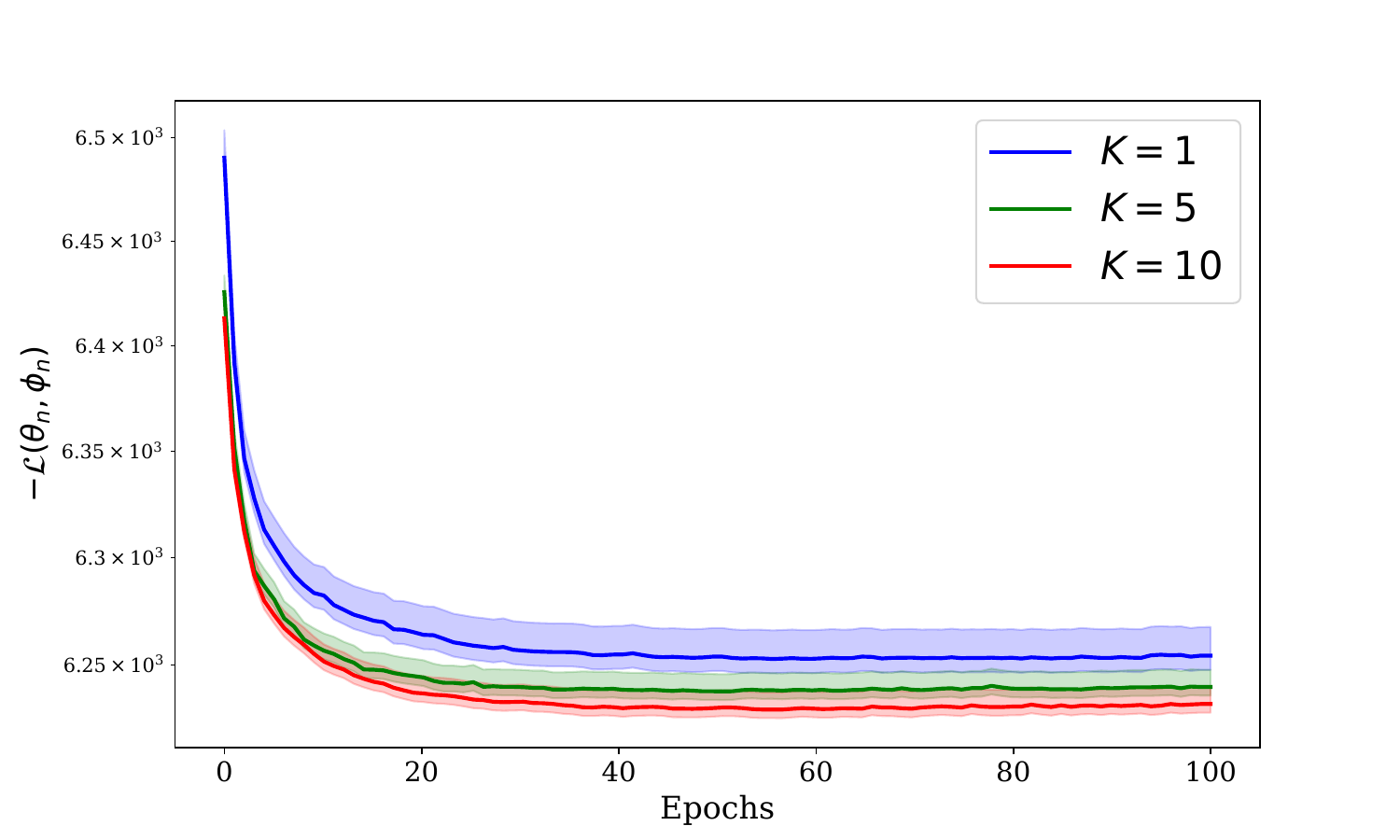}
    \includegraphics[width=0.5\textwidth]{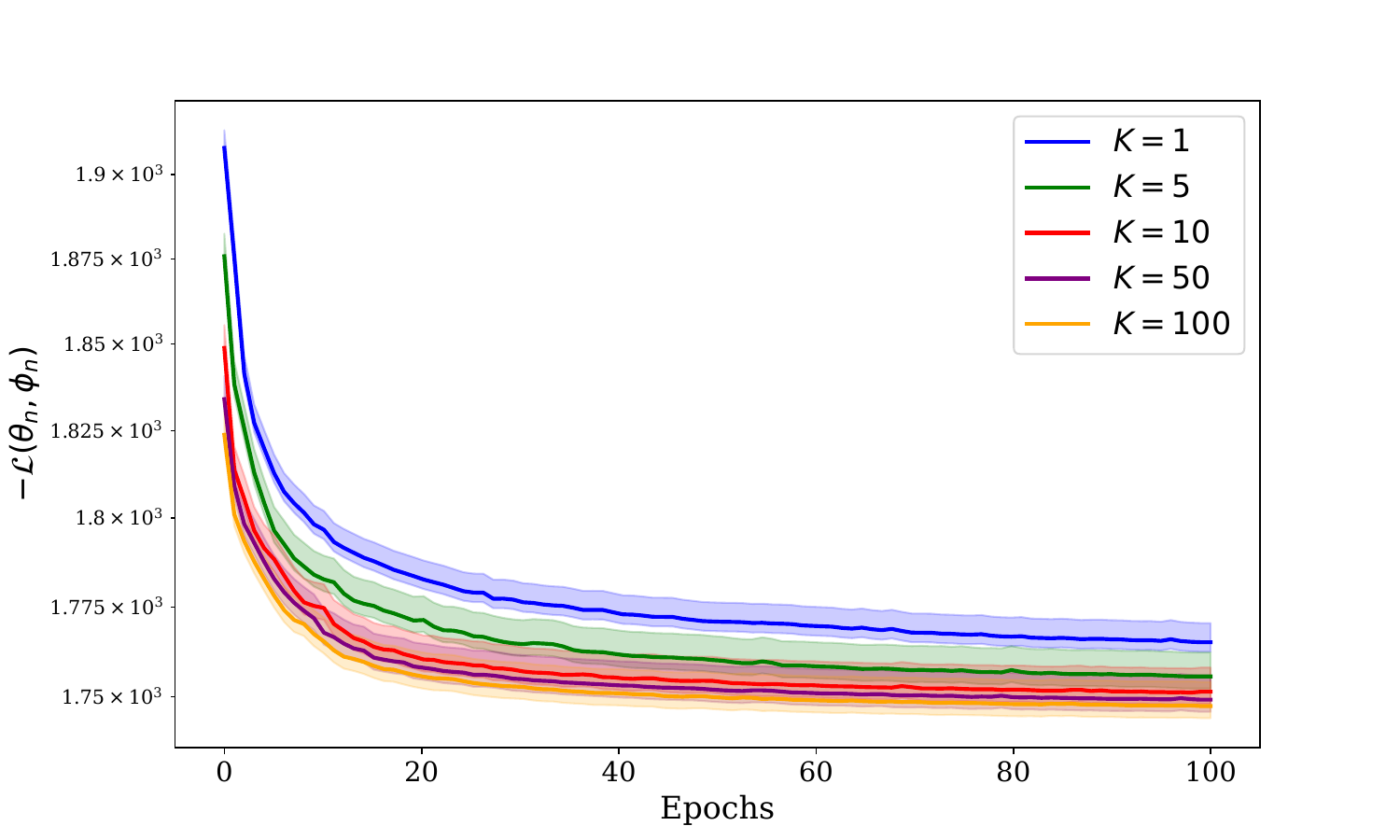}
    }
\caption{Negative ELBO in IWAE on the test set of the CelebA (on the left) and CIFAR-100 (on the right) trained with Adam. Bold lines represent the mean over 5 independent runs.
}
\label{iwae_CelebA_CIFAR}
\end{center}
\vskip -0.2in
\end{figure*}

In this section, we illustrate our theoretical results in the context of deep Gaussian VAE. The experiments were conducted using PyTorch \citep{paszke2017automatic}, and the source code can be found here\footnote{https://github.com/SobihanSurendran/VAE-Convergence-Guarantees}.

\textbf{Dataset and Model. } We conduct our experiments on the CelebA dataset \citep{liu2018large} and use a Convolutional Neural Network (CNN) architecture with Rectified Linear Unit (ReLU) and generalized soft-clipping activation functions for both the encoder and decoder networks. The latent space dimension is set to 100. We estimate the log-likelihood using the VAE, $\beta$-VAE, and IWAE models, all of which are trained for 100 epochs. Training is performed with Adam optimizer and learning rate decay defined as $\gamma_{n} = C_{\gamma}/\sqrt{n}$, where $C_{\gamma}=0.001$. The momentum parameters are set to $\beta_1 = 0.9$ and $\beta_2 = 0.999$, and the regularization parameter $\delta$ is fixed at $10^{-8}$. Note that while all figures are plotted with respect to epochs, $n$ here denotes the number of gradient updates. Additional details are provided in Appendix~\ref{supp:sec:exp}.

For the first experiment, we illustrate our convergence results of the standard VAE with our choice of activation functions.
Figure \ref{all_CelebA} shows the squared norm of the gradients $\| \nabla \mathcal{L}(\theta_n, \phi_n) \|^{2}$ and the Negative ELBO on the test dataset for both ReLU and the generalized soft-clipping activation function with various values of $s$.
We observe a similar convergence rate for all values of $s$. However, selecting an appropriate value of $s$ is crucial to achieve optimal convergence. Theoretically, smaller values of $s$ should yield faster convergence. In practice, however, when $s$ becomes too small, the generalized soft-clipping function behaves nearly like a constant, effectively halting learning. Thus, $s=5$ appears to be a reasonable choice, balancing convergence rate and effective gradient flow.

Moreover, our choice of activation functions leads to improved convergence rates compared to the standard VAE with ReLU. While our theoretical analysis does not directly apply to ReLU, experimental results indicate that similar convergence performance can be achieved.

Figure \ref{beta_iwae_CelebA} illustrates the squared norm of the gradients $\| \nabla \mathcal{L}(\theta_n, \phi_n) \|^{2}$ for both the $\beta$-VAE and IWAE models, evaluated across different values of $\beta$ and $K$. The standard VAE is a special case of these models, corresponding to $\beta = 1$ in the $\beta$-VAE and $K = 1$ in the IWAE. As expected, we observe that for the $\beta$-VAE, smaller values of $\beta$ lead to faster convergence. Similarly, for the IWAE, increasing the value of $K$ results in faster convergence.
However, beyond a certain threshold, neither the gradient norm nor the objective value improves significantly (Figure \ref{iwae_CelebA_CIFAR}), instead incurring unnecessary computational cost. This behavior aligns with the earlier discussion on the Signal to Noise Ratio, see Section~\ref{sec:iwae}.

While the objective is for the gradient norm to converge to zero, it should not do so too quickly. If the gradient becomes too small as $K$ increases, the training procedure is prone to yield poor results regarding $\phi$, thereby limiting improvements in $\theta$. This highlights the need for careful selection of $K$. It is also crucial to consider the other hyperparameters, as they impact the convergence rate. When adjusting $K$, it is important to adjust the other parameters accordingly to optimize the convergence rate. 
One approach is to gradually increase $K$ until a suitable threshold is reached \citep{surendran2024non}. Alternatively, Variational Rényi IWAE can be used, ensuring the SNR scales as $\sqrt{BK}$ for both $\theta$ and $\phi$ \citep{daudel2023alpha}.

\section{DISCUSSION}

This paper provides a non-asymptotic convergence analysis of VAE trained using both SGD and Adam algorithms. We derive a convergence rate of $\mathcal{O}(\log n/\sqrt{n})$, applicable to Linear VAE, Deep Gaussian VAE, and several VAE variants.
Our analysis indicates that smaller values of $\beta$ in $\beta$-VAE and the large values of $K$ in IWAE lead to faster convergence rates. However,  increasing $K$ independently of the batch size $B$ can lead to vanishing SNR and poor gradient estimates for $\phi$, thereby hindering the learning of $\theta$.

For Deep Gaussian VAE, we introduce a generalized soft-clipping activation function that supports our theoretical claims and yields improved convergence rates compared to standard VAE using ReLU.
Although our analysis does not directly address ReLU, empirical results suggest that similar convergence rates can still be achieved. A promising direction for future work is to explore alternative distributions for the encoder and decoder, along with different deep architectures. Additionally, extending our results to Variational Rényi IWAE would be a valuable direction for future work.

\section*{Acknowledgements}
The PhD of Sobihan Surendran was funded by the Paris Region PhD Fellowship Program of Région Ile-de-France. We would like to thank SCAI (Sorbonne Center for Artificial Intelligence) for providing the computing clusters. We also express our gratitude to the reviewers for their insightful comments and suggestions, which have helped improve this paper.

\bibliographystyle{abbrvnat}
\bibliography{ref}

\newpage
\section*{Checklist}

 \begin{enumerate}

 \item For all models and algorithms presented, check if you include:
 \begin{enumerate}
   \item A clear description of the mathematical setting, assumptions, algorithm, and/or model. [Yes]
   \item An analysis of the properties and complexity (time, space, sample size) of any algorithm. [Yes]
   \item (Optional) Anonymized source code, with specification of all dependencies, including external libraries. [Yes]
 \end{enumerate}

 \item For any theoretical claim, check if you include:
 \begin{enumerate}
   \item Statements of the full set of assumptions of all theoretical results. [Yes]
   \item Complete proofs of all theoretical results. [Yes]
   \item Clear explanations of any assumptions. [Yes]     
 \end{enumerate}

 \item For all figures and tables that present empirical results, check if you include:
 \begin{enumerate}
   \item The code, data, and instructions needed to reproduce the main experimental results (either in the supplemental material or as a URL). [Yes]
   \item All the training details (e.g., data splits, hyperparameters, how they were chosen). [Yes]
         \item A clear definition of the specific measure or statistics and error bars (e.g., with respect to the random seed after running experiments multiple times). [Yes]
         \item A description of the computing infrastructure used. (e.g., type of GPUs, internal cluster, or cloud provider). [Yes]
 \end{enumerate}

 \item If you are using existing assets (e.g., code, data, models) or curating/releasing new assets, check if you include:
 \begin{enumerate}
   \item Citations of the creator If your work uses existing assets. [Not Applicable]
   \item The license information of the assets, if applicable. [Not Applicable]
   \item New assets either in the supplemental material or as a URL, if applicable. [Not Applicable]
   \item Information about consent from data providers/curators. [Not Applicable]
   \item Discussion of sensible content if applicable, e.g., personally identifiable information or offensive content. [Not Applicable]
 \end{enumerate}

 \item If you used crowdsourcing or conducted research with human subjects, check if you include:
 \begin{enumerate}
   \item The full text of instructions given to participants and screenshots. [Not Applicable]
   \item Descriptions of potential participant risks, with links to Institutional Review Board (IRB) approvals if applicable. [Not Applicable]
   \item The estimated hourly wage paid to participants and the total amount spent on participant compensation. [Not Applicable]
 \end{enumerate}

 \end{enumerate}

\appendix
\onecolumn

\doparttoc 
\faketableofcontents 
\part{} 

\appendix
\aistatstitle{Theoretical Convergence Guarantees for Variational Autoencoders: Supplementary Materials}
\vspace{-4cm}
\parttoc
\newpage

\textbf{Notations. } 
Given vectors $v = \begin{bmatrix} v_1, v_2, \ldots, v_p \end{bmatrix}^\top$ and $u = \begin{bmatrix} u_1, u_2, \ldots, u_d \end{bmatrix}^\top$, where $u$ is a function of $v$, the derivative of $u$ with respect to the vector $v$, denoted by $\nabla_{v} u$, is a matrix of size $(d,p)$, and it is defined as follows:
$$
\nabla_{v} u := \frac{\partial u}{\partial v}^\top\eqsp,
$$
so that for all $1\leq i \leq d$, $1\leq j \leq p$, $(\nabla_{v} u)_{ij} = \partial u_i/\partial v_j$. 
For all $v\in\Rset^d$,  we use $\text{Diag}(v)$ to denote the diagonal matrix with diagonal given by $v$. For all $A\in\Rset^{d\times d}$, $\text{diag}(A)$ is the vector obtained with the diagonal elements of $A$. The Hadamard product of vectors $u$ and $v$ is denoted by $u \odot v$, and $v^2 = v \odot v$ represents the elementwise product of $v$ with itself. For $d \geq 1$, let $\mathrm{I}_{d}$ denote the $d \times d$ identity matrix and $\mathbf{1}$ be the vector in $\Rset^{d}$ whose entries are all equal to $1$. 
Let $u \in \rset^{d}$ be a vector and $A \in \rset^{d \times p}$ be a matrix with columns $A_1, \ldots, A_p \in \rset^{d}$, The element-wise product of $u$ with the matrix $A$, denoted by  $x \cdot A$, is defined as:
$$u \cdot A = \left[u \odot A_1, \ldots, u \odot A_p \right]\eqsp,$$
where $\odot$ denotes the Hadamard (element-wise) product.
For matrices $W_1, \cdots, W_N$ where $W_i \in \rset^{d_{i} \times d_{i-1}}$, the product $\prod_{j=1}^{N} W_j$ is defined as: $$\prod_{j=1}^{N} W_j = W_N W_{N-1} \cdots W_1\eqsp,$$
which is generally not equal to $W_1 W_{2} \cdots W_N$.
The notation $\det(\cdot)$ denotes the determinant of a matrix, and $\text{tr}(\cdot)$ denotes the trace of a matrix.

\section{ADDITIONAL METHODOLOGICAL DETAILS}
\label{supp:sec:add_details}

The conditional likelihood $p_{\theta}(x|z)$, referred to as the decoder distribution, is generally defined as a Gaussian distribution for real-valued data or a Bernoulli distribution for binary data. Specifically, for the Gaussian decoder where $p_{\theta}(x|z) = \mathcal{N}(x;G_{\theta}(z), \Gamma_{\theta}(z))$, the reconstruction loss simplifies to the Mean Squared Error if $\Gamma_{\theta}$ is assumed to be the identity matrix. For a Bernoulli decoder, this corresponds to the binary cross-entropy loss. The prior over the latent variables is typically chosen to be an isotropic multivariate Gaussian.
The encoder distribution is also commonly modeled as a Gaussian.

\begin{figure}[htbp]
    \vspace{10pt}
    \centering
    \begin{center}
\begin{tikzpicture}[auto, thick, >=Triangle]
    \tikzstyle{connector} = [->, thick]
    \tikzstyle{box} = [draw, rectangle, thick, minimum height=3em, minimum width=4em, align=center]
    \tikzstyle{smallbox} = [draw, rectangle, thick, minimum height=3em, minimum width=2em, align=center]

    \node[circle, draw, minimum size=2em, fill=ngreen!30] (x) at (0, 0) {$\mathbf{x}$};
    \node[smallbox, right=of x, fill=ngreen!30] (ez) {$q_{\phi}(\mathbf{z} \mid \mathbf{x})$};
    \node[box, right=of ez, align=center, fill=ngreen!30] (meanvarphi) {$\mu_{\phi}(\mathbf{x})$ \\ $\log \sigma_{\phi}^2(\mathbf{x})$};
    \node[circle, draw, minimum size=2em, fill=nred!50, right=of meanvarphi] (z) {$\mathbf{z}$};
    \node[smallbox, right=of z, fill=nblue!30] (dz) {$p_{\theta}(\mathbf{x} \mid \mathbf{z})$};
    \node[box, right=of dz, align=center, fill=nblue!30] (meanvartheta) {$G_{\theta}(\mathbf{z})$};
    \node[circle, draw, minimum size=2em, fill=nblue!30, right=of meanvartheta] (xrec) {$\mathbf{x'}$};

    \draw[connector] (x) -- (ez);
    \draw[connector] (ez) -- (meanvarphi);
    \draw[connector, dashed] (meanvarphi) -- (z);
    \draw[connector] (z) -- (dz);
    \draw[connector] (dz) -- (meanvartheta);
    \draw[connector, dashed] (meanvartheta) -- (xrec);

    \draw [decorate,decoration={brace,amplitude=10pt,mirror,raise=5pt},yshift=0pt] (ez.south west) -- (meanvarphi.south east) node [black,midway,yshift=-1cm] {\footnotesize Encoder};
    \draw [decorate,decoration={brace,amplitude=10pt,mirror,raise=5pt},yshift=0pt] (dz.south west) -- (meanvartheta.south east) node [black,midway,yshift=-1cm] {\footnotesize Decoder};

    \node[below=0.2cm of x] (xlabel) {};
    \node[below=0.4cm of z] (zlabel1) {\footnotesize $z = \mu_{\phi}(\mathbf{x}) + \sigma_{\phi}(\mathbf{x}) \odot \varepsilon$};
    \node[below=0.01cm of zlabel1] (zlabel2) {\footnotesize $\varepsilon \sim \mathcal{N}(0, \mathrm{I}_{d_{\mathbf{z}}})$};
    \node[below=0.2cm of xrec] (xreclabel) {};

\end{tikzpicture}
\end{center}
    \vspace{-15pt}
    \caption{Illustration of the Architecture of a VAE with the Multivariate Gaussian.}
    \label{fig:vae}
\end{figure}
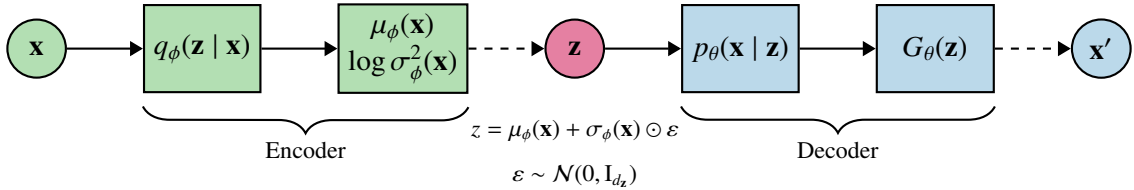

Figure \ref{fig:vae} displays an illustration of a VAE, where both the encoder and the decoder are represented as Gaussian distributions, and the prior over the latent variables is an isotropic multivariate Gaussian, as discussed in Section \ref{sec:deepVAE}. It is important to note that the encoder outputs the logarithm of the  variance instead of the variance to ensure that the variance is always positive, avoiding the need for explicit constraints on the output.
If $z\mapsto q_{\phi}(z|x)$ is a multivariate Gaussian density with a diagonal covariance structure, the reparameterization trick can be expressed as:
$$z = g(\varepsilon, \phi) = \mu_{\phi}(x) + \sigma_{\phi}(x) \odot \varepsilon , \quad \varepsilon \sim \mathcal{N}(0, \mathrm{I}_{d_z})\eqsp.$$
This technique allows the Gaussian VAE to generate latent representations by sampling through the reparameterization trick, thus enabling gradient-based optimization with the pathwise gradient estimator.

\paragraph{Derivation of the Score Function Gradient. }
Proposition~\ref{prop:score_estimator} provides the form of the score function gradient of the expected ELBO with respect to $\phi$.

\begin{Proposition}\label{prop:score_estimator}
For all $\theta \in \Theta$, $\phi \in \Phi$, we have:
$$\nabla^{\SF}_{\phi} \mathcal{L}(\theta, \phi) = \mathbb{E}_{\pi}\left[ \mathbb{E}_{q_{\phi}(\cdot|X)}\left[ \log \frac{p_{\theta}(X, Z)}{q_{\phi}(Z | X)} \nabla_{\phi} \log q_{\phi}(Z | X) \right] \right]\eqsp.$$
\end{Proposition}

\begin{proof}
For a given observation $x \in \Xset$, the score function gradient of the ELBO with respect to $\phi$ is given by:
\begin{align*}
\nabla_{\phi} \mathcal{L}(\theta, \phi; x) &= \nabla_{\phi} \mathbb{E}_{q_{\phi}(\cdot|x)}\left[ \log p_{\theta}(x, Z) - \log q_{\phi}(Z|x)\right] \\
&= \nabla_{\phi} \int (\log p_{\theta}(x, z) - \log q_{\phi}(z|x)) q_{\phi}(z|x) \, \rmd z  \\
&= \int \nabla_{\phi}\left[(\log p_{\theta}(x, z) - \log q_{\phi}(z|x)) q_{\phi}(z|x)\right] \, \rmd z  \\
&= \mathbb{E}_{q_{\phi}(\cdot|x)}\left[ \nabla_{\phi} \log q_{\phi}(Z|x) (\log p_{\theta}(x, Z) - \log q_{\phi}(Z|x)) \right] - \mathbb{E}_{q_{\phi}(\cdot|x)}\left[\nabla_{\phi} \log q_{\phi}(Z|x)\right]\eqsp.
\end{align*}
Using the fact that $\mathbb{E}_{q_{\phi}(\cdot|x)}\left[\nabla_{\phi} \log q_{\phi}(Z|x)\right]=0$ due to the regularity conditions on $q_{\phi}(z|x)$ yields
\begin{align*}
\nabla_{\phi} \mathcal{L}(\theta, \phi; x) &= \mathbb{E}_{q_{\phi}(\cdot|x)}\left[ \nabla_{\phi} \log q_{\phi}(Z|x) (\log p_{\theta}(x, Z) - \log q_{\phi}(Z|x)) \right] \\
&= \mathbb{E}_{q_{\phi}(\cdot|x)}\left[ \log \frac{p_{\theta}(x, Z)}{q_{\phi}(Z|x)} \nabla_{\phi} \log q_{\phi}(Z|x) \right]\eqsp.
\end{align*}
\end{proof}

\paragraph{Score Function and Pathwise Gradient Estimator. }
The gradient estimator of the ELBO for a given batch of observations $\{x_i\}_{i=1}^B$, where $B$ is the batch size, can be computed using Monte Carlo sampling as follows:
\begin{align} \label{eq:grad_estimator}
\widehat{\nabla}_{\theta, \phi} \mathcal{L}(\theta, \phi; \{x_i\}_{i=1}^B) = \frac{1}{B} \sum_{i=1}^{B} \frac{1}{K} \sum_{\ell=1}^{K} \tilde{g}_{i,\ell}\eqsp,
\end{align}
where $K$ denotes the number of samples drawn from the latent space, and $\tilde{g}_{i,\ell}$ is either the gradient estimator via the score function $\tilde{g}^{\text{SF}}_{i,\ell}$ or the pathwise gradient estimator $\tilde{g}^{\PW}_{i,\ell}$.
For the score function gradient estimator, the expression is given by:
\begin{align} \label{eq:grad_estimator_SF}
\tilde{g}^{\SF}_{i,\ell} = \left( \nabla_{\theta} \log \frac{p_{\theta}(x_i, z_i^{(\ell)})}{q_{\phi}(z_i^{(\ell)} | x_i)}, \log \frac{p_{\theta}(x_i, z_i^{(\ell)})}{q_{\phi}(z_i^{(\ell)} | x_i)} \nabla_{\phi} \log \frac{p_{\theta}(x_i, z_i^{(\ell)})}{q_{\phi}(z_i^{(\ell)} | x_i)}\right)^\top\eqsp,
\end{align}
where, for all $1 \leq i \leq B$ and $1 \leq \ell \leq K$, $z_i^{(\ell)}$ are independent samples from $q_{\phi}(\cdot|x_i)$.
The pathwise gradient estimator is given by:
\begin{align} \label{eq:grad_estimator_PW}
\tilde{g}^{\PW}_{i,\ell} = \left( \nabla_{\theta} \log \frac{p_{\theta}(x_i, g(\varepsilon_i^{(\ell)}, \phi))}{q_{\phi}(g(\varepsilon_i^{(\ell)}, \phi) | x_i)}, \nabla_z \log \frac{p_{\theta}(x_i, g(\varepsilon_i^{(\ell)}, \phi))}{q_{\phi}(g(\varepsilon_i^{(\ell)}, \phi) | x_i)} \nabla_{\phi} g(\varepsilon_i^{(\ell)}, \phi) - \nabla_{\phi} \log q_{\phi}(g(\varepsilon_i^{(\ell)}, \phi) | x_i)\right)^\top\eqsp,
\end{align}
where, for all $1 \leq i \leq B$ and $1 \leq \ell \leq K$, $\varepsilon_i^{(\ell)}$ are independent samples from $p_{\varepsilon}$.

\begin{algorithm}
   \caption{Adam Algorithm for ELBO Maximization}
   \label{alg:adam}
\begin{algorithmic}[1]
   \State \textbf{Input:} Initial points $\theta_{0}, \phi_{0}$, maximum number of iterations $n$, step sizes $\{\gamma_{k}\}_{k \geq 1}$, momentum parameters $\beta_{1}, \beta_{2} \in [0,1)$, regularization parameter $\delta \geq 0$ and batch size $B$.
   \State Set $m_{0} = 0$ and $v_{0} = 0$.
   \For{$k=0$ to $n-1$}
   \State Sample a mini-batch $\{x_{i}\}_{i=1}^B$.
   \State Compute the stochastic gradient $g_{k+1} = \widehat{\nabla}_{\theta, \phi} \mathcal{L}(\theta_{k}, \phi_{k}; \{x_i\}_{i=1}^B)$ using \eqref{eq:grad_estimator}.
   \State $m_{k+1} = \beta_{1}m_{k} + (1-\beta_{1})g_{k+1}$.
   \State $v_{k+1} = \beta_{2}v_{k} + (1-\beta_{2}) g_{k+1} \odot g_{k+1} $.
   \State $(\theta_{k+1}, \phi_{k+1}) = (\theta_{k}, \phi_{k}) + \gamma_{k+1} \big[ v_{k+1} + \delta \big]^{-1/2} m_{k+1}$.
   \EndFor
   \State \textbf{Output:} $\left(\theta_{k}, \phi_{k}\right)_{0 \leq k \leq n}$.
\end{algorithmic}
\end{algorithm}

\section{CONVERGENCE PROOFS}
\label{supp:sec:proofs}

\subsection{Proof of Proposition \ref{prop:ELBO_smooth_score_pathwise}}

We divide the proof into two lemmas where Lemma \ref{lemma:ELBO_smooth_score} establishes the smoothness condition when the gradient is computed using the score function, while Lemma \ref{lemma:ELBO_smooth_pathwise} presents the smoothness condition for the pathwise gradient.

\begin{lemma}\label{lemma:ELBO_smooth_score}
Let Assumptions \ref{ass:A1} and \ref{ass:A2}(i) hold. For all 
$\theta, \theta' \in \Theta$ and $\phi, \phi' \in \Phi$,
$$
\begin{aligned}
\left\| \nabla^{\SF}_{\theta, \phi} \mathcal{L}(\theta, \phi) - \nabla^{\SF}_{\theta, \phi} \mathcal{L}(\theta', \phi') \right\| &\leq L^{\SF} \left\| \left(\theta, \phi\right) - \left(\theta', \phi'\right) \right\|\eqsp,
\end{aligned}
$$
where $L^{\SF} = \sup_{\phi \in \Phi} \mathbb{E}_{\pi, \phi}\left[ L_1(x,z) + 2\alpha(x,z) L_2(x,z) + 4M(x,z)^2 + 4\alpha(x,z) M(x,z)^2 \right]$.
\end{lemma}

\begin{proof}
First, for all $\theta,\theta'\in\Theta$, $\phi,\phi'\in\Phi$, we have:
\begin{align}
\left\| \nabla^{\SF}_{\theta, \phi} \mathcal{L}(\theta, \phi) - \nabla^{\SF}_{\theta, \phi} \mathcal{L}(\theta', \phi') \right\| &= \left\| \left(\nabla_{\theta} \mathcal{L}(\theta, \phi) - \nabla_{\theta} \mathcal{L}(\theta', \phi'), \nabla^{\SF}_{\phi} \mathcal{L}(\theta, \phi) - \nabla^{\SF}_{\phi} \mathcal{L}(\theta', \phi') \right)\right\| \nonumber \\
&\leq \left\| \nabla_{\theta} \mathcal{L}(\theta, \phi) - \nabla_{\theta} \mathcal{L}(\theta', \phi') \right\| + \left\| \nabla^{\SF}_{\phi} \mathcal{L}(\theta, \phi) - \nabla^{\SF}_{\phi} \mathcal{L}(\theta', \phi') \right\|\eqsp \label{eq:lip}.
\end{align}

\paragraph{Lipschitz condition of $\nabla_{\theta} \mathcal{L}(\theta, \phi)$. }

\begin{align*}
\left\| \nabla_{\theta} \mathcal{L}(\theta, \phi) - \nabla_{\theta} \mathcal{L}(\theta', \phi') \right\| &\leq \left\| \nabla_{\theta} \mathcal{L}(\theta, \phi) - \nabla_{\theta} \mathcal{L}(\theta', \phi) \right\| + \left\| \nabla_{\theta} \mathcal{L}(\theta', \phi) - \nabla_{\theta} \mathcal{L}(\theta', \phi') \right\|   
\end{align*}
Now, we bound each of these terms individually.
By Assumption \ref{ass:A2}(i),  for all $\theta,\theta'\in\Theta$, $\phi\in\Phi$,
$$
\begin{aligned}
\left\| \nabla_{\theta} \mathcal{L}(\theta, \phi) - \nabla_{\theta} \mathcal{L}(\theta', \phi) \right\| &= \left\| \mathbb{E}_{\pi, \phi}\left[ \nabla_{\theta} \log p_{\theta}(x,z) - \nabla_{\theta} \log p_{\theta'}(x,z) \right] \right\| \\
&\leq \mathbb{E}_{\pi, \phi}\left[ \| \nabla_{\theta} \log p_{\theta}(x,z) - \nabla_{\theta} \log p_{\theta'}(x,z) \| \right] \\
&\leq L^{1}_{dd} \left\| \theta - \theta' \right\|\eqsp,
\end{aligned}
$$
where $L^{1}_{dd} = \mathbb{E}_{\pi, \phi}\left[ L_1(x,z) \right]$. For the second term, we have:
$$
\begin{aligned}
\left\| \nabla_{\theta} \mathcal{L}(\theta', \phi) - \nabla_{\theta} \mathcal{L}(\theta', \phi') \right\| &= \left\| \mathbb{E}_{\pi, \phi}\left[ \nabla_{\theta} \log p_{\theta'}(x,z) \right] - \mathbb{E}_{\pi, \phi'}\left[ \nabla_{\theta} \log p_{\theta'}(x,z) \right] \right\| \\
& \leq \mathbb{E}_{\pi} \left[ \left\| \int \nabla_{\theta} \log p_{\theta'}(x,z) (q_{\phi}(z|x) - q_{\phi'}(z|x)) \, \rmd z \right\| \right] \\
& \leq \mathbb{E}_{\pi} \left[ \int \left\| \nabla_{\theta} \log p_{\theta'}(x,z) (q_{\phi}(z|x) - q_{\phi'}(z|x)) \right\| \, \rmd z \right] \\
& \leq \mathbb{E}_{\pi} \left[ \int M(x,z) \left| q_{\phi}(z|x) - q_{\phi'}(z|x) \right| \, \rmd z \right]\eqsp,
\end{aligned}
$$
where the inequality follows from Assumption \ref{ass:A2}(i). Then, using Assumption \ref{ass:A2}(i) and that for all $x\geq 1$, $x - 1 \leq x \log x \leq \left|x \log x\right|$,
\begin{align*}
\left\| \nabla_{\theta} \mathcal{L}(\theta', \phi) - \nabla_{\theta} \mathcal{L}(\theta', \phi') \right\| &\leq \mathbb{E}_{\pi} \left[ \int M(x,z) \left( \frac{q_{\phi}(z|x)}{q_{\phi'}(z|x)} - 1 \right) 1_{q_{\phi}(z|x) \geq q_{\phi'}(z|x)} q_{\phi'}(z|x) \, \rmd z \right] \\
& \quad + \mathbb{E}_{\pi} \left[ \int M(x,z) \left( \frac{q_{\phi'}(z|x)}{q_{\phi}(z|x)} - 1 \right) 1_{q_{\phi'}(z|x) > q_{\phi}(z|x)} q_{\phi}(z|x) \, \rmd z\right] \\
& \leq \mathbb{E}_{\pi} \left[ \int M(x,z) \left|\frac{q_{\phi}(z|x)}{q_{\phi'}(z|x)} \log \frac{q_{\phi}(z|x)}{q_{\phi'}(z|x)}\right| q_{\phi'}(z|x) \, \rmd z \right] \\
& \quad + \mathbb{E}_{\pi} \left[ \int M(x,z) \left|\frac{q_{\phi'}(z|x)}{q_{\phi}(z|x)} \log \frac{q_{\phi'}(z|x)}{q_{\phi}(z|x)}\right| q_{\phi}(z|x) \, \rmd z \right] \\
& \leq \mathbb{E}_{\pi} \left[ \int M(x,z) \left|\log \frac{q_{\phi}(z|x)}{q_{\phi'}(z|x)}\right| q_{\phi}(z|x) \, \rmd z + \int M(x,z) \left|\log \frac{q_{\phi'}(z|x)}{q_{\phi}(z|x)}\right| q_{\phi'}(z|x) \, \rmd z \right] \\
& \leq \mathbb{E}_{\pi, \phi} \left[ M(x,z) \left|\log q_{\phi}(z|x) - \log q_{\phi'}(z|x)\right| \right] + \mathbb{E}_{\pi, \phi'} \left[ M(x,z) \left|\log q_{\phi}(z|x) - \log q_{\phi'}(z|x)\right| \right] \\
& \leq L^{2}_{dd} \left\|\phi - \phi'\right\|\eqsp,
\end{align*}
where $L^{2}_{dd} = \mathbb{E}_{\pi, \phi} \left[ M(x,z)^2 \right] + \mathbb{E}_{\pi, \phi'} \left[ M(x,z)^2 \right]$.
This completes the proof for the first term in \eqref{eq:lip}.

\paragraph{Lipschitz condition of $\nabla^{\SF}_{\phi} \mathcal{L}(\theta, \phi)$. } 

This case can be treated similarly to the Lipschitz condition with respect to $\theta$. For all $\theta,\theta'\in\Theta$, $\phi,\phi'\in\Phi$,
\begin{align*}
\left\| \nabla^{\SF}_{\phi} \mathcal{L}(\theta, \phi) - \nabla^{\SF}_{\phi} \mathcal{L}(\theta', \phi') \right\| &\leq \left\| \nabla^{\SF}_{\phi} \mathcal{L}(\theta, \phi) - \nabla^{\SF}_{\phi} \mathcal{L}(\theta', \phi) \right\| + \left\| \nabla^{\SF}_{\phi} \mathcal{L}(\theta', \phi) - \nabla^{\SF}_{\phi} \mathcal{L}(\theta', \phi') \right\|   
\end{align*}
For all $\theta,\theta'\in\Theta$, $\phi\in\Phi$, using Assumption \ref{ass:A2}(i), we have:
$$
\begin{aligned}
\left\| \nabla^{\SF}_{\phi} \mathcal{L}(\theta, \phi) - \nabla^{\SF}_{\phi} \mathcal{L}(\theta', \phi) \right\| &= \left\| \mathbb{E}_{\pi, \phi}\left[ \log \frac{p_{\theta}(x, z)}{q_{\phi}(z|x)} \nabla_{\phi} \log q_{\phi}(z|x) -
\log \frac{p_{\theta'}(x, z)}{q_{\phi}(z|x)} \nabla_{\phi} \log q_{\phi}(z|x) \right] \right\| \\ 
&\leq \mathbb{E}_{\pi, \phi}\left[ \left\|\log \frac{p_{\theta}(x, z)}{q_{\phi}(z|x)} \nabla_{\phi} \log q_{\phi}(z|x) - 
\log \frac{p_{\theta'}(x, z)}{q_{\phi}(z|x)} \nabla_{\phi} \log q_{\phi}(z|x) \right\| \right] \\
&\leq \mathbb{E}_{\pi, \phi}\left[ M(x,z) \left\|\log p_{\theta}(x, z) - 
\log p_{\theta'}(x, z)\right\| \right] \\
&\leq \mathbb{E}_{\pi, \phi}\left[ M(x,z)^2 \left\|\theta - \theta'\right\| \right] \\
&\leq L^{1}_{ed} \left\|\theta - \theta'\right\|\eqsp, 
\end{aligned}
$$
where $L^{1}_{ed} = \mathbb{E}_{\pi, \phi}\left[ M(x,z)^2 \right]$. On the other hand, for all $\theta'\in\Theta$, $\phi,\phi'\in\Phi$,
$$
\begin{aligned}
\left\| \nabla^{\SF}_{\phi} \mathcal{L}(\theta', \phi) - \nabla^{\SF}_{\phi} \mathcal{L}(\theta', \phi') \right\| &= \left\| \mathbb{E}_{\pi, \phi}\left[ \log \frac{p_{\theta'}(x, z)}{q_{\phi}(z|x)} \nabla_{\phi} \log q_{\phi}(z|x) \right] - \mathbb{E}_{\pi, \phi'}\left[ 
\log \frac{p_{\theta'}(x, z)}{q_{\phi'}(z|x)} \nabla_{\phi} \log q_{\phi'}(z|x) \right] \right\| \\ 
&\leq A_1 + A_2\eqsp,
\end{aligned}
$$
where
\begin{align*}
    A_1 &= \left\| \mathbb{E}_{\pi, \phi}\left[ \log \frac{p_{\theta'}(x, z)}{q_{\phi}(z|x)} \nabla_{\phi} \log q_{\phi}(z|x) \right] - \mathbb{E}_{\pi, \phi}\left[ \log \frac{p_{\theta'}(x, z)}{q_{\phi'}(z|x)} \nabla_{\phi} \log q_{\phi'}(z|x) \right] \right\|,\\
    A_2 &= \left\| \mathbb{E}_{\pi, \phi}\left[ \log \frac{p_{\theta'}(x, z)}{q_{\phi'}(z|x)} \nabla_{\phi} \log q_{\phi'}(z|x) \right] - \mathbb{E}_{\pi, \phi'}\left[ \log \frac{p_{\theta'}(x, z)}{q_{\phi'}(z|x)} \nabla_{\phi} \log q_{\phi'}(z|x) \right] \right\|.
\end{align*}
By decomposing again $A_1$, we get that
\begin{align*}
A_1 &\leq \mathbb{E}_{\pi, \phi}\left[ \left\| \log \frac{p_{\theta'}(x, z)}{q_{\phi}(z|x)} \nabla_{\phi} \log q_{\phi}(z|x) - \log \frac{p_{\theta'}(x, z)}{q_{\phi'}(z|x)} \nabla_{\phi} \log q_{\phi'}(z|x) \right\| \right]\\
& \leq \mathbb{E}_{\pi, \phi}\left[ \left\| \log p_{\theta'}(x, z) (\nabla_{\phi} \log q_{\phi}(z|x) - \nabla_{\phi} \log q_{\phi'}(z|x) \right\| \right] \\
& \quad + \mathbb{E}_{\pi, \phi}\left[ \left\| \log q_{\phi}(z|x) \nabla_{\phi} \log q_{\phi}(z|x) - \log q_{\phi'}(z|x) \nabla_{\phi} \log q_{\phi'}(z|x) \right\| \right] \\
& \leq \mathbb{E}_{\pi, \phi}\left[ \left\| \log p_{\theta'}(x, z) (\nabla_{\phi} \log q_{\phi}(z|x) - \nabla_{\phi} \log q_{\phi'}(z|x) \right\| \right] \\
& \quad + \mathbb{E}_{\pi, \phi}\left[ \left| \log q_{\phi}(z|x) \right| \left\| \nabla_{\phi} \log q_{\phi}(z|x) - \nabla_{\phi} \log q_{\phi'}(z|x) \right\| \right] \\
& \quad + \mathbb{E}_{\pi, \phi}\left[ \left\| \nabla_{\phi} \log q_{\phi'}(z|x) \right\| \left\| \log q_{\phi}(z|x) - \log q_{\phi'}(z|x) \right\| \right]\eqsp.
\end{align*}
Then, using the Mean Value Theorem along with Assumptions \ref{ass:A1} and \ref{ass:A2}(i),
\begin{equation*}
A_1  \leq 2\mathbb{E}_{\pi, \phi}\left[ \alpha(x,z) L_2(x,z) \right] \left\|\phi - \phi'\right\| + \mathbb{E}_{\pi, \phi}\left[ M(x,z)^2 \right] \left\|\phi - \phi'\right\|\eqsp.
\end{equation*}
For the term $A_2$, note that
\begin{align*}
A_2 & \leq \mathbb{E}_{\pi} \left[ \left\| \int \log \frac{p_{\theta'}(x, z)}{q_{\phi'}(z|x)} \nabla_{\phi} \log q_{\phi'}(z|x) (q_{\phi}(z|x) - q_{\phi'}(z|x)) \, \rmd z \right\| \right] \\
& \leq \mathbb{E}_{\pi} \left[ \int \left\| \log \frac{p_{\theta'}(x, z)}{q_{\phi'}(z|x)} \nabla_{\phi} \log q_{\phi'}(z|x) (q_{\phi}(z|x) - q_{\phi'}(z|x)) \right\| \, \rmd z \right] \\
& \leq 2\mathbb{E}_{\pi} \left[ \int \alpha(x,z) M(x,z) \left| q_{\phi}(z|x) - q_{\phi'}(z|x) \right| \, \rmd z \right],
\end{align*}
where the inequality follows from Assumptions \ref{ass:A1} and \ref{ass:A2}(i). Then, using Assumption \ref{ass:A2}(i) and that for all $x\geq 1$, $x - 1 \leq x \log x \leq \left|x \log x\right|$,
\begin{align*}
A_2 &\leq 2\mathbb{E}_{\pi} \left[ \int \alpha(x,z) M(x,z) \left( \frac{q_{\phi}(z|x)}{q_{\phi'}(z|x)} - 1 \right) 1_{q_{\phi}(z|x) \leq q_{\phi'}(z|x)} q_{\phi'}(z|x) \, \rmd z \right] \\
& \quad + 2\mathbb{E}_{\pi} \left[ \int \alpha(x,z) M(x,z) \left( \frac{q_{\phi'}(z|x)}{q_{\phi}(z|x)} - 1 \right) 1_{q_{\phi'}(z|x) > q_{\phi}(z|x)} q_{\phi}(z|x) \, \rmd z\right] \\
& \leq 2\mathbb{E}_{\pi} \left[ \int \alpha(x,z) M(x,z) \left|\frac{q_{\phi}(z|x)}{q_{\phi'}(z|x)} \log \frac{q_{\phi}(z|x)}{q_{\phi'}(z|x)}\right| q_{\phi'}(z|x) \, \rmd z \right] \\
& \quad + 2\mathbb{E}_{\pi} \left[ \int \alpha(x,z) M(x,z) \left|\frac{q_{\phi'}(z|x)}{q_{\phi}(z|x)} \log \frac{q_{\phi'}(z|x)}{q_{\phi}(z|x)}\right| q_{\phi}(z|x) \, \rmd z \right] \\
& \leq 2\mathbb{E}_{\pi} \left[ \int \alpha(x,z) M(x,z) \left|\log \frac{q_{\phi}(z|x)}{q_{\phi'}(z|x)}\right| q_{\phi}(z|x) \, dz + \int \alpha(x,z) M(x,z) \left|\log \frac{q_{\phi'}(z|x)}{q_{\phi}(z|x)}\right| q_{\phi'}(z|x) \, \rmd z \right] \\
& \leq 2 \mathbb{E}_{\pi, \phi} \left[ \alpha(x,z) M(x,z) \left|\log q_{\phi}(z|x) - \log q_{\phi'}(z|x)\right| \right] + 2 \mathbb{E}_{\pi, \phi'} \left[ \alpha(x,z) M(x,z) \left|\log q_{\phi}(z|x) - \log q_{\phi'}(z|x)\right| \right] \\
& \leq 2 \left( \mathbb{E}_{\pi, \phi} \left[ \alpha(x,z) M(x,z)^2 \right] + \mathbb{E}_{\pi, \phi'} \left[ \alpha(x,z) M(x,z)^2 \right] \right) \left\|\phi - \phi'\right\|\eqsp.
\end{align*}
By combining these two terms, we obtain:
$$\left\| \nabla^{\SF}_{\phi} \mathcal{L}(\theta, \phi) - \nabla^{\SF}_{\phi} \mathcal{L}(\theta, \phi') \right\| \leq L_{ed} \left\|\phi - \phi'\right\|,$$
where $L^{2}_{ed} = \mathbb{E}_{\pi, \phi}\left[ 2\alpha(x,z) L_2(x,z) +  M(x,z)^2 + 2\alpha(x,z) M(x,z)^2 \right] + 2\mathbb{E}_{\pi, \phi'}\left[ \alpha(x,z) M(x,z)^2 \right]$ and concludes the argument for the second term in \eqref{eq:lip}.

Then,
\begin{align*}
\left\| \nabla^{\SF}_{\theta, \phi} \mathcal{L}(\theta, \phi) - \nabla^{\SF}_{\theta, \phi} \mathcal{L}(\theta', \phi') \right\| &\leq \left(L^{1}_{dd} + L^{1}_{ed}\right) \left\| \theta - \theta' \right\| + \left(L^{2}_{dd} + L^{2}_{ed} \right) \left\|\phi - \phi'\right\|\eqsp, \\
&\leq L^{\SF} \left\| \left(\theta, \phi\right) - \left(\theta', \phi'\right) \right\|\eqsp,
\end{align*}
where $L^{\SF} = \sup_{\phi \in \Phi} \mathbb{E}_{\pi, \phi}\left[ L_1(x,z) + 2\alpha(x,z) L_2(x,z) + 4M(x,z)^2 + 4\alpha(x,z) M(x,z)^2 \right]$.
\end{proof}

\begin{lemma}\label{lemma:ELBO_smooth_pathwise}
Let Assumptions \ref{ass:A2}(ii) and \ref{ass:A2_pathwise} hold. For all 
$\theta, \theta' \in \Theta$ and $\phi, \phi' \in \Phi$,
$$
\begin{aligned}
\| \nabla^{\PW}_{\theta, \phi} \mathcal{L}(\theta, \phi) - \nabla^{\PW}_{\theta, \phi} \mathcal{L}(\theta', \phi') \| &\leq L^{\PW} \| (\theta, \phi) - (\theta', \phi') \|\eqsp,
\end{aligned}
$$
where $L^{\PW} = \mathbb{E}_{\pi, p_\varepsilon}\left[ L_{p}(x,\varepsilon) + M_{g}(x,\varepsilon)^2 \left(L_{p}(x,\varepsilon) + 2L_{q}(x,\varepsilon) \right) + 3 L_{g}(x,\varepsilon) M(x,\varepsilon) + 2M_{g}(x,\varepsilon) L_{q}(x,\varepsilon) \right] + \mathbb{E}_{\pi, p_\varepsilon}\left[ L_{p}(x,\varepsilon) M_{g}(x,\varepsilon) \right]$.
\end{lemma}

\begin{proof}
We proceed by dividing the proof into two cases, following the same structure as in the proof of Lemma \ref{lemma:ELBO_smooth_score}. We establish that the following inequalities hold: for all $\theta,\theta'\in\Theta$, $\phi,\phi'\in\Phi$,
$$
\begin{aligned}
\left\| \nabla_{\theta} \mathcal{L}(\theta, \phi) - \nabla_{\theta} \mathcal{L}(\theta', \phi') \right\| &\leq L^{1}_{dd} \left\|\theta - \theta'\right\| + L^{2}_{dd}\left\|\phi - \phi'\right\|\eqsp, \\
\left\| \nabla^{\PW}_{\phi} \mathcal{L}(\theta, \phi) - \nabla^{\PW}_{\phi} \mathcal{L}(\theta', \phi') \right\| &\leq L^{1}_{ed} \left\|\theta - \theta'\right\| + L^{2}_{ed}\left\|\phi - \phi'\right\|\eqsp,
\end{aligned}
$$
where $L^{1}_{dd} = \mathbb{E}_{\pi, p_\varepsilon}\left[ L_{p}(x,\varepsilon)\right]$, $L^{2}_{dd} = \mathbb{E}_{\pi, p_\varepsilon}\left[ L_{p}(x,\varepsilon) M_{g}(x,\varepsilon) \right]$, $L^{1}_{ed} = \mathbb{E}_{\pi, p_\varepsilon}\left[ M_{g}(x,\varepsilon) L_{p}(x,\varepsilon)\right]$ and $\\L^{2}_{ed} = \mathbb{E}_{\pi, p_\varepsilon}\left[ M_{g}(x,\varepsilon)^2 \left(L_{p}(x,\varepsilon) + 2L_{q}(x,\varepsilon) \right) + 3L_{g}(x,\varepsilon) M(x,\varepsilon) + 2M_{g}(x,\varepsilon) L_{q}(x,\varepsilon) \right]$.
\paragraph{Lipschitz condition of $\nabla_{\theta} \mathcal{L}(\theta, \phi)$. } 
\begin{align} \label{eq:lip_pathwise_theta}
\left\| \nabla_{\theta} \mathcal{L}(\theta, \phi) - \nabla_{\theta} \mathcal{L}(\theta', \phi') \right\| &\leq \left\| \nabla_{\theta} \mathcal{L}(\theta, \phi) - \nabla_{\theta} \mathcal{L}(\theta', \phi) \right\| + \left\| \nabla_{\theta} \mathcal{L}(\theta', \phi) - \nabla_{\theta} \mathcal{L}(\theta', \phi') \right\|   
\end{align}
Now, we bound each of these terms individually.
By Assumption \ref{ass:A2}(ii),  for all $\theta,\theta'\in\Theta$, $\phi\in\Phi$,
\begin{align*}
\left\| \nabla_{\theta} \mathcal{L}(\theta, \phi) - \nabla_{\theta} \mathcal{L}(\theta', \phi) \right\| &= \left\| \mathbb{E}_{\pi, p_\varepsilon}\left[ \nabla_{\theta} \log p_{\theta}(x, g(\varepsilon, \phi)) - \nabla_{\theta} \log p_{\theta'}(x, g(\varepsilon, \phi)) \right] \right\| \\
&\leq \mathbb{E}_{\pi, p_\varepsilon}\left[ \| \nabla_{\theta} \log p_{\theta}(x, g(\varepsilon, \phi)) - \nabla_{\theta} \log p_{\theta'}(x, g(\varepsilon, \phi)) \| \right] \\
&\leq \mathbb{E}_{\pi, p_\varepsilon}\left[ L_{p}(x,\varepsilon) \right] \left\| \theta - \theta' \right\|\eqsp,
\end{align*}
which concludes the bound for the first term in \eqref{eq:lip_pathwise_theta}. For the second term, we have:
\begin{align*}
\left\| \nabla_{\theta} \mathcal{L}(\theta', \phi) - \nabla_{\theta} \mathcal{L}(\theta', \phi') \right\| &= \left\| \mathbb{E}_{\pi, p_\varepsilon}\left[ \nabla_{\theta} \log p_{\theta'}(x, g(\varepsilon, \phi)) - \nabla_{\theta} \log p_{\theta'}(x, g(\varepsilon, \phi')) \right] \right\| \\
&\leq \mathbb{E}_{\pi, p_\varepsilon}\left[ \| \nabla_{\theta} \log p_{\theta'}(x, g(\varepsilon, \phi)) - \nabla_{\theta} \log p_{\theta'}(x, g(\varepsilon, \phi')) \| \right]\eqsp.
\end{align*}
Since for all $z, z' \in \Zset$, $\left\| \nabla_{\theta} \log p_{\theta'}(x, z) - \nabla_{\theta} \log p_{\theta'}(x, z') \right\| \leq L_{p}(x,\varepsilon) \left\| z - z' \right\|$, it follows that:
$$\left\| \nabla_{\theta} \log p_{\theta'}(x, g(\varepsilon, \phi)) - \nabla_{\theta} \log p_{\theta'}(x, g(\varepsilon, \phi')) \right\| \leq L_{p}(x,\varepsilon) \left\| g(\varepsilon, \phi) - g(\varepsilon, \phi') \right\|\eqsp.$$
Therefore,
\begin{align*}
\left\| \nabla_{\theta} \mathcal{L}(\theta', \phi) - \nabla_{\theta} \mathcal{L}(\theta', \phi') \right\| &\leq \mathbb{E}_{\pi, p_\varepsilon}\left[ L_{p}(x,\varepsilon) \| g(\varepsilon, \phi) - g(\varepsilon, \phi') \| \right] \\
&\leq \mathbb{E}_{\pi, p_\varepsilon}\left[ L_{p}(x,\varepsilon) M_{g}(x,\varepsilon) \right] \left\| \phi - \phi' \right\|\eqsp,
\end{align*}
which concludes the Lipschitz condition of $\nabla_{\theta} \mathcal{L}(\theta, \phi)$.

\paragraph{Lipschitz condition of $\nabla^{\PW}_{\phi} \mathcal{L}(\theta, \phi)$. } 
For all $\theta, \theta' \in \Theta$ and $\phi, \phi' \in \Phi$, we have the following inequality:
\begin{align} \label{eq:lip_pathwise_phi}
\left\| \nabla^{\PW}_{\phi} \mathcal{L}(\theta, \phi) - \nabla^{\PW}_{\phi} \mathcal{L}(\theta', \phi') \right\| &\leq \left\| \nabla^{\PW}_{\phi} \mathcal{L}(\theta, \phi) - \nabla^{\PW}_{\phi} \mathcal{L}(\theta', \phi) \right\| + \left\| \nabla^{\PW}_{\phi} \mathcal{L}(\theta', \phi) - \nabla^{\PW}_{\phi} \mathcal{L}(\theta', \phi') \right\|\eqsp.   
\end{align}
We now handle each term separately. For the first term, we have:
\begin{align*}
\left\| \nabla^{\PW}_{\phi} \mathcal{L}(\theta, \phi) - \nabla^{\PW}_{\phi} \mathcal{L}(\theta', \phi) \right\| &\leq \mathbb{E}_{\pi, p_\varepsilon}\left[ \left\|\nabla_{\phi} g(\varepsilon, \phi) \right\| \left\| \nabla_z \log p_{\theta}(x, g(\varepsilon, \phi)) - \nabla_z \log p_{\theta'}(x, g(\varepsilon, \phi)) \right\| \right] \\
&\leq \mathbb{E}_{\pi, p_\varepsilon}\left[ M_{g}(x,\varepsilon) L_{p}(x,\varepsilon)\right] \left\|\theta - \theta'\right\|\eqsp,
\end{align*}
which concludes the bound for the first term in \eqref{eq:lip_pathwise_phi}. For the second term in \eqref{eq:lip_pathwise_phi}, applying the triangle inequality, we have:
\begin{align*}
\left\| \nabla^{\PW}_{\phi} \mathcal{L}(\theta', \phi) - \nabla^{\PW}_{\phi} \mathcal{L}(\theta', \phi') \right\|
&\leq A_1 + A_2 + A_3\eqsp,
\end{align*}
where
\begin{align*}
A_1 &= \mathbb{E}_{\pi, p_\varepsilon}\left[ \left\| \nabla_z \log p_{\theta'}(x, g(\varepsilon, \phi) \nabla_{\phi} g(\varepsilon, \phi) - \nabla_z \log p_{\theta'}(x, g(\varepsilon, \phi')) \nabla_{\phi} g(\varepsilon, \phi') \right\| \right]\eqsp, \\ 
A_2 &= \mathbb{E}_{\pi, p_\varepsilon}\left[ \left\| \nabla_z \log q_{\phi}(g(\varepsilon, \phi)|x) \nabla_{\phi} g(\varepsilon, \phi) - \nabla_z \log q_{\phi'}(g(\varepsilon, \phi')|x) \nabla_{\phi} g(\varepsilon, \phi') \right\| \right]\eqsp, \\ 
A_3 &= \mathbb{E}_{\pi, p_\varepsilon}\left[ \left\| \nabla_{\phi} \log q_{\phi}(z|x) - \nabla_{\phi} \log q_{\phi'}(z|x)  \right\| \right]\eqsp.
\end{align*}
For $A_1$, using the boundedness of the gradient with respect to $z$ and the smoothness of $g$ (Assumptions \ref{ass:A2}(ii) and \ref{ass:A2_pathwise}), we obtain:
\begin{align*}
A_1 &= \mathbb{E}_{\pi, p_\varepsilon}\left[ \left\| \nabla_z \log p_{\theta'}(x, g(\varepsilon, \phi)) \nabla_{\phi} g(\varepsilon, \phi) - \nabla_z \log p_{\theta'}(x, g(\varepsilon, \phi')) \nabla_{\phi} g(\varepsilon, \phi') \right\| \right] \\
&\leq \mathbb{E}_{\pi, p_\varepsilon}\left[ \left\| \nabla_z \log p_{\theta'}(x, g(\varepsilon, \phi)) \nabla_{\phi} g(\varepsilon, \phi) - \nabla_z \log p_{\theta'}(x, g(\varepsilon, \phi')) \nabla_{\phi} g(\varepsilon, \phi) \right\| \right] \\
&\quad + \mathbb{E}_{\pi, p_\varepsilon}\left[ \left\| \nabla_z \log p_{\theta'}(x, g(\varepsilon, \phi')) \nabla_{\phi} g(\varepsilon, \phi) - \nabla_z \log p_{\theta'}(x, g(\varepsilon, \phi')) \nabla_{\phi} g(\varepsilon, \phi') \right\| \right] \\
&\leq \mathbb{E}_{\pi, p_\varepsilon}\left[ \left\|\nabla_{\phi} g(\varepsilon, \phi)\right\| \left\| \nabla_z \log p_{\theta'}(x, g(\varepsilon, \phi)) - \nabla_z \log p_{\theta'}(x, g(\varepsilon, \phi')) \right\| \right] \\
&\quad + \mathbb{E}_{\pi, p_\varepsilon}\left[ \left\| \nabla_z \log p_{\theta'}(x, g(\varepsilon, \phi')  \right\| \left\| \nabla_{\phi} g(\varepsilon, \phi) - \nabla_{\phi} g(\varepsilon, \phi') \right\| \right]\eqsp.
\end{align*}
Since for all $z, z' \in \Zset$, $\left\| \nabla_z \log p_{\theta'}(x, z) - \nabla_z \log p_{\theta'}(x, z') \right\| \leq L_{p}(x,\varepsilon) \left\| z - z' \right\|$, it follows that:
$$\left\| \nabla_z \log p_{\theta'}(x, g(\varepsilon, \phi)) - \nabla_z \log p_{\theta'}(x, g(\varepsilon, \phi')) \right\| \leq L_{p}(x,\varepsilon) \left\| g(\varepsilon, \phi) - g(\varepsilon, \phi') \right\|\eqsp.$$
Therefore,
\begin{align*}
A_1 &\leq \mathbb{E}_{\pi, p_\varepsilon}\left[ M_{g}(x,\varepsilon) L_{p}(x,\varepsilon) \left\| g(\varepsilon, \phi) - g(\varepsilon, \phi') \right\| \right] +\mathbb{E}_{\pi, p_\varepsilon}\left[ M(x,\varepsilon) \left\| \nabla_{\phi} g(\varepsilon, \phi) - \nabla_{\phi} g(\varepsilon, \phi') \right\| \right] \\
&\leq \mathbb{E}_{\pi, p_\varepsilon}\left[ M_{g}(x,\varepsilon)^2 L_{p}(x,\varepsilon) + M(x,\varepsilon) L_{g}(x,\varepsilon) \right] \left\| \phi -\phi' \right\|\eqsp.
\end{align*}
For $A_2$, we have:
\begin{align*}
A_2 &= \mathbb{E}_{\pi, p_\varepsilon}\left[ \left\| \nabla_z \log q_{\phi}(g(\varepsilon, \phi)|x) \nabla_{\phi} g(\varepsilon, \phi) - \nabla_z \log q_{\phi'}(g(\varepsilon, \phi')|x) \nabla_{\phi} g(\varepsilon, \phi') \right\| \right] \\
&\leq \mathbb{E}_{\pi, p_\varepsilon}\left[ \left\| \nabla_z \log q_{\phi}(g(\varepsilon, \phi)|x) \nabla_{\phi} g(\varepsilon, \phi) - \nabla_z \log q_{\phi'}(g(\varepsilon, \phi)|x) \nabla_{\phi} g(\varepsilon, \phi) \right\| \right] \\
& \quad + \mathbb{E}_{\pi, p_\varepsilon}\left[ \left\| \nabla_z \log q_{\phi'}(g(\varepsilon, \phi)|x) \nabla_{\phi} g(\varepsilon, \phi) - \nabla_z \log q_{\phi'}(g(\varepsilon, \phi')|x) \nabla_{\phi} g(\varepsilon, \phi) \right\| \right] \\
& \quad + \mathbb{E}_{\pi, p_\varepsilon}\left[ \left\| \nabla_z \log q_{\phi'}(g(\varepsilon, \phi')|x) \nabla_{\phi} g(\varepsilon, \phi) - \nabla_z \log q_{\phi'}(g(\varepsilon, \phi')|x) \nabla_{\phi} g(\varepsilon, \phi') \right\| \right] \\
&\leq \mathbb{E}_{\pi, p_\varepsilon}\left[ \left\| \nabla_{\phi} g(\varepsilon, \phi) \right\| \left\| \nabla_z \log q_{\phi}(g(\varepsilon, \phi)|x) - \nabla_z \log q_{\phi'}(g(\varepsilon, \phi)|x) \right\| \right] \\
& \quad + \mathbb{E}_{\pi, p_\varepsilon}\left[ \left\| \nabla_{\phi} g(\varepsilon, \phi) \right\| \left\| \nabla_z \log q_{\phi'}(g(\varepsilon, \phi)|x) - \nabla_z \log q_{\phi'}(g(\varepsilon, \phi')|x) \right\| \right] \\
& \quad + \mathbb{E}_{\pi, p_\varepsilon}\left[ \left\| \nabla_z \log q_{\phi'}(g(\varepsilon, \phi')|x) \right\| \left\| \nabla_{\phi} g(\varepsilon, \phi) - \nabla_{\phi} g(\varepsilon, \phi') \right\| \right]\eqsp.
\end{align*}
Since for all $z, z' \in \Zset$, $\left\| \nabla_z \log q_{\phi'}(z|x) - \nabla_z \log q_{\phi'}(z'|x) \right\| \leq L_{q}(x,\varepsilon) \left\| z - z' \right\|$, it follows that:
$$\left\| \nabla_z \log q_{\phi'}(g(\varepsilon, \phi)|x) - \nabla_z \log q_{\phi'}(g(\varepsilon, \phi')|x) \right\| \leq L_{q}(x,\varepsilon) \left\| g(\varepsilon, \phi) - g(\varepsilon, \phi') \right\|\eqsp.$$
Therefore,
\begin{align*}
A_2 &\leq \mathbb{E}_{\pi, p_\varepsilon}\left[ M_{g}(x,\varepsilon) L_{q}(x,\varepsilon) \left\| \phi - \phi' \right\| \right] \\
& \quad + \mathbb{E}_{\pi, p_\varepsilon}\left[ M_{g}(x,\varepsilon) L_{q}(x,\varepsilon) \left\| g(\varepsilon, \phi) - g(\varepsilon, \phi') \right\| \right] \\
& \quad + \mathbb{E}_{\pi, p_\varepsilon}\left[ M_{q}(x,\varepsilon) \left\| \nabla_{\phi} g(\varepsilon, \phi) - \nabla_{\phi} g(\varepsilon, \phi') \right\| \right] \\
&\leq \mathbb{E}_{\pi, p_\varepsilon}\left[ M_{g}(x,\varepsilon) L_{q}(x,\varepsilon) + M_{g}(x,\varepsilon)^2 L_{q}(x,\varepsilon) + M(x,\varepsilon) L_{g}(x,\varepsilon) \right] \left\| \phi - \phi' \right\|\eqsp.
\end{align*}
For  $A_3$, following the same procedure as for the term $A_2$ we get:
\begin{align*}
A_3 &= \mathbb{E}_{\pi, p_\varepsilon}\left[ \left\| \nabla_{\phi} \log q_{\phi}(g(\varepsilon, \phi)|x) - \nabla_{\phi} \log q_{\phi'}(g(\varepsilon, \phi')|x)  \right\| \right] \\
&= \mathbb{E}_{\pi, p_\varepsilon}\left[ \left\| \nabla_{z} \log q_{\phi}(g(\varepsilon, \phi)|x) \nabla_{\phi} g(\varepsilon, \phi) - \nabla_{z} \log q_{\phi'}(g(\varepsilon, \phi')|x) \nabla_{\phi} g(\varepsilon, \phi') \right\| \right] \\
&\leq \mathbb{E}_{\pi, p_\varepsilon}\left[ M_{g}(x,\varepsilon) L_{q}(x,\varepsilon) + M_{g}(x,\varepsilon)^2 L_{q}(x,\varepsilon) + M(x,\varepsilon) L_{g}(x,\varepsilon) \right] \left\| \phi - \phi' \right\|\eqsp,
\end{align*}
which concludes the bound for the second term in \eqref{eq:lip_pathwise_phi}, thereby completing the proof.
\end{proof}

\subsection{Proof of Theorem \ref{th:rate_general_SGD}}

\begin{proof}
We address both the score function and pathwise gradient cases at the same time.
Our proof is adapted from Theorem 2.1 in \cite{ghadimi2013stochastic},  with modifications to account for the two variables and the dual randomness arising from both the data and the latent variables. Using the smoothness of $\mathcal{L}$ by Proposition \ref{prop:ELBO_smooth_score_pathwise} (Descent Lemma \citep{ortega2000iterative, nesterov2013introductory} for maximization), for $\idx \in \{\SF, \PW\}$, we have:
\begin{align*}
\mathcal{L}\left(\theta_{k+1}, \phi_{k+1}\right) & \geq \mathcal{L}\left(\theta_{k}, \phi_{k}\right) + \left\langle\nabla^{\idx}_{\theta, \phi} \mathcal{L}\left(\theta_{k}, \phi_{k}\right) , \left(\theta_{k+1}, \phi_{k+1}\right) - \left(\theta_{k}, \phi_{k}\right) \right\rangle - \frac{L}{2}\left\|\left(\theta_{k+1}, \phi_{k+1}\right) - \left(\theta_{k}, \phi_{k}\right) \right\|^{2} \\
& \geq \mathcal{L}\left(\theta_{k}, \phi_{k}\right) + \gamma_{n+1} \left\langle\nabla^{\idx}_{\theta, \phi} \mathcal{L}\left(\theta_{k}, \phi_{k}\right) , \widehat{\nabla}^{\idx}_{\theta, \phi} \mathcal{L}\left(\theta_{k}, \phi_{k}; \mathcal{D}_{k+1}\right) \right\rangle - \frac{L}{2} \gamma_{k+1}^{2} \left\|\widehat{\nabla}^{\idx}_{\theta, \phi} \mathcal{L}\left(\theta_{k}, \phi_{k}; \mathcal{D}_{k+1}\right) \right\|^{2}\eqsp,
\end{align*}
where $\mathcal{D}_{k+1}$ corresponds to the mini-batch of data used to compute the gradient estimator at iteration $k+1$. For all $k \geq 0$, let $\mathcal{F}_{k} = \sigma \left( \theta_{0}, \{\mathcal{D}_{i}, Z_{i}\}_{1 \leq i \leq k } \right)$, where $\mathcal{D}_{i}$ denotes the mini-batch of data used at iteration $i$, and $Z_{i}$ represents all latent samples drawn at iteration $i$. Taking the conditional expectation with respect to the filtration $(\mathcal{F}_{k})_{k\geq 0}$, 
\begin{align*}
\mathbb{E}\left[\mathcal{L}\left(\theta_{k+1}, \phi_{k+1}\right) \mid \mathcal{F}_{k}\right] &\geq \mathcal{L}\left(\theta_{k}, \phi_{k}\right) + \gamma_{k+1} \left\langle \nabla^{\idx}_{\theta, \phi} \mathcal{L}\left(\theta_{k}, \phi_{k}\right) , \mathbb{E}\left[ \widehat{\nabla}^{\idx}_{\theta, \phi} \mathcal{L}\left(\theta_{k}, \phi_{k}; \mathcal{D}_{k+1}\right) \mid \mathcal{F}_{k}\right] \right\rangle \\
& \quad- \frac{L\gamma_{k+1}^{2}}{2} \|\nabla^{\idx}_{\theta, \phi} \mathcal{L}\left(\theta_{k}, \phi_{k}\right)\|^{2} - \frac{L\gamma_{k+1}^{2}}{2} \mathbb{E}\left[ \|\widehat{\nabla}^{\idx}_{\theta, \phi} \mathcal{L}(\theta_{k}, \phi_{k}; \mathcal{D}_{k+1})) - \nabla^{\idx}_{\theta, \phi} \mathcal{L}(\theta_{k}, \phi_{k})\|^2 \mid \mathcal{F}_{k}\right] \eqsp.
\end{align*}
Given the assumption on the variance of the gradient estimator, we obtain:
\begin{align*}
\mathbb{E}\left[\left\|\widehat{\nabla}^{\idx}_{\theta, \phi} \mathcal{L}(\theta_{k}, \phi_{k}; \mathcal{D}_{k+1}) - \nabla^{\idx}_{\theta, \phi} \mathcal{L}(\theta_{k}, \phi_{k})\right\|^2 \mid \mathcal{F}_{k}\right] &= \mathbb{E}\left[\left\|\frac{1}{B} \sum_{i=1}^{B} \frac{1}{K} \sum_{\ell=1}^{K} \tilde{g}^{\idx}_{i,\ell} - \nabla^{\idx}_{\theta, \phi} \mathcal{L}(\theta_{k}, \phi_{k})\right\|^2 \mid \mathcal{F}_{k}\right] \\
&\leq \frac{1}{BK} \mathbb{E}\left[\left\| \tilde{g}^{\idx}_{i,\ell} - \nabla^{\idx}_{\theta, \phi} \mathcal{L}(\theta_{k}, \phi_{k})\right\|^2 \mid \mathcal{F}_{k}\right] \\
&\leq \frac{\sigma^2}{BK} \eqsp.
\end{align*}
Using this result, and noting that $\widehat{\nabla}^{\idx}_{\theta, \phi} \mathcal{L}(\theta, \phi; \mathcal{D}_{k+1})$ is an unbiased estimator of $\nabla^{\idx}_{\theta, \phi} \mathcal{L}\left(\theta, \phi\right)$, we get:
\begin{align*}
\mathbb{E}\left[ \mathcal{L}\left(\theta_{k+1}, \phi_{k+1}\right) \mid \mathcal{F}_{k}\right]  &\geq \mathcal{L}\left(\theta_{k}, \phi_{k}\right) + \gamma_{k+1}\|\nabla^{\idx}_{\theta, \phi} \mathcal{L}\left(\theta_{k}, \phi_{k}\right)\|^{2} - \frac{L\gamma_{k+1}^{2}}{2} \|\nabla^{\idx}_{\theta, \phi} \mathcal{L}\left(\theta_{k}, \phi_{k}\right)\|^{2} - \frac{L\gamma_{k+1}^{2}}{2} \frac{\sigma^2}{BK} \eqsp.
\end{align*}
Therefore,
\begin{align*}
\sum_{k=0}^{n} \left(\gamma_{k+1} - \frac{L \gamma_{k+1}^2}{2}\right) \mathbb{E}\left[\left\| \nabla^{\idx}_{\theta, \phi} \mathcal{L}\left(\theta_{k}, \phi_{k}\right)\right\|^{2}\right] &\leq \sum_{k=0}^{n} \mathbb{E} [\mathcal{L}\left(\theta_{k+1}, \phi_{k+1}\right) - \mathcal{L}\left(\theta_{k}, \phi_{k}\right)] + \frac{L \sigma^{2}}{2BK} \sum_{k=0}^{n} \gamma_{k+1}^{2}\eqsp, \\
&\leq \mathbb{E} [\mathcal{L}\left(\theta_{n+1}, \phi_{n+1}\right)] - \mathcal{L}\left(\theta_{0}, \phi_{0}\right) + \frac{L \sigma^{2}}{2BK} \sum_{k=0}^{n} \gamma_{k+1}^{2}\eqsp.
\end{align*}
Consequently, by the definition of the discrete random variable $R$ and choosing $\gamma_n = n^{-1/2}$, 
\begin{align*}
\mathbb{E}\left[\left\| \nabla^{\idx}_{\theta, \phi} \mathcal{L}\left(\theta_{R}, \phi_{R}\right)\right\|^{2}\right] &= \frac{1}{n}\sum_{k=0}^{n} \mathbb{E}\left[\left\| \nabla^{\idx}_{\theta, \phi} \mathcal{L}\left(\theta_{k}, \phi_{k}\right)\right\|^{2}\right] \\ &\leq \sum_{k=0}^{n} \frac{\gamma_{k+1}}{\sum_{k=0}^{n} \gamma_{k+1}}\mathbb{E}\left[\left\| \nabla^{\idx}_{\theta, \phi} \mathcal{L}\left(\theta_{k}, \phi_{k}\right)\right\|^{2}\right] \\  &\leq \frac{ 2\left(\mathbb{E} [\mathcal{L}\left(\theta_{n+1}, \phi_{n+1}\right)] - \mathcal{L}\left(\theta_{0}, \phi_{0}\right) \right) + L\sigma^{2} \sum_{k=0}^{n} \gamma_{k+1}^{2} / (BK)}{\sqrt{n}}\eqsp, 
\end{align*}
which concludes the proof by noting that $\mathcal{L}\left(\theta_{n+1}, \phi_{n+1}\right) \leq \mathcal{L}\left(\theta^{*}, \phi^{*}\right)$.
\end{proof}

\subsection{Proof of Theorem \ref{th:rate_general_Adam}}

\begin{proof}
Since $\widehat{\nabla}_{\theta, \phi} \mathcal{L}\left(\theta, \phi; \mathcal{D}\right)$ is an unbiased estimator of the gradient of the expected ELBO, the proof is a direct consequence of \cite[Proposition 4.2]{shi2021rmsprop} using the smoothness of the ELBO (Proposition \ref{prop:ELBO_smooth_score_pathwise}).
\end{proof}

\section{LINEAR GAUSSIAN VAE}
\label{supp:sec:linear_gauss_vae}

\subsection{Analytic ELBO of the Linear VAE}

While analytic solutions for deep latent models are generally not available, the Linear VAE provides analytic solutions for optimal parameters, allowing us to gain insights into various phenomena associated with VAE training. For instance, \cite{dai2018connections} explore the connections between Linear VAE, probabilistic PCA \citep{tipping1999probabilistic}, and robust PCA \citep{candes2011robust, chandrasekaran2011rank}, and analyze the local minima smoothing effects of VAE. Similarly, \cite{lucas2019don, wang2022posterior} use Linear VAE to study the causes of posterior collapse \cite{razavi2019preventing}.
The following proposition provides the analytical form of the ELBO for the Linear VAE defined in \eqref{eq:linear_VAE}, which is crucial for analyzing the convergence rate.

\begin{Proposition} \label{prop:closed_form_Linear_VAE}
The KL-divergence term and the reconstruction term can be expressed respectively for all $x\in\Xset$ as:
$$
\text{KL}(q_{\phi}(\cdot|x) \| p) = \frac{1}{2} \left( -\log \det D + \left\| W_{2}x + b_{2} \right\|^2 + \mathrm{tr}(D) - d_{z} \right)\eqsp,
$$
\begin{align*}
\mathbb{E}_{q_{\phi}(\cdot|x)} \left[\log p_{\theta}(x|Z)\right] &= \frac{1}{2c^2} \left[ -\mathrm{tr}(W_{1}DW_{1}^\top) - \left\| W_{1}(W_{2}x + b_{2}) + b_{1} - x\right\|^2 \right] - \frac{d_{z}}{2} \log 2\pi c^2\eqsp.
\end{align*}
\end{Proposition}

\begin{proof}
First, we have:
$$
\begin{aligned}
\int q_{\phi}(z|x) \log p(z) \, \rmd z &= \mathbb{E}_{q_{\phi}(\cdot|x)} \left[ - \frac{d_{z}}{2} \log 2\pi - \frac{1}{2} \left\| Z \right\|^2 \right] \\
&= - \frac{d_{z}}{2} \log 2\pi - \frac{1}{2} \sum_{i=1}^{d} \mathbb{E}_{q_{\phi}(\cdot|x)} \left[ Z_{i}^2 \right] \\
&= - \frac{d_{z}}{2} \log 2\pi - \frac{1}{2} \left( \text{tr}(D) + \left\| W_{2}x + b_{2} \right\|^2 \right)\eqsp.
\end{aligned}
$$

$$
\begin{aligned}
\int q_{\phi}(z|x) \log q_{\phi}(z|x) \, \rmd z &= - \frac{d_{z}}{2} \log 2\pi - \frac{1}{2} \log \det D - \frac{1}{2} \mathbb{E}_{q_{\phi}(\cdot|x)} \left[ (Z - W_{2}x - b_{2})^\top D^{-1}(Z - W_{2}x - b_{2}) \right] \\
&= - \frac{d_{z}}{2} \log 2\pi - \frac{1}{2} \log \det D - \frac{d_{z}}{2}\eqsp.
\end{aligned}
$$
By subtracting these two terms, we obtain the KL-divergence:
$$
\text{KL}(q_{\phi}(\cdot|x) || p) = \frac{1}{2} \left( -\log \det D + \left\| W_{2}x + b_{2} \right\|^2 + \text{tr}(D) - d_{z} \right)\eqsp.
$$
For the reconstruction term, we have:
$$
\begin{aligned}
\mathbb{E}_{q_{\phi}(\cdot|x)} [\log p_{\theta}(x|Z)] &= - \frac{d_{z}}{2} \log 2\pi c^2 - \frac{1}{2c^2}\mathbb{E}_{q_{\phi}(\cdot|x)} \left[ \left\| x - W_{1}Z - b_{1} \right\|^2 \right] \\
&= - \frac{d_{z}}{2} \log 2\pi c^2 - \frac{1}{2c^2} \mathbb{E}_{q_{\phi}(\cdot|x)} \left[ \left\| W_{1}Z \right\|^2 - 2(x - b_{1})^\top  W_{1}Z + \left\| x - b_{1} \right\|^2 \right] \\
&= \frac{1}{2c^2} \left[ -\text{tr}(W_{1}DW_{1}^\top ) - \left\| W_{1}(W_{2}x + b_{2}) \right\|^2 + 2(x - b_{1})^\top  W_{1} (W_{2}x + b_{2}) - \left\| x - b_{1} \right\|^2 \right] \\
&\quad - \frac{d_{z}}{2} \log 2\pi c^2\eqsp,
\end{aligned}
$$
where we used the fact that $W_{1}z \sim \mathcal{N} \left( W_{1}(W_{2}x + b_{2}), W_{1}DW_{1}^\top  \right)$.
\end{proof}

\subsection{Proof of Corollary \ref{cor:Linear_VAE_rate}}
\begin{proof}
First, we compute the derivative of the ELBO with respect to each parameter of the encoder and decoder.
\paragraph{Derivatives of ELBO. }  
$$
\nabla_{W_{1}} \mathcal{L}(\theta, \phi; x) = \frac{1}{c^2} \left( (x - b_{1})(W_{2}x + b_{2})^\top  - W_{1}D - W_{1}(W_{2}x + b_{2})(W_{2}x + b_{2})^\top  \right) 
$$
$$
\nabla_{W_{2}} \mathcal{L}(\theta, \phi; x) = \frac{1}{c^2} \left( W_{1}^{\top}(x - b_{1})x^{\top} - W_{1}^\top W_{1}(W_{2}x + b_{2})x^{\top} - c^2 (W_{2}x + b_{2})x^{\top} \right) 
$$
$$
\nabla_{b_{1}} \mathcal{L}(\theta, \phi; x) = \frac{1}{c^2} \left( x - b_{1} - W_{1}(W_{2}x + b_{2}) \right) 
$$
$$
\nabla_{b_{2}} \mathcal{L}(\theta, \phi; x) = \frac{1}{c^2} \left( W_{1}^\top (x - b_{1}) - W_{1}^{\top}W_{1}(W_{2}x + b_{2}) - c^2(W_{2}x + b_{2}) \right) 
$$
$$
\nabla_{D} \mathcal{L}(\theta, \phi; x) = \frac{1}{2} \left( D^{-1} - \mathrm{I}_{d_z} - \frac{1}{c^2} \text{diag}(W_{1}^\top W_{1} ) \right) 
$$

\paragraph{Smoothness Property. }

$$
\left\| \nabla_{W_{1}} \mathcal{L}(\theta, \phi) - \nabla_{W'_{1}} \mathcal{L}(\theta, \phi) \right\| \leq \frac{1}{c^2} \left( \left\| D \right\| + \left\| \mathbb{E}_{\pi} \left[(W_{2}x + b_{2})(W_{2}x + b_{2})^\top \right] \right\| \right) \left\| W_{1} - W'_{1} \right\|
$$
$$
\left\| \nabla_{W_{2}} \mathcal{L}(\theta, \phi) - \nabla_{W'_{2}} \mathcal{L}(\theta, \phi) \right\| \leq \frac{1}{c^2} \left\| W_{1}^{\top}W_{1} + c^2 \mathrm{I}_{d_x} \right\| \left\| \mathbb{E}_{\pi} \left[xx^{\top}\right] \right\| \left\| W_{2} - W'_{2} \right\|
$$
$$
\left\| \nabla_{b_{1}} \mathcal{L}(\theta, \phi) - \nabla_{b'_{1}} \mathcal{L}(\theta, \phi) \right\| = \frac{1}{c^2} \left\| b_{1} - b'_{1} \right\|
$$
$$
\left\| \nabla_{b_{2}} \mathcal{L}(\theta, \phi) - \nabla_{b'_{2}} \mathcal{L}(\theta, \phi) \right\| \leq \frac{1}{c^2} \left\| W_{1}^{\top}W_{1} + c^2 \mathrm{I}_{d_x} \right\| \left\| b_{2} - b'_{2} \right\|
$$
$$
\left\| \nabla_{D} \mathcal{L}(\theta, \phi) - \nabla_{D'} \mathcal{L}(\theta, \phi) \right\| = \frac{1}{2} \left\| D^{-1} - D'^{-1} \right\| \leq \frac{1}{2} \left\| D^{-1} \right\| \left\| D'^{-1} \right\| \left\| D - D' \right\|
$$

If $D = \text{Diag}(\sigma^{2}_1, \ldots, \sigma^{2}_{d_{z}})$ as used in practice, the last inequality can be expressed as: 
$$
\left| \nabla_{\sigma^{2}} \mathcal{L}(\theta, \phi) - \nabla_{\sigma'^{2}} \mathcal{L}(\theta, \phi) \right| = \frac{1}{2\sigma^{2} \sigma'^{2}} \left| \sigma^{2} - \sigma'^{2} \right|.
$$
Since the parameter space is compact, $\lambda_{\min}(D) \geq c_{D}$ for some $c_{D} > 0$, and the inputs have bounded second moments, it follows that the ELBO is smooth. 
The proof is completed using Theorem \ref{th:rate_general_Adam}.
\end{proof}

\section{DEEP GAUSSIAN VAE}
\label{supp:sec:deep_gauss_vae}

\subsection{Activation Functions in Deep Gaussian VAE}
\label{supp:subsec:activation_function}

\begin{definition}\label{def:NN_act}
The activation functions Sigmoid, Hyperbolic Tangent (Tanh), Softplus, and Continuously Differentiable Exponential Linear Units (CELU) are defined for all $x \in \Rset$ as follows:
\begin{align*}
\text{Sigmoid}(x) = \frac{1}{1 + \rme^{-x}}, \quad
\text{Tanh}(x) = \frac{\rme^{x} - \rme^{-x}}{\rme^{x} + \rme^{-x}}, \quad
\text{Softplus}(x) = \ln(1 + \rme^{x})\eqsp, 
\end{align*}
and
$$
\text{CELU}(x) = \begin{cases}
x & \text{if } x > 0, \\
\alpha \left(\exp \left(\frac{x}{\alpha}\right) - 1\right) & \text{if } x \leq 0.
\end{cases}
$$
\end{definition}

Consider the following neural network formulations for the encoder and decoder:
\begin{multline*}
\mathcal{F}_G = \Big\{ G_{\theta} : \Rset^{d_z} \to \Rset^{d_x} \,;\, G_{\theta}(z) = \NN(z; \theta, f, N_{dd}), \, \theta = \{W_{\ell}, b_{\ell}\}_{\ell=1}^{N_{dd}} \in \Theta, \\
\sigma_{\ell} \in \mathcal{F}_{\text{SL}}, \ell = 1, \ldots, N_{dd}, \; \text{and} \; \| G_{\theta}(z) \| \leq C_{G} \Big\},  
\end{multline*}
and
\begin{multline*}
\mathcal{F}_{\mu,\Sigma} = \Big\{ (\mu_{\phi}, \Sigma_{\phi}) : \Rset^{d_x} \to \Rset^{d_z} \times \Rset^{d_z \times d_z} \,;\, \left(\mu_{\phi}(x), \Sigma_{\phi}(x)\right) = \NN(x; \phi, f, N_{ed}), \, \phi = \{W_{\ell}, b_{\ell}\}_{\ell=1}^{N_{ed}} \in \Phi, \\
\qquad f_{\ell} \in \mathcal{F}_{\text{SL}}, \ell = 1, \ldots, N_{ed}, \, \| \mu_{\phi}(x) \| \leq C_{\mu}, \quad \lambda_{\min}(\Sigma_{\phi}(x)) \geq c_{\Sigma}, \; \text{and} \; \| \Sigma_{\phi}(x) \| \leq C_{\Sigma} \Big\},
\end{multline*}
where $\mathcal{F}_{\text{SL}}$ denotes the set of functions that are both smooth and Lipschitz continuous and $\Fb$ denotes the set of bounded functions.

In $\mathcal{F}_{G}$ and $\mathcal{F}_{\mu,\Sigma}$, $f$ represents an activation function, specifically one chosen from the set that includes the sigmoid, hyperbolic tangent (tanh), softplus \citep{glorot2011deep}, or Continuously Differentiable Exponential Linear Units (CELU) \citep{clevert2015fast, barron2017continuously}.
These activation functions are crucial in neural network architectures as they ensure the assumptions made about the encoder and decoder distributions.

\begin{Proposition}\label{prop:smooth_act}
The activation functions sigmoid, tanh, softplus, and CELU are Lipschitz continuous and smooth.
\end{Proposition}

\begin{proof}
We consider each activation function separately.
\paragraph{Sigmoid. }
The first and second derivatives of the sigmoid function are given  for all $x$ by:
$$f_1'(x) = f_1(x)(1 - f_1(x))\quad \mathrm{and} \quad f_1''(x) = f_1(x)(1 - f_1(x))(1 - 2f_1(x))\eqsp.$$
Since the sigmoid function is bounded by 1, it follows that for all $x$, $\left|f_1'(x)\right| \leq 1/4$ and $\left|f_1''(x)\right| \leq 1/4$. Therefore, the sigmoid activation function is Lipschitz continuous and smooth.

\paragraph{Tanh. }
The first and second derivatives of the hyperbolic tangent function are given  for all $x$ by:
$$
f_2'(x) = 1 - f_2^2(x) \quad \mathrm{and} \quad f_2''(x) = -2f_2(x)(1 - f_2^2(x))\eqsp.$$
Since $f_2$ is bounded by 1, for all $x$,  $\left|f_2'(x)\right| \leq 1$ and $\left|f_2''(x)\right| \leq 1$. Therefore, the Tanh activation function is Lipschitz continuous and smooth.

\paragraph{Softplus. }
The first and second derivatives of the softplus function are given for all $x$  by:
$$
f_3'(x) = f_1(x) \quad \mathrm{and} \quad f_3''(x) = f_1(x)(1 - f_1(x))\eqsp.$$
Since $f_1$ is bounded by 1, it follows that $\left|f_3'(x)\right| \leq 1$ and $\left|f_3''(x)\right| \leq 1/4$. Therefore, the softplus activation function is Lipschitz continuous and smooth.

\paragraph{CELU. }
The first and second derivatives of the CELU function are given for all $x$ by:
$$
f_4'(x) = \begin{cases} 
1 & \text{if } x > 0, \\
\exp \left(\frac{x}{\alpha}\right) & \text{if } x \leq 0,
\end{cases}
$$
$$
f_4''(x) = \begin{cases} 
0 & \text{if } x > 0, \\
\frac{1}{\alpha} \exp \left(\frac{x}{\alpha}\right) & \text{if } x \leq 0.
\end{cases}
$$
Since $\rme^{x} \leq 1$ for $x \leq 0$, it follows that $\left|f_4'(x)\right| \leq 1$ and $\left|f_4''(x)\right| \leq 1/\alpha$. Therefore, the CELU activation function is Lipschitz continuous and smooth.
\end{proof}

Proposition \ref{prop:smooth_act} highlights that activation functions such as sigmoid, tanh, softplus, and CELU are suitable for use in the encoder and decoder of a network due to their Lipschitz continuity and smoothness.

\subsection{Proof of Proposition \ref{prop:soft_clip}}

\begin{proof}
The first derivative of $f$ is given by:
$$
f'(x) = \frac{e^{s(x-s_1)}}{1 + e^{s(x-s_1)}} - \frac{e^{s(x-s_2)}}{1 + e^{s(x-s_2)}} \eqsp.
$$
Given that $s_1 \leq s_2$ and that the sigmoid function is increasing, we have $f'(x) \geq 0$. Therefore, $f$ is an increasing function. Additionally, since $f(x) \to s_1$ as $x \to -\infty$ and $f(x) \to s_2$ as $x \to +\infty$, it follows that $f$ is bounded between $s_1$ and $s_2$.
Furthermore, since $\left|f'(x)\right| \leq \tanh \left(s(s_2-s_1)/4\right) \leq 1$, $f$ is Lipschitz continuous.
The second derivative of $f$ is given by:
$$
f''(x) = s \left[\frac{e^{s(x-s_1)}}{(1 + e^{s(x-s_1)})^2} - \frac{e^{s(x-s_2)}}{(1 + e^{s(x-s_2)})^2} \right]\eqsp.
$$
Since $\left|f''(x)\right| \leq s/4$, it follows that $f$ is smooth.
\end{proof}

\subsection{Proof of Theorem \ref{thm:con_gauss}}

First, the ELBO objective can be expressed for all $x \in \Xset$ as:
$$
\mathcal{L}(\theta, \phi; x) = - \text{KL}(q_{\phi}(\cdot | x) \| p) - \frac{1}{2 c^2} \mathbb{E}_{q_{\phi}(\cdot|x)} \left[ \| G_{\theta}(z) - x \|^{2} \right] - \frac{1}{2} \log(2\pi c^2).
$$

We divide the proof into two parts: Section \ref{sec:score_gauss} establishes the convergence rate using the score function gradient, and Section \ref{sec:pathwise_gauss} presents the convergence rate with the pathwise gradient.

\subsubsection{Analysis with the score function}
\label{sec:score_gauss}

We first analyze the convergence rate using the score function in the Gaussian case.

\begin{theorem}\label{thm:supp:con_gauss_score_function}
Let $c_0 > 0$, and consider 
\begin{align*}
 \mathcal{F}_{dd} &= \{ (x,z)\mapsto p_{\theta}(x|z) = \mathcal{N} (x;G_{\theta}(z), c^2 \mathrm{I}_{d_x}) \,|\, G_{\theta}(z) \in \mathcal{F}_G, \theta \in \Theta \subseteq \Rset^{d_{\theta}}, c > c_0 \}\eqsp,\\
 \mathcal{F}_{ed} &= \{ (x,z)\mapsto q_{\phi}(z|x) = \mathcal{N} (z;\mu_{\phi}(x), \Sigma_{\phi}(x)) \,|\, (\mu_{\phi}(x), \Sigma_{\phi}(x)) \in \mathcal{F}_{\mu, \Sigma}, \phi \in \Phi \subseteq \Rset^{d_{\phi}} \}\eqsp.
\end{align*}
Assume that there exists $C_{rec} \in \mathsf{M}(\Xset \times \Zset)$ such that $\left\| x - G_{\theta}(z) \right\| \leq C_{rec}(x, z)$ for all $\theta \in \Theta$ and $(x, z) \in \Xset \times \Zset$.
Assume also that the data distribution $\pi$ has a finite fourth moment, and that there exists some constant $a$ such that for all $\theta \in \Theta$ and $\phi \in \Phi$,
$$ 
\lVert \theta \rVert_{\infty} + \lVert \phi \rVert_{\infty} \leq a\eqsp. 
$$
Let $\left(\theta_{n},\phi_{n}\right) \in \Theta \times \Phi$ be the $n$-th iterate of the recursion in Algorithm \ref{alg:adam}, where $\gamma_{n} = C_{\gamma}n^{-1/2}$ with $C_{\gamma}>0$. Assume that $\beta_1 < \sqrt{\beta_2} < 1$.
For any $n \geq 1$, let $R \in \{0, \ldots, n\}$ be a uniformly distributed random variable. Then,
$$
\mathbb{E}\left[\left\| \nabla \mathcal{L}\left(\theta_{R}, \phi_{R}\right)\right\|^{2}\right] = \mathcal{O}\left(\frac{\mathcal{L}^*}{\sqrt{n}} + d_{z}^2 \frac{N_{\text{max}} a^{2(N_{\text{max}}-1)}}{1-\beta_{1}} \frac{d^{*} \log n}{\sqrt{n}}\right)\eqsp,
$$
where $\mathcal{L}^* = \mathcal{L}\left(\theta^{*}, \phi^{*}\right) - \mathcal{L}\left(\theta_{0}, \phi_{0}\right)$, $d^{*} = d_{\theta} + d_{\phi}$ is the total dimension of the parameters, $N_{\text{max}} = \max\{N_{ed}, N_{dd}\}$ represents the maximum number of layers in the model architecture.
\end{theorem}

We divide the proof into two lemmas. Lemma \ref{lem:ELBO_A1_gauss} ensures Assumption \ref{ass:A1}, while Lemma \ref{lem:ELBO_smooth_gauss} guarantees Assumption \ref{ass:A2}(i).

\begin{lemma}\label{lem:ELBO_A1_gauss}
Let $c_0 > 0$, and consider 
\begin{align*}
  \mathcal{F}_{dd} &= \{ (x,z)\mapsto p_{\theta}(x|z) = \mathcal{N} (x;G_{\theta}(z), c^2 \mathrm{I}_{d_x}) \,|\, G_{\theta}(z) \in \mathcal{F}_G, \theta \in \Theta \subseteq \Rset^{d_{\theta}}, c > c_0 \}\eqsp,\\
  \mathcal{F}_{ed} &= \{ (x,z)\mapsto q_{\phi}(z|x) = \mathcal{N} (z;\mu_{\phi}(x), \Sigma_{\phi}(x)) \,|\, (\mu_{\phi}(x), \Sigma_{\phi}(x)) \in \mathcal{F}_{\mu, \Sigma}, \phi \in \Phi \subseteq \Rset^{d_{\phi}} \}\eqsp.
\end{align*}
Then, Assumption \ref{ass:A1} is satisfied i.e. for all $\theta \in \Theta$, $\phi \in \Phi$, $x \in \Xset$ and $z \in \Zset$,
$$
\max\{ \left| \log p_{\theta}(x , z) \right|, \left|\log q_{\phi}(z|x)\right| \} \leq \alpha(x,z)\eqsp,
$$
with 
$$
\alpha(x,z) = \max \left\{\frac{d_z}{2} \log(2\pi C_{\Sigma}) + \frac{1}{c_{\Sigma}} \left(\|z\|^2 + C_{\mu}^2 \right), \frac{d_z}{2} \log(2\pi c^2) + \frac{1}{c^2} \left(\|x\|^2 + C_{G}^2 \right)
\right\}\eqsp.
$$
\end{lemma}

\begin{proof}
For all $x \in \Xset$, the density function $z \mapsto q_{\phi}(z|x)$ of the Gaussian encoder is given by:
$$
q_{\phi}(z|x) = \det(2\pi\Sigma_\phi(x))^{-1/2} \exp\left(-\frac{1}{2} (z - \mu_\phi(x))^\top \Sigma_\phi^{-1}(x) (z - \mu_\phi(x))\right)\eqsp,
$$
where $\mu_\phi(x)$ and $\Sigma_\phi(x)$ are the mean and covariance matrix of the Gaussian distribution.

\paragraph{Bounding the Normalization Factor. }

The determinant of $\Sigma_\phi(x)$ can be expressed as:
$$
\det(\Sigma_\phi(x)) = \prod_{i=1}^{d_z} \lambda_i(x),
$$
where $\lambda_1(x), \lambda_2(x), \dots, \lambda_{d_z}(x)$ are the eigenvalues of $\Sigma_\phi(x)$. Given the conditions $\lambda_{\min}(\Sigma_\phi(x)) \geq c_{\Sigma}$ and $\|\Sigma_\phi(x)\| \leq C_{\Sigma}$, it follows that:
$$c_{\Sigma}^{d_z} \leq \det(\Sigma_\phi(x)) \leq C_{\Sigma}^{d_z}\eqsp.$$
Thus, the normalization factor is bounded by:
$$
(2\pi C_{\Sigma})^{-d_z/2} \leq  \det(2\pi\Sigma_\phi(x))^{-1/2} \leq (2\pi c_{\Sigma})^{-d_z/2}\eqsp.
$$

\paragraph{Bounding the Exponential Term. }

We have:
$$
\exp\left(-\frac{1}{2c_{\Sigma}} \|z - \mu_\phi(x)\|^2\right) \leq \exp\left(-\frac{1}{2} (z - \mu_\phi(x))^\top \Sigma_\phi^{-1}(x) (z - \mu_\phi(x))\right) \leq \exp\left(-\frac{1}{2C_{\Sigma}} \|z - \mu_\phi(x)\|^2\right)\eqsp.
$$
Since for all $x \in \Xset$ $\|\mu_\phi(x)\| \leq C_{\mu}$, for all $x \in \Xset$,  and $z \in \Zset$,
$$
\|z - \mu_\phi(x)\|^2 \leq 2 \|z\|^2 + 2\|\mu_\phi(x)\|^2 \leq 2 \|z\|^2 + 2C_{\mu}^2\eqsp.
$$

Similarly, we can bound from below using:
$$
\|z - \mu_\phi(x)\|^2 \geq \frac{1}{2} \left(\|z\|^2 - \|\mu_\phi(x)\|^2\right) \geq \frac{1}{2} \left(\|z\|^2 - C_{\mu}^2\right)\eqsp.
$$

Therefore, we have uniform bounds on the exponential term that are independent of the encoder parameters $\phi$:
$$
\exp\left(-\frac{1}{c_{\Sigma}} \left(\|z\|^2 + C_{\mu}^2 \right) \right) \leq \exp\left(-\frac{1}{2} (z - \mu_\phi(x))^\top \Sigma_\phi^{-1}(x) (z - \mu_\phi(x))\right) \leq \exp\left(-\frac{1}{4C_{\Sigma}} \left(\|z\|^2 - C_{\mu}^2 \right)\right)\eqsp.
$$

Combining these results, we obtain:
$$
\frac{1}{(2\pi C_{\Sigma})^{d_z/2}} \exp\left(-\frac{1}{c_{\Sigma}} \left(\|z\|^2 + C_{\mu}^2 \right) \right) \leq q_{\phi}(z|x) \leq \frac{1}{(2\pi c_{\Sigma})^{d_z/2}} \exp\left(-\frac{1}{4C_{\Sigma}} \left(\|z\|^2 - C_{\mu}^2 \right)\right)\eqsp.
$$

Since the activation function in the final layer is bounded, there exists a constant $C_{G} > 0$ such that for all $z \in \Zset$, $\|G_{\theta}(z)\| \leq C_{G}$.
Proceeding similarly for the Gaussian decoder, where the density function is given by $x \mapsto p_{\theta}(x | z) = \mathcal{N}(x;G_{\theta}(z), c^2 \mathrm{I}_{d_x})$ yields:
$$
\frac{1}{(2\pi c^2)^{d_z/2}} \exp\left(-\frac{1}{c^2} \left(\|x\|^2 + C_{G}^2 \right) \right) \leq p_{\theta}(x | z) \leq \frac{1}{(2\pi c^2)^{d_z/2}} \exp\left(-\frac{1}{4c^2} \left(\|x\|^2 - C_{G}^2 \right)\right)\eqsp.
$$
This implies that Assumption \ref{ass:A1} is verified.
\end{proof}

Lemma~\ref{lem:ELBO_smooth_gauss} shows that Assumption \ref{ass:A2} holds without explicitly specifying the smoothness constant. However, Lemma \ref{lem:ELBO_smooth_gauss_finite} provides the smoothness constant $L^{\SF}$, and also shows that it is well-defined and finite.

\begin{lemma}\label{lem:ELBO_smooth_gauss}
Let $c_0 > 0$, and consider
\begin{align*}
 \mathcal{F}_{dd} &= \{ (x,z)\mapsto p_{\theta}(x|z) = \mathcal{N} (x;G_{\theta}(z), c^2 \mathrm{I}_{d_x}) \,;\, G_{\theta} \in \mathcal{F}_G, \theta \in \Theta \subseteq \Rset^{d_{\theta}}, c > c_0 \}\eqsp,\\
 \mathcal{F}_{ed} &= \{(x,z)\mapsto q_{\phi}(z|x) = \mathcal{N} (z;\mu_{\phi}(x), \Sigma_{\phi}(x)) \,;\, (\mu_{\phi}(x), \Sigma_{\phi}(x)) \in \mathcal{F}_{\mu, \Sigma}, \phi \in \Phi \subseteq \Rset^{d_{\phi}} \}\eqsp.
\end{align*}
Assume that there exists $C_{rec} \in \mathsf{M}(\Xset \times \Zset)$ such that $\left\| x - G_{\theta}(z) \right\| \leq C_{rec}(x, z)$ for all $\theta \in \Theta$ and $(x, z) \in \Xset \times \Zset$.
and that there exists some constant $a$ such that for any $\theta \in \Theta$ and $\phi \in \Phi$,
$$ \lVert \theta \rVert_{\infty} + \lVert \phi \rVert_{\infty} \leq a\eqsp. $$
Then, there exist $L_1 \in \mathsf{M}(\Xset \times \Zset)$ and $L_2 \in \mathsf{M}(\Xset \times \Zset)$ such that for all $\theta, \theta' \in \Theta$, $\phi, \phi' \in \Phi$, $x \in \Xset$ and $z \in \Zset$:
$$
\left\|\nabla_{\theta} \log p_{\theta}(x|z) - \nabla_{\theta} \log p_{\theta'}(x|z)\right\| \leq L_1(x,z) \left\|\theta - \theta'\right\|\eqsp,
$$
$$
\left\|\nabla_{\phi} \log q_{\phi}(z|x) - \nabla_{\phi} \log q_{\phi'}(z|x)\right\| \leq L_2(x,z) \left\|\phi - \phi'\right\|\eqsp.
$$
\end{lemma}

\begin{proof}
\textbf{Gaussian Decoder. }
First, the gradient of $\log p_{\theta}(x|z)$ is given by:
\begin{align*}
\nabla_{\theta} \log p_{\theta}(x|z) &= \frac{1}{c^2} \nabla_{\theta} G_{\theta}(z)^\top  ( x - G_{\theta}(z))\eqsp.
\end{align*}

We have:
\begin{align*}
\| \nabla_{\theta} \log p_{\theta}(x|z) - \nabla_{\theta} \log p_{\theta'}(x|z) \| &= \frac{1}{c^2} \left\| \nabla_{\theta} G_{\theta}(z)^\top  ( x - G_{\theta}(z)) - \nabla_{\theta} G_{\theta'}(z)^\top  ( x - G_{\theta'}(z)) \right\| \\
&\leq \frac{1}{c^2} \left\| \nabla_{\theta} G_{\theta}(z)^\top  ( x - G_{\theta}(z)) - \nabla_{\theta} G_{\theta}(z)^\top  ( x - G_{\theta'}(z)) \right\| \\ 
& \quad + \frac{1}{c^2} \left\| \nabla_{\theta} G_{\theta}(z)^\top  ( x - G_{\theta'}(z)) - \nabla_{\theta} G_{\theta'}(z)^\top  ( x - G_{\theta'}(z)) \right\| \\
&\leq \frac{1}{c^2} \left( \left\| \nabla_{\theta} G_{\theta}(z) \right\| \left\| G_{\theta}(z) - G_{\theta'}(z) \right\| + \left\| x - G_{\theta'}(z) \right\| \left\| \nabla_{\theta} G_{\theta}(z)  - \nabla_{\theta} G_{\theta'}(z) \right\|  \right)\eqsp,
\end{align*}
Let $M_{f_i}$ and $L_{f_i}$ represent the Lipschitz constant and the smoothness parameter of the activation function $f_i$ in the $i$-th layer, respectively.
Using Lemma \ref{lem:smooth_NN}, we get:
$$
\left\|\nabla_{\theta} G_{\theta}(z)\right\| \leq \left(\|z\| + 1\right) a^{N_{dd}-1} \prod_{j=1}^{N_{dd}} M_{f_j} =: M_{\nabla G}(z)\eqsp,
$$
$$\left\| \nabla_{\theta} G_{\theta}(z) - \nabla_{\theta} G_{\theta'}(z) \right\| \leq L_{\nabla G}(z)\left\| \theta - \theta'\right\|\eqsp,
$$
where $L_{\nabla G}(z) = N_{dd} \left(\|z\|^2 + 1\right) \sum_{k=1}^{N_{dd}} L_{f_k} a^{N_{dd}-2+k} \prod_{i=1}^{k-1} M_{f_i}^2 \prod_{i=k+1}^{N_{dd}} M_{f_i}$.
Therefore,
\begin{align*}
\| \nabla_{\theta} \log p_{\theta}(x|z) - \nabla_{\theta} \log p_{\theta'}(x|z) \| &\leq \frac{1}{c^2} \left( \left\| \nabla_{\theta} G_{\theta}(z) \right\|^2 + L_{\nabla G} \left\| x - G_{\theta'}(z) \right\| \right) \left\| \theta - \theta' \right\| \\
&\leq L^{G}(x,z)\left\| \theta - \theta' \right\|\eqsp,
\end{align*}
where $L^{G}(x,z) = \left(M_{\nabla G}(z)^2 + C_{rec}(x, z) L_{\nabla G}(z) \right)/c^2$.

\paragraph{Gaussian Encoder. }

Consider the Gaussian encoder $E_{\phi}$, parameterized by a mean function $\mu_{\phi}(x)$ and a diagonal covariance matrix $\Sigma_{\phi}(x) = \text{Diag}(\sigma^2_1(x), \ldots, \sigma^2_{d_{z}}(x))$. We define $l_{\phi}(x) = (\log \sigma^2_1(x), \ldots, \log \sigma^2_{d_{z}}(x))$ as the logarithm of the variance. The proof for the Gaussian encoder can be treated similarly to that of the decoder. However, it differs in that we also learn the variance of the Gaussian distribution. Since we use the same neural network architecture for both $\mu_{\phi}(x)$ and $l_{\phi}(x)$ as is used for $G_{\theta}(z)$, the mappings $\phi \mapsto \mu_{\phi}(x)$ and $\phi \mapsto l_{\phi}(x)$ are smooth functions with bounded gradients.
First, the log density of the variational distribution can be expressed as:
\begin{align*}
\log q_{\phi}(z|x) &= - \frac{1}{2} (\mu_{\phi}(x) - z)^\top \Sigma_{\phi}^{-1}(x)(\mu_{\phi}(x) - z) - \frac{1}{2} \log(2\pi \det \Sigma_{\phi}(x)) \\
&= - \frac{1}{2} (\mu_{\phi}(x) - z)^\top  e^{-l_{\phi}(x)} \odot (\mu_{\phi}(x) - z) - \text{tr}(L_{\phi}(x)) - \frac{1}{2} \log(2\pi)\eqsp,
\end{align*}
where $L_{\phi}(x) = \text{Diag}(l_{\phi}(x))$. Then, the gradient is given by:
\begin{align*}
\nabla_{\phi} \log q_{\phi}(z|x) &= \nabla_{\phi} \mu_{\phi}(x)^\top  e^{-l_{\phi}(x)} \odot (z - \mu_{\phi}(x)) + \frac{1}{2} \nabla_{\phi} l_{\phi}(x)^\top  e^{-l_{\phi}(x)} \odot (z - \mu_{\phi}(x))^2 - \nabla_{\phi} l_{\phi}(x)^\top  \mathbf{1}\eqsp.
\end{align*}
Therefore, we have:
$$
\| \nabla_{\phi} \log q_{\phi}(z|x) - \nabla_{\phi} \log q_{\phi'}(z|x) \| \leq A_1 + A_2 + A_3\eqsp,
$$
where
\begin{align*}
A_1 &= \| \nabla_{\phi} \mu_{\phi}(x)^\top  e^{-l_{\phi}(x)} \odot (z - \mu_{\phi}(x)) - \nabla_{\phi} \mu_{\phi'}(x)^\top  e^{-l_{\phi'}(x)} \odot (z - \mu_{\phi'}(x)) \|\eqsp, \\
A_2 &= \frac{1}{2} \| \nabla_{\phi} l_{\phi}(x)^\top  e^{-l_{\phi}(x)} \odot (z - \mu_{\phi}(x))^2 - \nabla_{\phi} l_{\phi'}(x)^\top  e^{-l_{\phi'}(x)} \odot (z - \mu_{\phi'}(x))^2 \|\eqsp, \\
A_3 &= \| \nabla_{\phi} l_{\phi}(x) - \nabla_{\phi} l_{\phi'}(x) \|\eqsp.
\end{align*}
Term $A_1$ is upper bounded as follows
$$
\begin{aligned}
A_1 &= \| \nabla_{\phi} \mu_{\phi}(x)^\top  \rme^{-l_{\phi}(x)} \odot (z - \mu_{\phi}(x)) - \nabla_{\phi} \mu_{\phi'}(x)^\top  \rme^{-l_{\phi'}(x)} \odot (z - \mu_{\phi'}(x)) \| \\ 
&\leq \| \nabla_{\phi} \mu_{\phi}(x)^\top  \rme^{-l_{\phi}(x)} \odot (z - \mu_{\phi}(x)) - \nabla_{\phi} \mu_{\phi'}(x)^\top  \rme^{-l_{\phi}(x)} \odot (z - \mu_{\phi}(x)) \| \\ 
& \quad + \| \nabla_{\phi} \mu_{\phi'}(x)^\top  \rme^{-l_{\phi}(x)} \odot (z - \mu_{\phi}(x)) - \nabla_{\phi} \mu_{\phi'}(x)^\top  \rme^{-l_{\phi'}(x)} \odot (z - \mu_{\phi'}(x)) \| \\ 
&\leq \| \rme^{-l_{\phi}(x)} \| \| z - \mu_{\phi}(x) \| \| \nabla_{\phi} \mu_{\phi}(x)- \nabla_{\phi} \mu_{\phi'}(x) \| \\ 
& \quad + \| \nabla_{\phi} \mu_{\phi'}(x) \| \| \rme^{-l_{\phi}(x)} \odot (z - \mu_{\phi}(x)) - \rme^{-l_{\phi'}(x)} \odot (z - \mu_{\phi'}(x)) \| \\
&\leq \| \rme^{-l_{\phi}(x)} \| \| z - \mu_{\phi}(x) \| \| \nabla_{\phi} \mu_{\phi}(x)- \nabla_{\phi} \mu_{\phi'}(x) \|  + \| \nabla_{\phi} \mu_{\phi'}(x) \| \| z - \mu_{\phi}(x) \| \| \rme^{-l_{\phi}(x)}  - \rme^{-l_{\phi'}(x)} \| \\
& \quad + \| \nabla_{\phi} \mu_{\phi'}(x) \| \| \rme^{-l_{\phi'}(x)} \| \| \mu_{\phi}(x) - \mu_{\phi'}(x)) \| \\
&\leq \| \rme^{-l_{\phi}(x)} \| \| z - \mu_{\phi}(x) \| \| \nabla_{\phi} \mu_{\phi}(x)- \nabla_{\phi} \mu_{\phi'}(x) \|  + \| \nabla_{\phi} \mu_{\phi'}(x) \| \| z - \mu_{\phi}(x) \| \big(\sup_{\phi \in \Phi}\| \rme^{-l_{\phi}(x)} \| \big) \| l_{\phi}(x) - l_{\phi'}(x)) \| \\
& \quad + \| \nabla_{\phi} \mu_{\phi'}(x) \| \| \rme^{-l_{\phi'}(x)} \| \| \mu_{\phi}(x) - \mu_{\phi'}(x)) \| \eqsp,
\end{aligned}
$$
where we used the Mean Value Theorem in the last inequality.
Given the conditions that $\| \rme^{-l_{\phi}(x)} \| \leq 1/c_{\Sigma}(x)$, $\| z - \mu_{\phi}(x) \| \leq \|z\| + C_{\mu}$, and considering the boundedness and smoothness of the functions $\phi \mapsto \mu_{\phi}(x)$ and $\phi \mapsto l_{\phi}(x)$, there exists $L^{\mu, \Sigma}_{1}$ such that 
$$
A_1 \leq L^{\mu, \Sigma}_{1}(x,z) \| \phi - \phi' \|\eqsp.
$$
For the second term $A_2$, write
$$
\begin{aligned}
2A_2 &= \| \nabla_{\phi} l_{\phi}(x)^\top  \rme^{-l_{\phi}(x)} \odot (z - \mu_{\phi}(x))^2 - \nabla_{\phi} l_{\phi'}(x)^\top  \rme^{-l_{\phi'}(x)} \odot (z - \mu_{\phi'}(x))^2 \| \\
& \leq \| \nabla_{\phi} l_{\phi}(x)^\top  \rme^{-l_{\phi}(x)} \odot (z - \mu_{\phi}(x))^2 - \nabla_{\phi} l_{\phi'}(x)^\top  \rme^{-l_{\phi}(x)} \odot (z - \mu_{\phi}(x))^2 \| \\
& \quad + \| \nabla_{\phi} l_{\phi'}(x)^\top  \rme^{-l_{\phi}(x)} \odot (z - \mu_{\phi}(x))^2 - \nabla_{\phi} l_{\phi'}(x)^\top  \rme^{-l_{\phi'}(x)} \odot (z - \mu_{\phi'}(x))^2 \| \\
& \leq \| \rme^{-l_{\phi}(x)} \| \| z - \mu_{\phi}(x) \|^2 \| \nabla_{\phi} l_{\phi}(x) - \nabla_{\phi} l_{\phi'}(x) \| \\
& \quad + \| \nabla_{\phi} l_{\phi'}(x) \| \| \rme^{-l_{\phi}(x)} \odot (z - \mu_{\phi}(x))^2 - \rme^{-l_{\phi'}(x)} \odot (z - \mu_{\phi'}(x))^2 \| \\
& \leq \| \rme^{-l_{\phi}(x)} \| \| z - \mu_{\phi}(x) \|^2 \| \nabla_{\phi} l_{\phi}(x) - \nabla_{\phi} l_{\phi'}(x) \|  + \| \nabla_{\phi} l_{\phi'}(x) \| \| z - \mu_{\phi}(x) \|^2 \| \rme^{-l_{\phi}(x)} - \rme^{-l_{\phi'}(x)} \| \\
& \quad + \| \nabla_{\phi} l_{\phi'}(x) \| \| \rme^{-l_{\phi'}(x)} \| \| (z - \mu_{\phi}(x))^2 - (z - \mu_{\phi'}(x))^2 \| \\
& \leq \| \rme^{-l_{\phi}(x)} \| \| z - \mu_{\phi}(x) \|^2 \| \nabla_{\phi} l_{\phi}(x) - \nabla_{\phi} l_{\phi'}(x) \| + \| \nabla_{\phi} l_{\phi'}(x) \| \| z - \mu_{\phi}(x) \|^2 \big(\sup_{\phi \in \Phi}\| \rme^{-l_{\phi}(x)} \| \big) \| l_{\phi}(x) - l_{\phi'}(x)) \| \\
& \quad + \| \nabla_{\phi} l_{\phi'}(x) \| \| \rme^{-l_{\phi'}(x)} \| \| \mu_{\phi}(x) - \mu_{\phi'}(x) \| \|2z - \mu_{\phi}(x) - \mu_{\phi'}(x) \|\eqsp,
\end{aligned}
$$
where we used the Mean Value Theorem and the equality $\| (z - \mu_{\phi}(x))^2 - (z - \mu_{\phi'}(x))^2 \| = \| \mu_{\phi}(x) - \mu_{\phi'}(x) \| \|2z - \mu_{\phi}(x)) - \mu_{\phi'}(x) \|$ in the last inequality.
Given the conditions that $\| \rme^{-l_{\phi}(x)} \| \leq 1/c_{\Sigma}(x)$, $\| z - \mu_{\phi}(x) \| \leq \|z\| + C_{\mu}$, $\|2z - \mu_{\phi}(x)) - \mu_{\phi'}(x) \| \leq 2(\|z\| + C_{\mu})$ and considering the boundedness and smoothness of the functions $\phi \mapsto \mu_{\phi}(x)$ and $\phi \mapsto l_{\phi}(x)$, there exists $L^{\mu, \Sigma}_{2}$ such that 
$$
A_2 \leq L^{\mu, \Sigma}_{2}(x,z) \| \| \phi - \phi' \|\eqsp.
$$
Finally, for the last term $A_3$, using the smoothness of $\phi \mapsto l_{\phi}(x)$, there exists $L^{\mu, \Sigma}_{3}$ such that 
$$
A_3 \leq L^{\mu, \Sigma}_{3}(x,z) \| \phi - \phi' \|\eqsp.
$$
\end{proof}

\begin{lemma}\label{lem:ELBO_smooth_gauss_finite}
Under the assumptions of Lemma \ref{lem:ELBO_smooth_gauss}, the smoothness constant $L^{\SF}$ of the expected ELBO is well-defined and finite.
\end{lemma}

\begin{proof}
Using \ref{lemma:ELBO_smooth_score}, the smoothness constant $L^{\SF}$ is defined by: $$L^{\SF} = \sup_{\phi \in \Phi} \mathbb{E}_{\pi, \phi}\left[ L_1(x,z) + 2\alpha(x,z) L_2(x,z) + 4M(x,z)^2 + 4\alpha(x,z) M(x,z)^2 \right]\eqsp.$$

Let $L^{\SF}_{dd}$ and $L^{\SF}_{ed}$ denote the smoothness constants of the decoder and the encoder, respectively.
For the smoothness constant of the decoder, applying Lemma \ref{lem:smooth_NN}, we have:
\begin{align*}
L_{dd} &= \mathbb{E}_{\pi, \phi}\left[ L_1(x,z) + M(x,z)^2 \right] \\ 
&= \frac{1}{c^2} \mathbb{E}_{\pi, \phi}\left[ 2M_{\nabla G}(z)^2 + C_{rec}(x, z) L_{\nabla G}(z) \right] \\
&= \frac{1}{c^2} \mathbb{E}_{\pi, \phi}\left[ 2\left(\|z\| + 1\right)^2 a^{2(N_{dd}-1)} \prod_{i=1}^{N_{dd}} M_{f_i}^2 + N_{dd} C_{rec}(x, z) \left(\|z\|^2 + 1\right) \sum_{k=1}^{N_{dd}} L_{f_k} a^{N_{dd}-2+k} \prod_{i=1}^{k-1} M_{f_i}^2 \prod_{i=k+1}^{N_{dd}} M_{f_i} \right]\eqsp.
\end{align*}
Since 
$$\mathbb{E}_{q_{\phi}(\cdot|x)} \left[\|z\|^2\right] = \text{Tr}(\Sigma_{\phi}(x)) + \|\mu_{\phi}(x)\|^2 \leq d_z \|\Sigma_{\phi}(x)\| + C_{\mu} \leq d_z C_{\Sigma} + C_{\mu}\eqsp,$$
it follows that $L^{\SF}_{dd}$ is well-defined and finite.

For the smoothness constant of the encoder, we have:
$$L^{\SF}_{ed} = \mathbb{E}_{\pi, \phi}\left[ 2\alpha(x,z) L_2(x,z) + 3M(x,z)^2 + 4\alpha(x,z) M(x,z)^2 \right]\eqsp,$$
where
\begin{align*}
M(x,z) &= \left(\|x\| + 1\right) a^{N_{ed}-1} \prod_{i=1}^{N_{ed}} M_{f_i}\eqsp, \\
L_2(x,z) &= \frac{N_{ed}}{c_{\Sigma}} \left( 2\|z\|^2 + 2C_{\mu}^2 + \|z\| + C_{\mu} + c_{\Sigma} \right)\left\|x\right\|^2 \sum_{k=1}^{N_{ed}} L_{f_k} a^{N_{ed}-2+k} \prod_{i=1}^{k-1} M_{f_i}^2 \prod_{i=k+1}^{N_{ed}} M_{f_i} \\
&\quad + \frac{N_{ed}}{c_{\Sigma}} \left(2\|z\|^2 + 2C_{\mu}^2 + 3\|z\| + 3C_{\mu} + 1 \right)\left\|x\right\|^2 a^{2(N_{ed}-1)} \prod_{i=1}^{N} M_{f_i}^2\eqsp, \\
\alpha(x,z) &= \max \left\{\frac{d_z}{2} \log(2\pi C_{\Sigma}) + \frac{1}{c_{\Sigma}} \left(\|z\|^2 + C_{\mu}^2 \right), \frac{d_z}{2} \log(2\pi c^2) + \frac{1}{c^2} \left(\|x\|^2 + C_{G}^2 \right)
\right\}\eqsp.
\end{align*}
Since $\pi$ has a finite fourth moment, it is evident that $L^{\SF}_{ed}$ is well-defined and finite.

Moreover, under the conditions $M_{f_i} \leq 1$ and $L_{f_i} \leq 1$ for all $1 \leq i \leq N$, we can establish the following bounds: 
$$L^{\SF}_{dd} = \mathcal{O}(d_z N_{dd} a^{2(N_{dd}-1)}) \;\; \text{and} \;\;
L^{\SF}_{ed} = \mathcal{O}(d^2_z N_{ed} a^{2(N_{ed}-1)})\eqsp,$$ 
where constants in these bounds depend on additional terms, but we focus here only on the dimensions of the latent space and the number of layers in the model architecture.
\end{proof}

\begin{proof}[Proof of Theorem \ref{thm:supp:con_gauss_score_function}]
Using Lemmas \ref{lem:ELBO_A1_gauss} and \ref{lem:ELBO_smooth_gauss}, we ensure that Assumptions \ref{ass:A1} and \ref{ass:A2}(i) are satisfied. The proof is then completed by applying Theorem \ref{th:rate_general_Adam}.
\end{proof}

\subsubsection{Analysis with the pathwise gradient}
\label{sec:pathwise_gauss}

We now analyze the convergence rate using the pathwise gradient in the Gaussian case.

\begin{theorem}\label{thm:supp:con_gauss_pathwise}
Let $c_0 > 0$, and consider 
\begin{align*}
 \mathcal{F}_{dd} &= \{ (x,z)\mapsto p_{\theta}(x|z) = \mathcal{N} (x;G_{\theta}(z), c^2 \mathrm{I}_{d_x}) \,|\, G_{\theta}(z) \in \mathcal{F}_G, \theta \in \Theta \subseteq \Rset^{d_{\theta}}, c > c_0 \}\eqsp,\\
 \mathcal{F}_{ed} &= \{ (x,z)\mapsto q_{\phi}(z|x) = \mathcal{N} (z;\mu_{\phi}(x), \Sigma_{\phi}(x)) \,|\, (\mu_{\phi}(x), \Sigma_{\phi}(x)) \in \mathcal{F}_{\mu, \Sigma}, \phi \in \Phi \subseteq \Rset^{d_{\phi}} \}\eqsp.
\end{align*}
Assume that there exists $C_{rec} \in \mathsf{M}(\Xset \times \Zset)$ such that $\left\| x - G_{\theta}(z) \right\| \leq C_{rec}(x, z)$ for all $\theta \in \Theta$ and $(x, z) \in \Xset \times \Zset$.
Assume also that the data distribution $\pi$ has a finite fourth moment, and that there exists some constant $a$ such that for all $\theta \in \Theta$ and $\phi \in \Phi$,
$$ 
\lVert \theta \rVert_{\infty} + \lVert \phi \rVert_{\infty} \leq a\eqsp. 
$$
Let $\left(\theta_{n},\phi_{n}\right) \in \Theta \times \Phi$ be the $n$-th iterate of the recursion in Algorithm \ref{alg:adam}, where $\gamma_{n} = C_{\gamma}n^{-1/2}$ with $C_{\gamma}>0$. Assume that $\beta_1 < \sqrt{\beta_2} < 1$.
For any $n \geq 1$, let $R \in \{0, \ldots, n\}$ be a uniformly distributed random variable. Then,
$$
\mathbb{E}\left[\left\| \nabla \mathcal{L}\left(\theta_{R}, \phi_{R}\right)\right\|^{2}\right] = \mathcal{O}\left(\frac{\mathcal{L}^*}{\sqrt{n}} + d_{z} \frac{N_{\text{total}} a^{2(N_{\text{total}}-1)}}{1-\beta_{1}} \frac{d^{*} \log n}{\sqrt{n}}\right)\eqsp,
$$
where $\mathcal{L}^* = \mathcal{L}\left(\theta^{*}, \phi^{*}\right) - \mathcal{L}\left(\theta_{0}, \phi_{0}\right)$, $d^{*} = d_{\theta} + d_{\phi}$ is the total dimension of the parameters, $N_{\text{total}} = N_{ed} + N_{dd}$ represents the total number of layers in the model architecture.
\end{theorem}

\begin{lemma}\label{lem:ELBO_pathwise_A1_gauss}
Let $c_0 > 0$, and consider 
\begin{align*}
  \mathcal{F}_{dd} &= \{ (x,z)\mapsto p_{\theta}(x|z) = \mathcal{N} (x;G_{\theta}(z), c^2 \mathrm{I}_{d_x}) \,|\, G_{\theta}(z) \in \mathcal{F}_G, \theta \in \Theta \subseteq \Rset^{d_{\theta}}, c > c_0 \}\eqsp,\\
  \mathcal{F}_{ed} &= \{ (x,z)\mapsto q_{\phi}(z|x) = \mathcal{N} (z;\mu_{\phi}(x), \Sigma_{\phi}(x)) \,|\, (\mu_{\phi}(x), \Sigma_{\phi}(x)) \in \mathcal{F}_{\mu, \Sigma}, \phi \in \Phi \subseteq \Rset^{d_{\phi}} \}\eqsp.
\end{align*}
Assume that there exists $C_{rec} \in \mathsf{M}(\Xset \times \Zset)$ such that $\left\| x - G_{\theta}(z) \right\| \leq C_{rec}(x, z)$ for all $\theta \in \Theta$ and $(x, z) \in \Xset \times \Zset$.
Assume also that the data distribution $\pi$ has a finite fourth moment, and that there exists some constant $a$ such that for all $\theta \in \Theta$ and $\phi \in \Phi$,
$$ 
\lVert \theta \rVert_{\infty} + \lVert \phi \rVert_{\infty} \leq a\eqsp. 
$$
Then, Assumptions \ref{ass:A2}(ii) and \ref{ass:A2_pathwise} are satisfied.
\end{lemma}

\begin{proof}
From the analysis of the boundedness and smoothness of the neural network (Lemma \ref{lem:smooth_NN}), we can establish both the boundedness of the gradient and the smoothness of the functions $\phi \mapsto \mu_{\phi}(x)$ and $\phi \mapsto \Sigma_{\phi}(x)$.
Notably, by following a similar reasoning to that used in the proof of Lemma \ref{lem:ELBO_smooth_gauss}, we can handle $\Sigma_{\phi}(x) = \text{Diag}(\sigma^2_1(x), \ldots, \sigma^2_{d_{z}}(x))$ by introducing the transformed variable $l_{\phi}(x) = (\log \sigma^2_1(x), \ldots, \log \sigma^2_{d_{z}}(x))$, which represents the logarithm of the variances. Although the proof is formulated using $\Sigma_{\phi}(x)$, it can be easily adapted to the case of $l_{\phi}(x)$.
In the following, each point of Assumptions \ref{ass:A2}(ii) and \ref{ass:A2_pathwise} is verified one by one.

\paragraph{Boundedness of the gradient of log density. }
For all $\theta \in \Theta$, $\phi \in \Phi$, $x \in \Xset$ and $\varepsilon, z \in \Zset$ such that $z = g(\varepsilon, \phi)$, we have:
\begin{align*}
\left\|\nabla_{z} \log p_{\theta}(x, z)\right\| &\leq \left\|\nabla_{z} \log p_{\theta}(x|z)\right\| + \left\|\nabla_{z} \log p_{\theta}(z)\right\| \\
&\leq \left\|\nabla_{z} G_{\theta}(z)^\top  ( x - G_{\theta}(z))\right\| + \left\|z\right\| \\
&\leq \left\|\nabla_{z} G_{\theta}(z)\right\| \left\|x - G_{\theta}(z)\right\| + \left\|z\right\|\eqsp.
\end{align*}
Using Lemma \ref{lem:smooth_NN}, we get:
\begin{align*}
\left\|\nabla_{z} \log p_{\theta}(x, z)\right\| &\leq C_{rec}(x, z) \left\|z\right\| a^{N_{dd}} \prod_{i=1}^{N_{dd}} M_{f_i} + \left\|z\right\| \\
&\leq \left(\left\|\mu_{\phi}(x)\right\| + \left\|\Sigma_{\phi}(x)\right\|^{1/2} \left\|\varepsilon\right\| \right) \left(1+C_{rec}(x, g(\varepsilon, \phi)) a^{N_{dd}} \prod_{i=1}^{N_{dd}} M_{f_i} \right)\eqsp \\
&\leq \left(C_{\mu} + C_{\Sigma}^{1/2} \left\|\varepsilon\right\| \right) \left(1+C_{rec}(x, g(\varepsilon, \phi)) a^{N_{dd}} \prod_{i=1}^{N_{dd}} M_{f_i} \right)\eqsp,
\end{align*}
where we used $z = \mu_{\phi}(x) + \Sigma_{\phi}(x)^{1/2} \varepsilon$.
For the variational density, the gradient is bounded as follows:
\begin{align*}
\left\|\nabla_{z} \log q_{\phi}(z|x)\right\| &= \left\| \Sigma_{\phi}(x)^{-1}(z - \mu_{\phi}(x))\right\| \\
&= \left\|\Sigma_{\phi}(x)^{-1/2} \varepsilon\right\| \\
&\leq c_{\Sigma}^{-1/2}\left\|\varepsilon\right\|\eqsp.
\end{align*}

\paragraph{Lipschitz condition on the gradient of log density. }

\begin{align} \label{eq:lip_theta_z}
\left\|\nabla_{z} \log p_{\theta}(x, z) - \nabla_{z} \log p_{\theta'}(x, z')\right\| &\leq \left\|\nabla_{z} \log p_{\theta}(x, z) - \nabla_{z} \log p_{\theta'}(x, z)\right\| + \left\|\nabla_{z} \log p_{\theta'}(x, z) - \nabla_{z} \log p_{\theta'}(x, z')\right\|\eqsp.
\end{align}
Given that $\nabla_{z} \log p_{\theta}(x, z) = \nabla_{z} G_{\theta}(z)^\top  ( x - G_{\theta}(z))$, we can derive the following bound for the first term in \eqref{eq:lip_theta_z}:
\begin{align*}
\left\|\nabla_{z} \log p_{\theta}(x, z) - \nabla_{z} \log p_{\theta'}(x, z)\right\| &\leq \left\|\nabla_{z} G_{\theta}(z)^\top  ( x - G_{\theta}(z)) - \nabla_{z} G_{\theta'}(z)^\top  ( x - G_{\theta'}(z))\right\| \\
&\leq \left\|\nabla_{z} G_{\theta}(z)^\top  ( x - G_{\theta}(z)) - \nabla_{z} G_{\theta'}(z)^\top  ( x - G_{\theta}(z))\right\| \\
&\quad + \left\|\nabla_{z} G_{\theta'}(z)^\top  ( x - G_{\theta}(z)) - \nabla_{z} G_{\theta'}(z)^\top  ( x - G_{\theta'}(z))\right\| \\
&\leq \left\|x - G_{\theta}(z)\right\| \left\|\nabla_{z} G_{\theta}(z) - \nabla_{z} G_{\theta'}(z)\right\| + \left\|\nabla_{z} G_{\theta'}(z)\right\| \left\|G_{\theta}(z) - G_{\theta'}(z)\right\| \\
&\leq C_{rec}(x, z) \left\|\nabla_{z} G_{\theta}(z) - \nabla_{z} G_{\theta'}(z)\right\| + \left\|\nabla_{z} G_{\theta'}(z)\right\| \left\|G_{\theta}(z) - G_{\theta'}(z)\right\|\eqsp.
\end{align*}
For the second term in \eqref{eq:lip_theta_z}, we get:
\begin{align*}
\left\|\nabla_{z} \log p_{\theta}(x, z) - \nabla_{z} \log p_{\theta}(x, z')\right\| &\leq \left\|\nabla_{z} G_{\theta}(z)^\top  ( x - G_{\theta}(z)) - \nabla_{z} G_{\theta}(z')^\top  ( x - G_{\theta}(z'))\right\| \\
&\leq \left\|\nabla_{z} G_{\theta}(z)^\top  ( x - G_{\theta}(z)) - \nabla_{z} G_{\theta}(z')^\top  ( x - G_{\theta}(z))\right\| \\
&\quad + \left\|\nabla_{z} G_{\theta}(z')^\top  ( x - G_{\theta}(z)) - \nabla_{z} G_{\theta'}(z')^\top  ( x - G_{\theta}(z'))\right\| \\
&\leq \left\|x - G_{\theta}(z)\right\| \left\|\nabla_{z} G_{\theta}(z) - \nabla_{z} G_{\theta}(z')\right\| + \left\|\nabla_{z} G_{\theta}(z')\right\| \left\|G_{\theta}(z) - G_{\theta}(z')\right\| \\
&\leq C_{rec}(x, z) \left\|\nabla_{z} G_{\theta}(z) - \nabla_{z} G_{\theta}(z')\right\| + \left\|\nabla_{z} G_{\theta}(z')\right\| \left\|G_{\theta}(z) - G_{\theta}(z')\right\|\eqsp.
\end{align*}

By combining these two terms and applying Lemmas \ref{lem:smooth_z_NN} and \ref{lem:grad_z_smooth_NN}, we obtain:
\begin{align*}
&\left\|\nabla_{z} \log p_{\theta}(x, z) - \nabla_{z} \log p_{\theta'}(x, z')\right\| \leq C_{rec}(x, g(\varepsilon, \phi)) \sum_{k=1}^{N_{dd}} L_{f_k} a^{N_{dd}+k} \prod_{i=1}^{k-1} M_{f_i}^2 \prod_{i=k+1}^{N_{dd}} M_{f_i} \left\| z - z'\right\| \\
&+ a^{2N_{dd}} \prod_{i=1}^{N_{dd}} M_{f_i}^2 \left\| z - z'\right\| + C_{rec}(x, g(\varepsilon, \phi))\left(C_{\mu} + \left\|\varepsilon\right\| C_{\Sigma}^{1/2}\right) \sum_{k=1}^{N_{dd}} L_{f_k} a^{N_{dd}-1+k} \prod_{i=1}^{N_{dd}-1} M_{f_i}^2 \prod_{i=k+1}^{N_{dd}} M_{f_i} \left\| \theta - \theta'\right\| \\
&+ C_{rec}(x, g(\varepsilon, \phi)) a^{N_{dd}-1} \prod_{i=1}^{N_{dd}} M_{f_i} \left\| \theta - \theta'\right\| + \left(C_{\mu}^{2} + \left\|\varepsilon\right\|^{2} C_{\Sigma}\right) a^{2N_{dd}-1} \prod_{i=1}^{N_{dd}} M_{f_i}^2 \left\| \theta - \theta'\right\|\eqsp.
\end{align*}

For the variational density, using Lemma \ref{lemma:tech_eq}, we obtain:
\begin{align*}
\left\|\nabla_{z} \log q_{\phi}(z|x) - \nabla_{z} \log q_{\phi'}(z'|x)\right\| &\leq \left\|\nabla_{z} \log q_{\phi}(z|x) - \nabla_{z} \log q_{\phi'}(z|x)\right\| + \left\|\nabla_{z} \log q_{\phi'}(z|x) - \nabla_{z} \log q_{\phi'}(z'|x)\right\| \\
&\leq \left\|\varepsilon\right\| \left\|\Sigma_{\phi}(x)^{-1/2} - \Sigma_{\phi'}(x)^{-1/2}\right\| \\
&\quad + \left\|\Sigma_{\phi'}(x)^{-1}(z - \mu_{\phi'}(x)) - \Sigma_{\phi'}(x)^{-1}(z' - \mu_{\phi'}(x))\right\| \\
&\leq \left\|\varepsilon\right\| \frac{1}{2} c_{\Sigma}^{-3/2} \|\Sigma_{\phi}(x) - \Sigma_{\phi'}(x)\| + c_{\Sigma}^{-1} \|z - z'\| \\
&\leq \left\|\varepsilon\right\| \frac{1}{2} c_{\Sigma}^{-3/2} \left\|x\right\| a^{N_{ed}-1} \prod_{i=1}^{N_{ed}} M_{f_i} \|\phi - \phi'\| + c_{\Sigma}^{-1} \|z - z'\|\eqsp,
\end{align*}
where we used the Lipschitz condition of $\phi \mapsto \Sigma_{\phi}(x)$.

\paragraph{Lipschitz and smoothness condition on $g$. }
\begin{align*}
\left\|\nabla_{\phi} g(\varepsilon, \phi)\right\| &\leq \left\|\nabla_{\phi} \mu_{\phi}(x) + \frac{1}{2}
\varepsilon \Sigma_{\phi}(x)^{-1/2} \nabla_{\phi} \Sigma_{\phi}(x)\right\| \\
&\leq \left\|\nabla_{\phi} \mu_{\phi}(x)\right\| + \frac{1}{2} \left\|\varepsilon\right\| \left\|\Sigma_{\phi}(x)^{-1/2}\right\| \left\|\nabla_{\phi} \Sigma_{\phi}(x)\right\| \\
&\leq \left\|\nabla_{\phi} \mu_{\phi}(x)\right\| + \frac{1}{2} \left\|\varepsilon\right\| c_{\Sigma}^{-1/2} \left\|\nabla_{\phi} \Sigma_{\phi}(x)\right\| \\
&\leq \left(1 + \frac{1}{2} \left\|\varepsilon\right\| c_{\Sigma}^{-1/2}\right) \left\|x\right\| a^{N_{ed}-1} \prod_{i=1}^{N_{ed}} M_{f_i}\eqsp.
\end{align*}

\begin{align*}
\left\|\nabla_{\phi} g(\varepsilon, \phi) - \nabla_{\phi} g(\varepsilon, \phi')\right\| &\leq \left\|\nabla_{\phi} \mu_{\phi}(x) - \nabla_{\phi} \mu_{\phi'}(x)\right\| + \frac{1}{2} \left\|\varepsilon\right\| \left\|\Sigma_{\phi}(x)^{-1/2} \nabla_{\phi} \Sigma_{\phi}(x) - \Sigma_{\phi'}(x)^{-1/2} \nabla_{\phi} \Sigma_{\phi'}(x)\right\| \\
&\leq \left\|\nabla_{\phi} \mu_{\phi}(x) - \nabla_{\phi} \mu_{\phi'}(x)\right\| + \frac{1}{2} \left\|\varepsilon\right\| c_{\Sigma}^{-1/2} \left\|\nabla_{\phi} \Sigma_{\phi}(x) - \nabla_{\phi} \Sigma_{\phi'}(x)\right\| \\
&\quad + \frac{1}{2} \left\|\varepsilon\right\| \left\|\nabla_{\phi} \Sigma_{\phi}(x)\right\| \left\|\Sigma_{\phi}(x)^{-1/2} - \Sigma_{\phi'}(x)^{-1/2} \right\| \\
&\leq \left\|\nabla_{\phi} \mu_{\phi}(x) - \nabla_{\phi} \mu_{\phi'}(x)\right\| + \frac{1}{2} \left\|\varepsilon\right\| c_{\Sigma}^{-1/2} \left\|\nabla_{\phi} \Sigma_{\phi}(x) - \nabla_{\phi} \Sigma_{\phi'}(x)\right\| \\
&\quad + \frac{1}{4} \left\|\varepsilon\right\| \left\|\nabla_{\phi} \Sigma_{\phi}(x)\right\| c_{\Sigma}^{-3/2} \|\Sigma_{\phi}(x) - \Sigma_{\phi'}(x)\| \\
&\leq \left(1 + \frac{1}{2} \left\|\varepsilon\right\| c_{\Sigma}^{-1/2}\right) N_{ed} \left(\|x\|^2 + 1\right) \sum_{k=1}^{N_{ed}} L_{f_k} a^{N_{ed}-2+k} \prod_{i=1}^{k-1} M_{f_i}^2 \prod_{i=k+1}^{N_{ed}} M_{f_i} \left\|\phi - \phi'\right\| \\
&\quad + \frac{1}{4} \left\|\varepsilon\right\| c_{\Sigma}^{-3/2} \left\|x\right\|^2 a^{2(N_{ed}-1)} \prod_{i=1}^{N_{ed}} M_{f_i}^2 \left\|\phi - \phi'\right\|\eqsp,
\end{align*}
where we used the Lipschitz and the smoothness of $\mu_{\phi}(x)$ and $\Sigma_{\phi}(x)$.
\end{proof}

\begin{lemma}\label{lem:ELBO_smooth_pathwise_gauss_finite}
Under the assumptions of Lemma \ref{lem:ELBO_smooth_gauss}, the smoothness constant $L^{\PW}$ of the expected ELBO is well-defined and finite.
\end{lemma}

\begin{proof}
Using \ref{lemma:ELBO_smooth_score}, the smoothness constant $L^{\PW}$ is defined by: 
\begin{align*}
L^{\PW} &= \mathbb{E}_{\pi, p_\varepsilon}\left[ L_{p}(x,\varepsilon) + M_{g}(x,\varepsilon)^2 \left(L_{p}(x,\varepsilon) + 2L_{q}(x,\varepsilon) \right) + 3 L_{g}(x,\varepsilon) M(x,\varepsilon) + 2M_{g}(x,\varepsilon) L_{q}(x,\varepsilon) \right] \\
&\quad + \mathbb{E}_{\pi, p_\varepsilon}\left[ L_{p}(x,\varepsilon) M_{g}(x,\varepsilon) \right]\eqsp.
\end{align*}

Let $L^{\PW}_{dd}$ and $L^{\PW}_{ed}$ denote the smoothness constants of the decoder and the encoder, respectively.
For the smoothness constant of the decoder, applying Lemma \ref{lem:smooth_NN}, we have:
\begin{align*}
L^{\PW}_{dd} &= \mathbb{E}_{\pi, p_\varepsilon}\left[ L_{p}(x,\varepsilon)\right]\eqsp,
\end{align*}
where 
\begin{align*}
&L_{p}(x,\varepsilon) = C_{rec}(x, g(\varepsilon, \phi)) \sum_{k=1}^{N_{dd}} L_{f_k} a^{N_{dd}+k} \prod_{i=1}^{k-1} M_{f_i}^2 \prod_{i=k+1}^{N_{dd}} M_{f_i} + a^{2N_{dd}} \prod_{i=1}^{N_{dd}} M_{f_i}^2 + \left(C_{\mu}^{2} + \left\|\varepsilon\right\|^{2} C_{\Sigma}\right) a^{2N_{dd}-1} \prod_{i=1}^{N_{dd}} M_{f_i}^2 \\
&\quad + C_{rec}(x, g(\varepsilon, \phi)) a^{N_{dd}-1} \prod_{i=1}^{N_{dd}} M_{f_i} + C_{rec}(x, g(\varepsilon, \phi))\left(C_{\mu} + \left\|\varepsilon\right\| C_{\Sigma}^{1/2}\right) \sum_{k=1}^{N_{dd}} L_{f_k} a^{N_{dd}-1+k} \prod_{i=1}^{N_{dd}-1} M_{f_i}^2 \prod_{i=k+1}^{N_{dd}} M_{f_i}\eqsp.
\end{align*}
Since 
$$\mathbb{E}_{p_{\varepsilon}}\left[\left\|\varepsilon\right\|\right] = \sqrt{2} \frac{\Gamma\left(\frac{d_z+1}{2}\right)}{\Gamma\left(\frac{d_z}{2}\right)}\eqsp,$$
it follows that $L^{\PW}_{ed}$ is well-defined and finite.

For the smoothness constant of the encoder, we have:
$$L^{\PW}_{ed} = \mathbb{E}_{\pi, p_\varepsilon}\left[ M_{g}(x,\varepsilon)^2 \left(L_{p}(x,\varepsilon) + 2L_{q}(x,\varepsilon) \right) + 3 L_{g}(x,\varepsilon) M(x,\varepsilon) + 2M_{g}(x,\varepsilon) L_{q}(x,\varepsilon) \right]\eqsp,$$
where
\begin{align*}
M(x,\varepsilon) &= \left(C_{\mu} + C_{\Sigma}^{1/2} \left\|\varepsilon\right\| \right) \left(1+C_{rec}(x, g(\varepsilon, \phi)) a^{N_{dd}} \prod_{i=1}^{N_{dd}} M_{f_i} \right) + c_{\Sigma}^{-1/2} \left\|\varepsilon\right\|\eqsp, \\
L_{q}(x,\varepsilon) &= c_{\Sigma}^{-1} + \left\|\varepsilon\right\| \frac{1}{2} c_{\Sigma}^{-3/2} \left\|x\right\| a^{N_{ed}-1} \prod_{i=1}^{N_{ed}} M_{f_i}\eqsp, \\
M_{g}(x,\varepsilon) &= \left(1 + \frac{1}{2} \left\|\varepsilon\right\| c_{\Sigma}^{-1/2}\right) \left\|x\right\| a^{N_{ed}-1} \prod_{i=1}^{N_{ed}} M_{f_i}\eqsp, \\
L_{g}(x,\varepsilon) &= \left(1 + \frac{1}{2} \left\|\varepsilon\right\| c_{\Sigma}^{-1/2}\right) N_{ed} \left(\|x\|^2 + 1\right) \sum_{k=1}^{N_{ed}} L_{f_k} a^{N_{ed}-2+k} \prod_{i=1}^{k-1} M_{f_i}^2 \prod_{i=k+1}^{N_{ed}} M_{f_i} \\
&\quad + \frac{1}{4} \left\|\varepsilon\right\| c_{\Sigma}^{-3/2} \left\|x\right\|^2 a^{2(N_{ed}-1)} \prod_{i=1}^{N_{ed}} M_{f_i}^2\eqsp.
\end{align*}
Since all the terms are well-defined, it is evident that $L^{\PW}_{dd}$ is also well-defined and finite.
Moreover, under the conditions $M_{f_i} \leq 1$ and $L_{f_i} \leq 1$ for all $1 \leq i \leq N$, we can establish the following bounds: 
$$L^{\PW}_{dd} = \mathcal{O}(\sqrt{d_z} N_{dd} a^{2N_{dd}}) \;\; \text{and} \;\;
L^{\PW}_{ed} = \mathcal{O}(d_z N_{ed} a^{2(N_{ed}-1)} a^{2N_{dd}}) = \mathcal{O}(d_z N_{ed} a^{2(N_{ed}+N_{dd}-1)})\eqsp,$$ 
where the constants in these bounds depend on additional terms. However, we focus here solely on the dimensions of the latent space and the number of layers in the model architecture, similar to the case of the score function.
\end{proof}

\begin{proof}[Proof of Theorem \ref{thm:supp:con_gauss_pathwise}]
Using Lemma \ref{lem:ELBO_pathwise_A1_gauss}, we ensure that Assumptions \ref{ass:A2}(ii) and \ref{ass:A2_pathwise} are satisfied. The proof is then completed by applying Theorem \ref{th:rate_general_Adam}.
\end{proof}

\section{IMPORTANCE WEIGHTED AUTOENCODER}
\label{supp:sec:iwae}

\subsection{Convergence Rate of IWAE with Respect to Marginal Log Likelihood}

Denoting the normalized importance weights as
$$\widetilde w_{\theta, \phi}(x,z^{(\ell)}) = \frac{w_{\theta, \phi}(x,z^{(\ell)})}{\sum_{\ell=1}^{K} w_{\theta, \phi}(x,z^{(\ell)})}\eqsp,$$
the gradient of the ELBO in IWAE can be expressed as:
\begin{equation*}
\nabla_{\theta, \phi} \mathcal{L}^{\mathsf{IS}}_K(\theta, \phi) = \mathbb{E}_{\pi} \left[ \mathbb{E}_{q^{\otimes K}_{\phi}(\cdot|x)} \left[ \sum_{\ell=1}^{K} \widetilde w_{\theta, \phi}(x,z^{(\ell)}) \nabla_{\theta, \phi} \log w_{\theta, \phi}(x,z^{(\ell)}) \right] \right]\eqsp.
\end{equation*}
Using the reparametrization trick, this can be rewritten as:
\begin{equation*}
\nabla_{\theta, \phi} \mathcal{L}^{\mathsf{IS}}_K(\theta, \phi) = \mathbb{E}_{\pi} \left[ \mathbb{E}_{p^{\otimes K}_{\varepsilon}} \left[ \sum_{\ell=1}^{K} \widetilde w_{\theta, \phi}(x, g(\varepsilon^{(\ell)}, \phi)) \nabla_{\theta, \phi} \log w_{\theta, \phi}(x,g(\varepsilon^{(\ell)}, \phi)) \right] \right]\eqsp.
\end{equation*}
Given a batch of observations $\{x_i\}_{i=1}^B$, the estimator for the gradient of this ELBO can be expressed as:
\begin{equation*}
\widehat{\nabla}_{\theta, \phi} \mathcal{L}^{\mathsf{IS}}_K(\theta, \phi; \{x_i\}_{i=1}^B) = \frac{1}{B} \sum_{i=1}^{B} \sum_{\ell=1}^{K} \frac{w_{\theta, \phi}(x_i,z_i^{(\ell)})}{\sum_{\ell=1}^{k} w_{\theta, \phi}(x_i,z_i^{(\ell)})} \nabla_{\theta, \phi} \log w_{\theta, \phi}(x_i,z_i^{(\ell)})\eqsp.
\end{equation*}
where, for all $1 \leq i \leq B$ and $1 \leq \ell \leq K$, $z_i^{(\ell)} = g(\varepsilon_i^{(\ell)}, \phi)$ with $\varepsilon_i^{(\ell)}$ being independent samples drawn from $p_{\varepsilon}$.
This estimator can be viewed as a biased gradient estimator for the marginal log-likelihood $\nabla_{\theta} \log p_{\theta}(x)$.  The non-asymptotic bound on the bias of this gradient estimator, as established in \cite[Theorem B.1]{surendran2024non}, indicates that
$$
\left\| \mathbb{E}_{q^{\otimes k}_{\phi}(\cdot|x)} \left[ \widehat{\nabla}_{\theta} \mathcal{L}^{\mathsf{IS}}_k(\theta, \phi; x) - \nabla_{\theta} \log p_{\theta}(x) \right] \right\| = \mathcal{O}\left(\frac{1}{k}\right)\eqsp.
$$
This result indicates that as the number of samples $K$ increases, the bias in the gradient estimator diminishes at a rate inversely proportional to $K$. Furthermore, \cite{surendran2024non} treats IWAE as a biased gradient and establishes a convergence rate of $\mathcal{O}(\log n/\sqrt{n} + b_n)$, where $b_n$ is related to the bias at iteration $n$. However, they do not explicitly verify all assumptions; thus, our results enable us to theoretically derive the following convergence rate for the IWAE with respect to the marginal log-likelihood:
$$
\mathbb{E}\left[\left\| \nabla_{\theta} \log p_{\theta_{R}}(x) \right\|^{2}\right] = \mathcal{O}\left(\frac{\log n}{\sqrt{n}} + b_n\right)\eqsp,
$$
where $R \in \{0, \ldots, n\}$ is a uniformly distributed random variable.

\subsection{Proof of Theorem \ref{th:con_IWAE}}

\begin{proof}
First, for all $K \in \mathbb{N}$, $x \in \Xset$, $z \in \Zset$ and $1 \leq \ell \leq K$, the Lipschitz condition and smoothness of $w_{\theta, \phi}(x,z^{(\ell)})$ with respect $\theta$ and $\phi$ follow from those of $\theta \mapsto \log p_{\theta}(x, z)$ and $\phi \mapsto \log q_{\phi}(z|x)$.
Next, we establish that $\widetilde w_{\theta, \phi}(x,z^{(\ell)})$ is also Lipschitz continuous and smooth with respect to $\theta$ and $\phi$ for all $K \in \mathbb{N}$, $x \in \Xset$, $z \in \Zset$ and $1 \leq \ell \leq K$.
For all $1 \leq \ell \leq K$, let $\widetilde w_{\theta, \phi}^{(\ell)} := \widetilde w_{\theta, \phi}(x,z^{(\ell)})$.
For $K=1$, $| \widetilde w_{\theta, \phi}^{(\ell)} - \widetilde w_{\theta', \phi'}^{(\ell)} | = 0$. For $K>1$,
\begin{align*}
\left| \widetilde w_{\theta, \phi}^{(\ell)} - \widetilde w_{\theta', \phi}^{(\ell)} \right| &= \left| \frac{w_{\theta, \phi}^{(\ell)}}{\sum_{\ell=1}^{K} w_{\theta, \phi}^{(\ell)}} - \frac{w_{\theta', \phi}^{(\ell)}}{\sum_{\ell=1}^{K} w_{\theta', \phi}^{(\ell)}} \right| \\
&\leq \left| \frac{w_{\theta, \phi}^{(\ell)}}{\sum_{\ell=1}^{K} w_{\theta, \phi}^{(\ell)}} - \frac{w_{\theta', \phi}^{(\ell)}}{\sum_{\ell=1}^{K} w_{\theta, \phi}^{(\ell)}} \right| + \left| \frac{w_{\theta', \phi}^{(\ell)}}{\sum_{\ell=1}^{K} w_{\theta, \phi}^{(\ell)}} - \frac{w_{\theta', \phi}^{(\ell)}}{\sum_{\ell=1}^{K} w_{\theta', \phi}^{(\ell)}} \right| \\
&\leq \left| \frac{1}{\sum_{\ell=1}^{K} w_{\theta, \phi}^{(\ell)}} \right| \left| w_{\theta, \phi}^{(\ell)} - w_{\theta', \phi}^{(\ell)} \right| + \left| \frac{w_{\theta', \phi}^{(\ell)}}{\sum_{\ell=1}^{K} w_{\theta, \phi}^{(\ell)} \sum_{\ell=1}^{K} w_{\theta', \phi}^{(\ell)}} \right| \left| \sum_{\ell=1}^{K} w_{\theta, \phi}^{(\ell)} - \sum_{\ell=1}^{K} w_{\theta', \phi}^{(\ell)} \right| \\
&\leq \frac{1}{K}\frac{\alpha^+(x,\varepsilon)}{\alpha^-(x,\varepsilon)} M(x,\varepsilon) \left\| \theta - \theta' \right\| + 2\frac{1}{K} \frac{\alpha^+(x,\varepsilon)^3}{\alpha^-(x,\varepsilon)^3} M(x,\varepsilon) \left\| \theta - \theta' \right\|,
\end{align*}
where we used Assumption \ref{ass:A3}. This concludes the Lipschitz continuity with respect to $\theta$.
The Lipschitz condition of $\widetilde w_{\theta, \phi}^{(\ell)}$ with respect to $\phi$ is treated in the same manner as with $\theta$.

For a given observation $x \in \Xset$, using the Lipschitz continuity of $(\theta, \phi) \mapsto \widetilde w_{\theta, \phi}^{(\ell)}$, we get:
\begin{align*}
&\left\| \nabla_{\theta, \phi} \mathcal{L}^{\mathsf{IS}}_K(\theta, \phi; x) - \nabla_{\theta, \phi} \mathcal{L}^{\mathsf{IS}}_K(\theta', \phi'; x) \right\| \\
&= \left\| \mathbb{E}_{p^{\otimes K}_{\varepsilon}} \left[ \sum_{\ell=1}^{K} \widetilde w_{\theta, \phi}(x, g(\varepsilon^{(\ell)}, \phi)) \nabla_{\theta, \phi} \log w_{\theta, \phi}(x,g(\varepsilon^{(\ell)}, \phi)) - \sum_{\ell=1}^{K} \widetilde w_{\theta', \phi'}(x, g(\varepsilon^{(\ell)}, \phi')) \nabla_{\theta, \phi} \log w_{\theta', \phi'}(x,g(\varepsilon^{(\ell)}, \phi')) \right] \right\| \\
&\leq \mathbb{E}_{p^{\otimes K}_{\varepsilon}} \left[ \left\| \sum_{\ell=1}^{K} \widetilde w_{\theta, \phi}(x, g(\varepsilon^{(\ell)}, \phi)) \nabla_{\theta, \phi} \log w_{\theta, \phi}(x,g(\varepsilon^{(\ell)}, \phi)) - \widetilde w_{\theta, \phi}(x, g(\varepsilon^{(\ell)}, \phi)) \nabla_{\theta, \phi} \log w_{\theta', \phi'}(x,g(\varepsilon^{(\ell)}, \phi')) \right\| \right] \\
&\quad + \mathbb{E}_{p^{\otimes K}_{\varepsilon}} \left[ \left\| \sum_{\ell=1}^{K} \widetilde w_{\theta, \phi}(x, g(\varepsilon^{(\ell)}, \phi)) \nabla_{\theta, \phi} \log w_{\theta, \phi'}(x,g(\varepsilon^{(\ell)}, \phi')) - \widetilde w_{\theta', \phi'}(x, g(\varepsilon^{(\ell)}, \phi')) \nabla_{\theta, \phi} \log w_{\theta', \phi'}(x,g(\varepsilon^{(\ell)}, \phi')) \right\| \right] \\
&\leq \mathbb{E}_{p^{\otimes K}_{\varepsilon}} \left[ \sum_{\ell=1}^{K} \widetilde w_{\theta, \phi}(x, g(\varepsilon^{(\ell)}, \phi)) \left\| \nabla_{\theta, \phi} \log w_{\theta, \phi}(x,g(\varepsilon^{(\ell)}, \phi)) - \nabla_{\theta, \phi} \log w_{\theta', \phi'}(x,g(\varepsilon^{(\ell)}, \phi')) \right\| \right] \\
&\quad + \mathbb{E}_{p^{\otimes K}_{\varepsilon}} \left[ \sum_{\ell=1}^{K} \left\| \nabla_{\theta, \phi} \log w_{\theta, \phi'}(x,g(\varepsilon^{(\ell)}, \phi')) \right\| \left\| \widetilde w_{\theta, \phi}(x, g(\varepsilon^{(\ell)}, \phi)) - \widetilde w_{\theta', \phi'}(x, g(\varepsilon^{(\ell)}, \phi')) \right\| \right] \\
&\leq L^{\PW} \left\| \left(\theta, \phi\right) - \left(\theta', \phi'\right) \right\| + \mathbf{1}_{K>1}\frac{1}{K} \mathbb{E}_{p_{\varepsilon}} \left[ \frac{\alpha^+(x,\varepsilon)}{\alpha^-(x,\varepsilon)} M(x,\varepsilon) + 2\frac{\alpha^+(x,\varepsilon)^3}{\alpha^-(x,\varepsilon)^3} M(x,\varepsilon) \right] \left\| \left(\theta, \phi\right) - \left(\theta', \phi'\right) \right\|\eqsp,
\end{align*}
where in the last inequality, we applied Lemma \ref{lemma:ELBO_smooth_pathwise} to the first term, and for the second term, we used Assumption \ref{ass:A2}(ii) along with the Lipschitz continuity of $\widetilde w_{\theta, \phi}^{(\ell)}$.
Taking the expectation over $\pi$ then leads to the conclusion that the ELBO for the IWAE is $L_{K}$-smooth, where
\begin{equation} \label{eq:smooth_iwae}
L_{K} = L^{\PW} + \mathbf{1}_{K>1}\mathbb{E}_{\pi, p_{\varepsilon}} \left[ M(x,\varepsilon) \alpha^+(x,\varepsilon) / \alpha^-(x,\varepsilon) + 2 M(x,\varepsilon) \alpha^+(x,\varepsilon)^3 / \alpha^-(x,\varepsilon)^3 \right]/K\eqsp.
\end{equation}
\end{proof}

\section{BLACK-BOX VARIATIONAL INFERENCE}\label{app:BBVI}

In the following, we refer to the ELBO in BBVI as ELBO-BBVI to distinguish it from the ELBO in VAE and avoid any confusion.

\subsection{Previous Work on the Convergence of BBVI}

Some existing results on the convergence of BBVI, such as \cite[Theorem 1]{regier2017fast} and \cite[Theorem 1]{buchholz2018quasi}, rely on the assumption of the complete smoothness of the ELBO-BBVI to establish their guarantees.
In \cite{domke2020provable}, it is shown that ELBO-BBVI is strongly concave under certain conditions, specifically when the posterior is strongly log-concave and linear parameterizations are used. However, their analysis also demonstrates that while the energy function (the expectation of the joint likelihood) is smooth, the ELBO-BBVI itself is not smooth. This conclusion is reached by decomposing ELBO-BBVI into the sum of two terms, where the entropic regularization term lacks smoothness. In contrast, \cite{kim2024convergence} shows that ELBO-BBVI can be smooth under specific conditions. 
The smoothness of ELBO-BBVI in their work (Theorem 1, Corollary 1) is established through the Hessian and the location-scale parameterization.
Our approach yields new results for BBVI without assuming the location-scale family and apply to a broader range of reparameterization families.

\subsection{Comparison of Our Results with Existing Work}

The following corollary provides a more detailed version of Corollary \ref{cor:ELBO_BBVI2}.

\begin{corollary}\label{cor:ELBO_BBVI2}
Assume that the following conditions hold:
There exist $M$, $M_g$, $L_g$, $L_p$, and $L_q \in \mathsf{M}(\Xset \times \Zset)$ such that for all $\phi \in \Phi$, $x \in \Xset$ and $\varepsilon \in \Zset$ with $z = g(\varepsilon, \phi)$,
\vspace{-4mm}
\begin{itemize}
    \item[(i)] $z \mapsto \log p(x,z)$ is $L_{p}(x, \varepsilon)$-smooth.
    \vspace{-1mm}
    \item[(ii)] $z \mapsto \log q_{\phi}(z|x)$ is $M(x, \varepsilon)$-Lipchitz and $L_{q}(x, \varepsilon)$-smooth.
    \vspace{-1mm}
    \item[(iii)] $\phi \mapsto g(\phi, \varepsilon)$ is $M_{g}(x, \varepsilon)$-Lipchitz and $L_{g}(x, \varepsilon)$-smooth.
\end{itemize}
\vspace{-2mm}
Then, $\phi \mapsto \mathcal{L}^{\mathsf{BBVI}}(\phi)$ is $L^{\mathsf{BBVI}}$-smooth, where the smoothness constant $L^{\mathsf{BBVI}}$ is given by:
\begin{equation} \label{eq:elbo_bbvi}
L^{\mathsf{BBVI}} = \mathbb{E}_{\pi, p_\varepsilon}\left[ M_{g}(x,\varepsilon)^2 \left(L_{p}(x,\varepsilon) + 2L_{q}(x,\varepsilon) \right) + 3 L_{g}(x,\varepsilon) M(x,\varepsilon) + 2M_{g}(x,\varepsilon) L_{q}(x,\varepsilon) \right]\eqsp.
\end{equation}
\end{corollary} 

The proof follows a similar analysis to that of the second term in the "Lipschitz Condition of $\nabla^{P}_{\phi} \mathcal{L}(\theta, \phi)$" in Lemma \ref{lemma:ELBO_smooth_pathwise}.

Assumption (i) is consistent with the assumptions used in prior works, such as \cite[Theorem 1]{domke2024provable} and \cite[Corollary 1]{kim2024convergence}.
Unlike other previous works, our results do not rely on any specific reparameterization trick function. In contrast, prior studies often assume a location-scale parameterization.

\textbf{Deep Gaussian Case: Verifying Assumptions (ii) and (iii). }

To clarify the assumptions (ii) and (iii), we consider the Deep Gaussian case. In this context, we define $\Sigma_{\phi}(x)$ as the diagonal conditioner and focus on analyzing this component, as the mean
$\mu_{\phi}(x)$ is assumed to follow a linear parameterization in the related works.

To verify assumptions (ii) and (iii), we require that $\Sigma_{\phi}(x)$ is both Lipschitz and smooth, and that $\Sigma_{\phi}(x) \geq c$ for some constant $c > 0$.
In \cite{kim2024convergence}, the diagonal conditioner is assumed to be 1-Lipschitz and smooth (as shown in Theorem 1 and Corollary 1). However, they do not impose the lower bound $\Sigma_{\phi}(x) \geq c$.
In contrast, \cite[Theorem 7]{domke2024provable} shows that this condition is verified with the Gaussian case and the location-scale parameterization.

The assumptions in \cite{domke2024provable} ensure that our conditions are satisfied in the Deep Gaussian case. Notably, our assumptions are less restrictive than those in existing works. While our analysis focuses on the diagonal case, the results can be extended to the full-rank case. In general, our results apply to a wider range of reparameterization families than those considered in the existing literature.

\section{SEQUENTIAL VARIATIONAL AUTOENCODERS}

\subsection{Introduction}

Consider an unobserved state  sequence $z_{0:T} = (z_0, \ldots, z_T)$ and an observation sequence $x_{0:T} = (x_0, \ldots, x_T)$. At each time $t \in \mathbb{N}$, the unobserved state $z_t$ and the observation $x_t$ are assumed to take values in some general measurable spaces $(\Zset_t, \mathcal{Z}_t)$ and $(\Xset_t, \Xsigma_t)$, respectively. 
Without any assumption on the sequential latent-variable model, the complete likelihood of the observation sequence $x_{0:T} = (x_0, \ldots, x_T)$ and the latent sequence $z_{0:T} = (z_0, \ldots, z_T)$ is defined as:
$$
p(x_{0:T}, z_{0:T}) = p(z_0)p(x_0|z_0)\prod_{t=1}^{T} p(x_t | x_{0:t-1}, z_{0:t}) p(z_t | z_{0:t-1}, x_{0:t-1})\eqsp.
$$
Then, the posterior distribution of this model can be factorized as:
$$p(z_{0:T} | x_{0:T}) = \prod_{t=0}^{T} p(z_t | z_{0:t-1}, x_{0:T})\eqsp,$$
with the convention $p(z_0 | z_{0:-1}, x_{0:T})=p(z_0 |  x_{0:T})$. 
The ELBO for a sequential VAE is defined as follows:
$$
\mathcal{L}_T = \mathbb{E}_{q(\cdot | x_{0:T})} \left[ \log \frac{p(x_{0:T}, z_{0:T})}{q(z_{0:T} | x_{0:T})} \right] = \mathcal{\ell}_T - \text{KL}(q(\cdot | x_{0:T}) \parallel p(\cdot | x_{0:T}))\eqsp,
$$
where $\ell_T = \log p(x_{0:T})$ corresponds to the log evidence. 
The variational family $q(z_{0:T} | x_{0:T})$ can be factorized using several specific graphical models. The most commonly used variational decompositions are listed in Table \ref{tab:sequential_vae}.

\begin{table}[htbp]
    \centering
    \caption{An Overview of Variational Approximation for Sequential Inference Networks in the Literature.}
    \label{tab:sequential_vae}
    \begin{tabular}{@{}lc@{}}
        \toprule
        Model & Variational Approximation for $z_t$ \\
        \midrule
        q-INDEP & $q(z_t | x_t)$ \\
        q-LR & $q(z_t | x_{t-1:t+1})$ \\
        q-RNN & $q(z_t | x_{0:t})$ \\
        q-BRNN & $q(z_t | x_{0:T})$ \\
        VRNN \citep{chung2015recurrent} & $q(z_t | z_{0:t-1}, x_{0:t})$ \\
        DVBF \citep{karl2016deep} & $q(z_t | z_{t-1}, x_{t})$ \\
        DKF \citep{krishnan2015deep} & $q(z_t | z_{t-1}, x_{0:T})$ \\
        DKS \citep{krishnan2017structured} & $q(z_t | z_{t-1}, x_{t:T})$ \\
        VF (Forward) \citep{marino2018general} & $q(z_t | z_{t-1}, x_{0:t})$ \\
        VF (Backward) \citep{campbell2021online} & $q(z_t | z_{t+1}, x_{0:t})$ \\
        \bottomrule
    \end{tabular}
\end{table}

In the first four models listed, the latent variable $z_t$ depends solely on the observed data, without any dependence on other latent variables $z_0, \dots, z_T$. These models are referred to as conditionally independent latent variable models, reflecting their simple structure. Specifically, each model uses different parameterizations for $q(z_t)$ based on the context of the observed data:
\begin{itemize}
    \item $q\text{-}\text{INDEP}$ where $q(z_t|x_t)$ is parameterized by an MLP,
    \item $q\text{-}\text{LR}$ where $q(z_t|x_{t-1:t+1})$ is parameterized by an MLP,
    \item $q\text{-}\text{RNN}$ where $q(z_t|x_{0:t})$ is parameterized by an RNN,
    \item $q\text{-}\text{BRNN}$ where $q(z_t|x_{0:T})$ is parameterized by a bi-directional RNN.
\end{itemize}

The remaining models use more complex structured variational approximations. For instance, VRNN \citep{chung2015recurrent} conditions $z_t$ on past latent states and observations through a recurrent structure, while other models assume a Markovian structure for the latent variables with varying approaches to handling the observed data. The model of particular focus here is the Variational Backward Model, where the latent variable is conditioned on both future latent states and past observations. This backward conditioning enables better smoothing by incorporating information from future observations.

\subsection{Convergence Results in a General Setting}

First, we compute the gradient of the expected ELBO with respect to $\theta$ and $\phi$ in this sequential setting.
The gradient with respect to $\theta$ is given by:
\begin{equation} \label{eq:grad_seq_theta}
\nabla_{\theta} \mathcal{L}_{T}(\theta, \phi) = \mathbb{E}_{\pi}\left[ \mathbb{E}_{q_{\phi}(\cdot | x_{0:T})}\left[ s_{0:T,\theta} \right] \right],
\end{equation} 
where $s_{0:T,\theta} = \sum_{t=0}^{T} s_{t,\theta}$ with $s_{t,\theta} : \mathsf{X_{0:t}} \times \mathsf{Z_{0:t}} \ni (x_{0:t}, z_{0:t}) \mapsto \nabla_{\theta} \log \{ p_{\theta}(x_t | x_{0:t-1}, z_{0:t}) p_{\theta}(z_t | z_{0:t-1}, x_{0:t-1}) \}$, with the conventions $p_{\theta}(x_0 | x_{0:-1}, z_{0})=p_{\theta}(x_0 |  z_{0})$ and $p_{\theta}(z_0 | z_{0:-1}, x_{0:-1})=p_{\theta}(z_0)$.

Now, for the score function gradient with respect to $\phi$, using Proposition \ref{prop:score_estimator}, we obtain:
\begin{equation} \label{eq:grad_seq_phi}
\nabla_{\phi} \mathcal{L}_{T}(\theta, \phi ) = \mathbb{E}_{\pi}\left[ \mathbb{E}_{q_{\phi}(\cdot | x_{0:T})}\left[ \log \frac{p_{\theta}(x_{0:T}, z_{0:T})}{q_{\phi}(z_{0:T} | x_{0:T})} \nabla_{\phi} \log q_{\phi}(z_{0:T} | x_{0:T}) \right] \right].
\end{equation}

To cover all possible scenarios mentioned previously, we consider the variational family $q_{\phi}(z_{0:T} | x_{0:T})$ which can be factorized as $\prod_{t} q_{\phi}(z_t | \bar{z}_{t}, \bar{x}_{t})$. Here, $\bar{z}_{t}$ can represent $\emptyset$, $z_{0:t-1}$, $z_{t-1}$ or $z_{t+1}$, and $\bar{x}_{t}$ can represent $\emptyset$, $x_{t}$, $x_{t-1:t+1}$, $x_{0:t}$, $x_{t:T}$ or $x_{0:T}$. In the following, we write  $\bar{x}_{t} \in \bar{\Xset}_{t}$ and $\bar{z}_{t} \in \bar{\Zset}_{t}$.

In this sequential framework, we work with the following assumptions.
\begin{assumption}\label{ass:strong_mixing}
(Strong Mixing) \\
For every $t \in \mathbb{N}$, there exist $0 < \sigma^-_t < \sigma^+_t < \infty$ such that for all $\theta \in \Theta$ and $\phi \in \Phi$,
    \begin{enumerate}
    \item[(i)]
    $\sigma^-_t \leq p_{\theta}(x_t | x_{0:t-1}, z_{0:t}) \leq \sigma^+_t $ for every $(x_{0:t}, z_{0:t}) \in \mathsf{X_{0:t}} \times \mathsf{Z_{0:t}}$,
    \item[(ii)] 
    $\sigma^-_t \leq p_{\theta}(z_t | z_{0:t-1}, x_{0:t-1}) \leq \sigma^+_t$ for every $(x_{0:t}, z_{0:t}) \in \mathsf{X_{0:t}} \times \mathsf{Z_{0:t}}$,
    \item[(iii)] $\sigma^-_t \leq q_{\phi}(z_t | \bar{z}_{t}, \bar{x}_{t}) \leq \sigma^+_t$ for every $z_t, \bar{z}_{t}, \bar{x}_{t} \in \mathsf{Z_{t}} \times \bar{\Zset}_{t} \times \bar{\Xset}_{t}$.
\end{enumerate}
\end{assumption}

Assumption \ref{ass:strong_mixing} is quite strong, but it is typically satisfied in models with a compact state space. This assumption is well-established in the Sequential Monte Carlo literature \citep{douc2011sequential, olsson2017efficient, gloaguen2022pseudo, cardoso2023state}, where it is used to obtain quantitative bounds for the errors or variances of estimators. Additionally, it is used to derive variational excess risk bounds for general state space models \citep{chagneux2024additive, gassiat2024variational}.
It is worth noting that in the context of approximating filtering distributions in the SMC literature, weaker assumptions, such as pseudo-mixing, are sometimes sufficient to obtain quantitative bounds on estimator errors or variances (see \cite{chigansky2004stability, douc2009forgetting}). However, extending these results to the smoothing context remains an open challenge. Consequently, obtaining convergence rates for sequential VAE within a general framework, particularly involving a backward kernel, under this weaker assumption is still far from being fully achieved.

\begin{assumption}\label{ass:lipschitz}
(Lipschitz Condition)
\begin{enumerate}
    \item[(i)] For all $t \in \mathbb{N}$, there exists $L^{s}_{t} \in \mathsf{M}(\mathsf{X_{0:t}} \times \mathsf{Z_{0:t}})$ such that for all $(x_{0:t}, z_{0:t}) \in \mathsf{X_{0:t}} \times \mathsf{Z_{0:t}}$, the function $\theta \mapsto s_{t, \theta}(x_{0:t}, z_{0:t})$ is $L^{s}_{t}(x_{0:t}, z_{0:t})$-Lipschitz, and $\theta \mapsto s_{t, \theta}(x_{0:t}, z_{0:t})$ is bounded by $\| s_t(\theta) \|_{\infty}$. Furthermore, $\| L^{s}_{t} \|_{\infty} < \infty$.
    \item[(ii)] For all $t \in \mathbb{N}$, there exists $L^{q}_{t} \in \mathsf{M}(\mathsf{Z_{t}} \times \bar{\Zset}_{t} \times \bar{\Xset}_{t})$ such that $\| L^{q}_{t} \|_{\infty} < \infty$ and that for all $(z_t, \bar{z}_{t}, \bar{x}_{t}) \in \mathsf{Z_{t}} \times \bar{\Zset}_{t} \times \bar{\Xset}_{t}$, $\phi \mapsto \log q_{\phi}(z_t | \bar{z}_{t}, \bar{x}_{t})$ is $L^{q}_{t}(z_t, \bar{z}_{t}, \bar{x}_{t})$-Smooth, and $\phi \mapsto \nabla_{\phi} \log q_{\phi}(z_t | \bar{z}_{t}, \bar{x}_{t})$ is bounded.
\end{enumerate}
\end{assumption}

Assumption \ref{ass:lipschitz}(i) is analogous to the one used in \cite{cardoso2023state} (see Assumption A B.9 (i)), which was employed to establish the convergence rate of their proposed algorithm, the PARIS Particle Gibbs (PPG) sampler. Their method is based on Sequential Monte Carlo (SMC) techniques for the online approximation of posterior distributions in state space models. In contrast, our approach uses variational methods. Consequently, we introduce condition (ii), which is similar to condition (i), but is relevant to the variational distribution.

\begin{theorem} \label{th:con_VAE_seq}
Let Assumptions \ref{ass:strong_mixing} and \ref{ass:lipschitz} hold.
Let $\left(\theta_{n},\phi_{n}\right) \in \Theta \times \Phi$ be the $n$-th iterate of the recursion in Algorithm \ref{alg:adam} where $\gamma_{n} = C_{\gamma}n^{-1/2}$ with $C_{\gamma}>0$. 
For all $n \geq 1$, let $R \in \{0, \ldots, n\}$ be a uniformly distributed random variable. Then,
$$
\mathbb{E}\left[\left\| \nabla_{\theta, \phi} \mathcal{L}\left(\theta_{R}, \phi_{R}\right)\right\|^{2}\right] = \mathcal{O}\left(d^{*} C_T \frac{\log n}{\sqrt{n}}\right),
$$
where $d^{*} = d_{\theta} + d_{\phi}$ represents the total dimension of the parameters, and $C_T$ is a constant that depends on $T$.
\end{theorem}

In this convergence rate, the factor $C_T$ depends on $T$ and increases as $T$ grows. Since the exact structure of the model is unknown, deriving a bound that depends on $T$ is challenging. Nevertheless, we establish the convergence rate, with dependence on $T$, for Deep Gaussian Non-Linear State-Space Models within the context of Variational Smoothing below.

\subsection{Application to Variational Smoothing in Deep Gaussian Non-Linear State-Space Models}

We consider a general form of state-space models (SSMs) where the filtering and smoothing distributions, as well as the log evidence, are typically not available in closed form and thus need to be approximated.

Let $\theta \in \Theta \subset \Rset^{d_\theta}$ be a parameter of interest. In the context of SSMs, it is assumed that the sequence $\{z_t\}_{t \in \mathbb{N}}$ forms a Markov chain with an initial distribution $\nu_\theta$ and transition kernels $\{M_{\theta,t}\}_{t \in \mathbb{N}}$ where each kernel has a transition density $m_\theta$ with respect to some reference measure. Given the states $\{z_t\}_{t \in \mathbb{N}}$, the observations $\{x_t\}_{t \in \mathbb{N}}$ are assumed to be independent and such that for all $t \in \mathbb{N}$, the conditional distribution of the observation $x_t$ depends only on the current state $z_t$. This distribution is assumed to admit a density $g_\theta(z_t, \cdot)$ with respect to some reference measure.
In summary, these correspond to the generative model:
$$z_0 \sim \nu_\theta(z_0)\eqsp, \quad z_{t+1} | z_t \sim m_\theta(z_t, z_{t+1})\eqsp, \quad x_t | z_t \sim g_\theta(z_t, x_t)\eqsp.
$$ 
The complete likelihood of the observation sequence $x_{0:T} = (x_0, \ldots, x_T)$ and the latent sequence $z_{0:T} = (z_0, \ldots, z_T)$ is defined as:
$$p_\theta(x_{0:T}, z_{0:T}) = \nu_\theta(z_0) g_\theta(z_0, x_0) \prod_{t=0}^{T-1} m_\theta(z_{t}, z_{t+1}) g_\theta(z_{t+1}, x_{t+1})\eqsp.$$ 
The posterior distribution of this model can also be written as:
\begin{equation*}
p_{\theta}(z_{0:T} | x_{0:T}) = p_{\theta}(z_T | x_T) \prod_{t=1}^{T-1} p_{\theta}(z_t | x_t, z_{t+1})\eqsp, 
\quad \text{where} \quad 
p_{\theta}(x_t | x_{0:t}, z_{t+1}) = \frac{m_{\theta}(z_{t}, z_{t+1}) p_{\theta}(z_t | x_{0:t})}{p_{\theta}(z_{t+1} | x_{0:t})}\eqsp.
\end{equation*}

Given $x_T$, this shows that the sequence $(z_t)_{t=0}^{T}$ forms a reverse-time Markov chain with the initial distribution $p_{\theta}(z_T | x_T)$ and backward Markov transition kernels $p_{\theta}(z_t | x_{0:t}, z_{t+1})$ \citep{kantas2015particle}.

\paragraph{Backward Decomposition of the Variational Smoothing Distribution. }

We consider a variational smoothing distribution of the form
\begin{equation} \label{eq:backward}
q_{\phi_{0:T}}(z_{0:T}|x_{0:T}) = q_{T, \phi_T}(z_T) \prod_{t=0}^{T-1} q_{t|t+1, \phi_t}(z_{t+1}, z_t)\eqsp,
\end{equation}
where $q_{T, \phi_T}(z_T)$ and $q_{t|t+1, \phi_t}(z_{t+1}, z_t)$ are variational approximations of the filtering density $p_\theta(z_T|x_{0:T})$ and the backward transition density $p_\theta(z_t|x_{0:t}, z_{t+1})$ respectively.

\paragraph{Convergence Results. }

We establish the convergence rate for this specific Sequential VAE structure, where the transition, emission, and backward kernel densities follow Gaussian distributions. In this framework, both the mean and variance are parameterized by neural networks, utilizing the same architecture for the mean and variance as specified in Theorem \ref{thm:con_gauss}, which is commonly used in practice. 

\begin{theorem}\label{thm:con_gauss_seq}
Let $T \in \mathbb{N}$ and for all $0 \leq t \leq T$, consider
\begin{align*}
\mathcal{F}_{m} &= \{ (z,z') \mapsto m_{\theta}(z,z') = \mathcal{N} (z';\mu^{g}_{\theta}(z), \tau_{m}^2 \mathrm{I}_{d_z}) \,|\, \mu^{m}_{\theta}(z) \in \mathcal{F}_{G}, \theta \in \Theta \subseteq \Rset^{d_{\theta}} \}\eqsp,\\
\mathcal{F}_{g} &= \{ (z,x) \mapsto g_{\theta}(z,x) = \mathcal{N} (x;\mu^{g}_{\theta}(z), \tau_{g}^2 \mathrm{I}_{d_x}) \,|\, \mu^{g}_{\theta}(z) \in \mathcal{F}_{G}, \theta \in \Theta \subseteq \Rset^{d_{\theta}} \}\eqsp,\\
\mathcal{F}_{t} &= \{ (z,z') \mapsto q_{\phi_t, t|t+1}(z, z') = \mathcal{N} (z';\mu_{\phi_t}(z), \Sigma_{\phi_t}(z)) \,|\, (\mu_{\phi_t}(z), \Sigma_{\phi_t}(z)) \in \mathcal{F}_{\mu, \Sigma}, \phi_t \in \Phi \subseteq \Rset^{d_{\phi}} \}\eqsp.
\end{align*}
For all $t \in \mathbb{N}$, assume that there exists $C_{\infty} > 0$ such that $\left\|z\right\|_{\infty} \leq C_{\infty}$ for all $z \in \Zset$.
Assume also that the data distribution $\pi$ has a finite fourth moment, and that there exists some constant $a$ such that for all $\theta \in \Theta$ and $\phi \in \Phi$,
$$ 
\lVert \theta \rVert_{\infty} + \lVert \phi \rVert_{\infty} \leq a\eqsp. 
$$
Let $\left(\theta_{n},\phi_{n}\right) \in \Theta \times \Phi$ be the $n$-th iterate of the recursion in Algorithm \ref{alg:adam}, where $\gamma_{n} = C_{\gamma}n^{-1/2}$ with $C_{\gamma}>0$. Assume that $\beta_1 < \sqrt{\beta_2} < 1$.
For all $n \geq 1$, let $R \in \{0, \ldots, n\}$ be a uniformly distributed random variable. Then,
$$
\mathbb{E}\left[\left\| \nabla^{\PW}_{\theta, \phi} \mathcal{L}_{T}\left(\theta_{R}, \phi_{R}\right)\right\|^{2}\right] = \mathcal{O}\left(\frac{\mathcal{L}^*}{\sqrt{n}} + T \frac{N a^{2(N-1)}}{1-\beta_{1}} \frac{d^{*} \log n}{\sqrt{n}}\right)\eqsp,
$$
where $\mathcal{L}^* = \mathcal{L}\left(\theta^{*}, \phi^{*}\right) - \mathcal{L}\left(\theta_{0}, \phi_{0}\right)$, $d^{*} = d_{\theta} + d_{\phi}$ is the total dimension of the parameters, $N = \max\{N_{m}, N_{g}\} + \max_{t}\{N_{t}\}$ the total number of layers in the encoder and decoder.
\end{theorem}

Theorem \ref{thm:con_gauss_seq} provides the convergence rate of $\mathcal{O}\left(\log n / \sqrt{n}\right)$ for Deep Gaussian Non-Linear State-Space Models with variational smoothing, similar to \ref{thm:con_gauss} in the independent case, but with an additional dependence on the time series length $T$. Notably, the convergence rate scales almost linearly with $T$. 
In this context, the gradient is computed over the entire sequence of length $T$.
Alternatively, the time series can be divided into smaller segments, with the gradient computed for each segment to potentially reduce the convergence rate term. However, this introduces bias into the gradient estimator, affecting the overall convergence rate. We leave this exploration for future work.

\subsection{Convergence Proofs for the Sequential VAE}

\subsubsection{Proof of Theorem \ref{th:con_VAE_seq}}

\begin{proof}
The strong mixing assumption (Assumption \ref{ass:strong_mixing}) ensures that Assumption \ref{ass:A1} is satisfied with $\alpha(x,z)$ defined by:
$$\alpha(x,z) = \sum_{t=0}^{T} 2\max\left(\left| \log \sigma^-_t \right|, \left| \log \sigma^+_t \right|\right).$$ 
By using the decomposition of the joint likelihood and the variational distribution, along with the boundedness of $s_{t, \theta}(x_{0:t}, z_{0:t})$ and $\log q_{\phi}(z_t | \bar{z}_{t}, \bar{x}_{t})$, we derive the boundedness of the gradients of $\log p_{\theta}(x|z)$ and $\log q_{\phi}(z|x)$.
Next, we will establish the smoothness of these quantities, addressing each one separately.

\textbf{Lipschitz condition of $\nabla_{\theta} \log p_{\theta}(x_{0:T}|z_{0:T})$:}

$$
\begin{aligned}
\left\| \nabla_{\theta} \log p_{\theta}(x_{0:T}|z_{0:T}) - \nabla_{\theta} \log p_{\theta'}(x_{0:T}|z_{0:T}) \right\| &=  
\left\| s_{0:T,\theta}(x_{0:T}, z_{0:T}) - s_{0:T,\theta'}(x_{0:T}, z_{0:T}) \right\| \\ 
&\leq \sum_{t=0}^{T} \left\| s_{t,\theta}(x_{0:t}, z_{0:t}) - s_{t,\theta'}(x_{0:t}, z_{0:t}) \right\| \\ 
&\leq \sum_{t=0}^{T} L^{s}_{t}(x_{0:t}, z_{0:t}) \left\| \theta - \theta' \right\|.
\end{aligned}
$$

\textbf{Lipschitz condition of $\nabla_{\phi} \log q_{\phi}(z_{0:T}|x_{0:T})$:}

$$
\begin{aligned}
\left\|\nabla_{\phi} \log q_{\phi}(z_{0:T}|x_{0:T}) - \nabla_{\phi} \log q_{\phi'}(z_{0:T}|x_{0:T})\right\| &= \left\| \sum_{t=0}^{T} \nabla_{\phi} \log q_{\phi}(z_t | \bar{z}_{t}, \bar{x}_{t}) - \nabla_{\phi} \log q_{\phi'}(z_t | \bar{z}_{t}, \bar{x}_{t})\right\|\\
&\leq \sum_{t=0}^{T} \left\| \nabla_{\phi} \log q_{\phi}(z_t | \bar{z}_{t}, \bar{x}_{t}) - \nabla_{\phi} \log q_{\phi'}(z_t | \bar{z}_{t}, \bar{x}_{t})\right\|\\
&\leq \sum_{t=0}^{T} L^{q}_{t}(z_t, \bar{z}_{t}, \bar{x}_{t}) \left\| \phi - \phi'\right\|,
\end{aligned}
$$
which provides the smoothness condition, and thus Assumption \ref{ass:A2} is satisfied with $L_1(x,z) = \sum_{t=0}^{T} L^{s}_{t}(x_{0:t}, z_{0:t})$ and $L_2(x,z) = \sum_{t=0}^{T} L^{q}_{t}(z_t, \bar{z}_{t}, \bar{x}_{t})$ for $x=x_{0:T}$ and $z=z_{0:T}$.
We conclude the proof by applying Theorem \ref{th:rate_general_Adam}.
\end{proof}

\subsubsection{Proof of Theorem \ref{thm:con_gauss_seq}}

\begin{proof}
For all $T \in \mathbb{N}$, writing $\nu_\theta(z_0) = m_{\theta}(z_{-1}, z_0)$ and $q_{T, \phi_T}(z_T) = q_{T|T+1, \phi_T}(z_{T+1}, z_T)$, the gradient of the ELBO with respect to $\theta$ is given by:
\begin{align*}
\nabla_{\theta} \mathcal{L}_{T}(\theta, \phi; \mathbf{x}_{0:T}) &= \sum_{t=0}^{T} \mathbb{E}_{q_{\phi}(\cdot|x_{0:T})} \left[ \nabla_{\theta} \left( \log m_{\theta}(z_{t-1}, z_t) + \log g_{\theta}(z_t, x_t) \right) \right] \\
&= \sum_{t=0}^{T} \mathbb{E}_{q_{\phi}(\cdot|x_{0:T})} \left[ \nabla_{\theta} \left( -\frac{1}{2\tau_{m}^2}\|z_t - \mu^{m}_{\theta}(z_{t-1})\|^2  - \frac{1}{2\tau_{g}^2}\|x_t - \mu^{g}_{\theta}(z_t)\|^2 \right) \right] \\
&= \sum_{t=0}^{T} \mathbb{E}_{q_{\phi}(\cdot|x_{0:T})} \left[ \frac{1}{\tau_{m}^2}\nabla_{\theta} \mu^{m}_{\theta}(z_{t-1})^\top(z_t - \mu^{m}_{\theta}(z_{t-1})) + \frac{1}{\tau_{g}^2} \nabla_{\theta} \mu^{g}_{\theta}(z_t)^\top (x_t - \mu^{g}_{\theta}(z_t)) \right] \eqsp.
\end{align*}

For simplicity, we consider the case where the Lipschitz constant $M_{f_i}$ and the smoothness constant $L_{f_i}$ are both less than 1. This can be easily adapted for any Lipschitz and smoothness constants, similar to the case of independent data.

\paragraph{Lipschitz condition of $\nabla_{\theta} \mathcal{L}_{T}(\theta, \phi)$. }
First, as in the independent case, we have:
\begin{equation} \label{eq:lip_seq}
\left\| \nabla_{\theta} \mathcal{L}(\theta, \phi) - \nabla_{\theta} \mathcal{L}(\theta', \phi') \right\| \leq \left\| \nabla_{\theta} \mathcal{L}(\theta, \phi) - \nabla_{\theta} \mathcal{L}(\theta', \phi) \right\| + \left\| \nabla_{\theta} \mathcal{L}(\theta', \phi) - \nabla_{\theta} \mathcal{L}(\theta', \phi') \right\|\eqsp.   
\end{equation}
Now, we bound each of these terms individually.
For all $\theta,\theta'\in\Theta$, $\phi\in\Phi$,
\begin{align*}
&\left\| \nabla_{\theta} \mathcal{L}_{T}(\theta, \phi; \mathbf{x}_{0:T}) - \nabla_{\theta} \mathcal{L}_{T}(\theta', \phi; \mathbf{x}_{0:T}) \right\| \\
&\leq \frac{1}{\tau_{m}^2} \sum_{t=0}^{T} \mathbb{E}_{q_{\phi}(\cdot|x_{0:T})} \left[ \left\| \nabla_{\theta} \mu^{m}_{\theta}(z_{t-1})^\top (z_t - \mu^{m}_{\theta}(z_{t-1})) - \nabla_{\theta} \mu^{m}_{\theta'}(z_{t-1})^\top (z_t - \mu^{m}_{\theta'}(z_{t-1})) \right\| \right] \\
&\quad + \frac{1}{\tau_{g}^2} \sum_{t=0}^{T} \mathbb{E}_{q_{\phi}(\cdot|x_{0:T})} \left[ \left\| \nabla_{\theta} \mu^{g}_{\theta}(z_t)^\top (x_t - \mu^{g}_{\theta}(z_t)) - \nabla_{\theta} \mu^{g}_{\theta'}(z_t)^\top (x_t - \mu^{g}_{\theta'}(z_t)) \right\| \right] \\
&\leq \frac{1}{\tau_{m}^2} \sum_{t=0}^{T} \mathbb{E}_{q_{\phi}(\cdot|x_{0:T})} \left[ \|z_t - \mu^{m}_{\theta}(z_{t-1})\| \left\| \nabla_{\theta} \mu^{m}_{\theta}(z_{t-1}) - \nabla_{\theta} \mu^{m}_{\theta'}(z_{t-1}) \right\| + \left\| \nabla_{\theta} \mu^{m}_{\theta'}(z_{t-1}) \right\| \left\| \mu^{m}_{\theta}(z_{t-1}) - \mu^{m}_{\theta'}(z_{t-1}) \right\| \right] \\
&\quad + \frac{1}{\tau_{g}^2} \sum_{t=0}^{T} \mathbb{E}_{q_{\phi}(\cdot|x_{0:T})} \left[ \|x_t - \mu^{g}_{\theta}(z_t)\| \left\| \nabla_{\theta} \mu^{g}_{\theta}(z_t) - \nabla_{\theta} \mu^{g}_{\theta'}(z_t) \right\| + \left\| \nabla_{\theta} \mu^{g}_{\theta'}(z_t) \right\| \left\| \mu^{g}_{\theta}(z_t) - \mu^{g}_{\theta'}(z_t) \right\| \right] \\
&\leq L^{dd,1}_{T} \left\| \theta - \theta' \right\|\eqsp,
\end{align*}
where $L^{dd,1}_{T} = (T+1) (a^{2(N_m - 1)}/\tau_{m}^2 + a^{2(N_g - 1)}/\tau_{g}^2) + (T+1) N_m a^{2(N_m - 1)} C_{\infty}^2 \left( C_{G} + C_{\infty}\right)/\tau_{m}^2 + (T+1) N_g a^{2(N_g - 1)} C_{\infty}^2 \left( C_{G} + \left\|x_{0:T}\right\|\right)/\tau_{g}^2$ using the boundedness of the latent state space and the Lipschitz continuity and smoothness of $\theta \mapsto \mu^{m}_{\theta}$ and $\theta \mapsto \mu^{g}_{\theta}$ (Lemma \ref{lem:smooth_NN}). This concludes the proof for the first term in \eqref{eq:lip_seq}.

For the second term in \eqref{eq:lip_seq}, we have:
\begin{align*}
\left\| \nabla_{\theta} \mathcal{L}_{T}(\theta, \phi; \mathbf{x}_{0:T}) - \nabla_{\theta} \mathcal{L}_{T}(\theta, \phi'; \mathbf{x}_{0:T}) \right\| \leq  \sum_{t=0}^{T} A^{1}_t + \sum_{t=0}^{T}A^{2}_t\eqsp,
\end{align*}
where,
\begin{align*}
A^{1}_t &= \frac{1}{\tau_{m}^2} \left\|\int \left( \nabla_{\theta} \mu^{m}_{\theta}(z_{t-1})^\top(z_t - \mu^{m}_{\theta}(z_{t-1})) \right) \left( q_{\phi}(z_{0:T}|x_{0:T}) - q_{\phi'}(z_{0:T}|x_{0:T}) \right) \, \rmd z_{0:T} \right\|\eqsp, \\
A^{2}_t &= \frac{1}{\tau_{g}^2} \left\|\int \left( \nabla_{\theta} \mu^{g}_{\theta}(z_t)^\top (x_t - \mu^{g}_{\theta}(z_t)) \right) \left( q_{\phi}(z_{0:T}|x_{0:T}) - q_{\phi'}(z_{0:T}|x_{0:T}) \right) \, \rmd z_{0:T} \right\|\eqsp.
\end{align*}
Next, we define the constants from Lemma \ref{lem:ELBO_A1_gauss} that satisfy the strong mixing condition (Assumption \ref{ass:strong_mixing}) for the Gaussian density:
$$
\sigma^- = \frac{1}{(2\pi C_{\Sigma})^{d_z/2}} \exp\left(-\frac{1}{c_{\Sigma}} \left(C_{\infty}^2 + C_{\mu}^2 \right) \right)\eqsp, \quad\mathrm{and}\quad
\sigma^+ = \frac{1}{(2\pi c_{\Sigma})^{d_z/2}} \exp\left(\frac{C_{\mu}^2}{4C_{\Sigma}} \right)\eqsp.
$$
For all $t \in \mathbb{N}$, let $f_{\theta}(z_{t-1}, z_t) = \nabla_{\theta} \mu^{m}_{\theta}(z_{t-1})^\top(z_t - \mu^{m}_{\theta}(z_{t-1}))/\tau_{m}^2$ for $A^{1}_t$ and $f_{\theta}(z_{t-1}, z_t) = \nabla_{\theta} \mu^{g}_{\theta}(z_t)^\top (x_t - \mu^{g}_{\theta}(z_t))/\tau_{g}^2$ for $A^{2}_t$. In addition, for all $u\leq v$, write $\bar q_{u:v+1,\phi} = \prod_{s=u}^{v} q_{s|s+1, \phi_s}$. Then,
\begin{align*}
&\left\| \int f_{\theta}(z_{t-1}, z_t) \left( q_{\phi}(z_{0:T}|x_{0:T}) - q_{\phi'}(z_{0:T}|x_{0:T}) \right) \, \rmd z_{0:T} \right\|  \\
&= \left\| \int f_{\theta}(z_{t-1}, z_t) \left(\prod_{s=t-1}^{T} q_{s|s+1, \phi_s}(z_{s+1}, z_s) - \prod_{s=t-1}^{T} q_{s|s+1, \phi'_s}(z_{s+1}, z_s) \right) \, \rmd z_{0:T} \right\|\\
&= \left\| \int f_{\theta}(z_{t-1}, z_t) \sum_{s=t-1}^{T} \left( \bar q_{s:T+1,\phi}(z_{s:T}) \bar q_{t-1:s,\phi'}(z_{t-1:s})  - \bar q_{s+1:T+1,\phi}(z_{s+1:T}) \bar q_{t-1:s+1,\phi'}(z_{t-1:s+1}) \right) \, \rmd z_{0:T} \right\| \\
&\leq \sum_{s=t-1}^{T} \left\| \int f_{\theta}(z_{t-1}, z_t) \left( \bar q_{s:T+1,\phi}(z_{s:T}) \bar q_{t-1:s,\phi'}(z_{t-1:s}) - \bar q_{s+1:T+1,\phi}(z_{s+1:T}) \bar q_{t-1:s+1,\phi'}(z_{t-1:s+1}) \right) \, \rmd z_{0:T} \right\| \\
&\leq \left\| f_\theta \right\|_{\infty} \sum_{s=t-1}^{T} \rho^{s-t+1} \left\| \mu_{s:T,\phi} - \widetilde\mu_{s:T,\phi,\phi'} \right\|_{\text{TV}}\eqsp,
\end{align*}
where for all measurable set $A$, $\mu_{s:T,\phi}(A) = \int\prod_{\ell=s+1}^{T} q_{\ell|\ell+1, \phi_\ell}(z_{\ell+1},z_\ell) q_{s|s+1, \phi_s}(z_{s+1},z_s)\mathbf{1}_A(z_s)\rmd z_{s:T}$, $\widetilde \mu_{s:T,\phi,\phi'}(A) = \int\prod_{\ell=s+1}^{T} q_{\ell|\ell+1, \phi_\ell}(z_{\ell+1},z_\ell) q_{s|s+1, \phi'_s}(z_{s+1},z_s)\mathbf{1}_A(z_s)\rmd z_{s:T}$ and  
where we used \cite[Theorem 4.10]{gloaguen2022pseudo} with $\rho = 1 - \sigma^-/\sigma^+$.
Then, similar as in the independent case, using the inequality for all $x\geq 1$, $x - 1 \leq x \log x \leq \left|x \log x\right|$, we have:
\begin{align*}
&\left\| \mu_{s:T,\phi} - \widetilde\mu_{s:T,\phi,\phi'} \right\|_{\text{TV}} \\
&= \frac{1}{2} \int \left| \bar q_{s+1:T+1,\phi}(z_{s+1:T}) \left( q_{s|s+1, \phi_s}(z_{s+1}, z_s) - q_{s|s+1, \phi'_s}(z_{s+1}, z_s) \right) \right| \, \rmd z_{s:T} \\
&\leq \frac{1}{2} \int \bar q_{s+1:T+1,\phi}(z_{s+1:T}) \left| q_{s|s+1, \phi_s}(z_{s+1}, z_s) - q_{s|s+1, \phi'_s}(z_{s+1}, z_s) \right| \, \rmd z_{s:T} \\
&\leq \frac{1}{2} \int \bar q_{s+1:T+1,\phi}(z_{s+1:T}) \left( \frac{q_{s|s+1, \phi_s}(z_{s+1}, z_s)}{q_{s|s+1, \phi'_s}(z_{s+1}, z_s)} - 1 \right) 1_{q_{s|s+1, \phi_s}(z_{s+1}, z_s) \geq q_{s|s+1, \phi'_s}(z_{s+1}, z_s)} q_{s|s+1, \phi'_s}(z_{s+1}, z_s) \, \rmd z_{s:T} \\
&\quad + \frac{1}{2} \int \bar q_{s+1:T+1,\phi}(z_{s+1:T}) \left( \frac{q_{s|s+1, \phi_s}(z_{s+1}, z_s)}{q_{s|s+1, \phi'_s}(z_{s+1}, z_s)} - 1 \right) 1_{q_{s|s+1, \phi_s}(z_{s+1}, z_s) > q_{s|s+1, \phi'_s}(z_{s+1}, z_s)} q_{s|s+1, \phi'_s}(z_{s+1}, z_s) \, \rmd z_{s:T} \\
&\leq \frac{1}{2} \int \bar q_{s+1:T+1,\phi}(z_{s+1:T}) \left|\frac{q_{s|s+1, \phi_s}(z_{s+1}, z_s)}{q_{s|s+1, \phi'_s}(z_{s+1}, z_s)} \log \frac{q_{s|s+1, \phi_s}(z_{s+1}, z_s)}{q_{s|s+1, \phi'_s}(z_{s+1}, z_s)}\right| q_{s|s+1, \phi'_s}(z_{s+1}, z_s) \, \rmd z_{s:T} \\
&\quad + \frac{1}{2} \int \bar q_{s+1:T+1,\phi}(z_{s+1:T}) \left|\frac{q_{s|s+1, \phi'_s}(z_{s+1}, z_s)}{q_{s|s+1, \phi_s}(z_{s+1}, z_s)} \log \frac{q_{s|s+1, \phi'_s}(z_{s+1}, z_s)}{q_{s|s+1, \phi_s}(z_{s+1}, z_s)}\right| q_{s|s+1, \phi_s}(z_{s+1}, z_s) \, \rmd z_{s:T} \\
&\leq \frac{1}{2} \int \bar q_{s+1:T+1,\phi}(z_{s+1:T}) \left|\log \frac{q_{s|s+1, \phi_s}(z_{s+1}, z_s)}{q_{s|s+1, \phi'_s}(z_{s+1}, z_s)}\right| q_{s|s+1, \phi_s}(z_{s+1}, z_s) \, \rmd z_{s:T} \\
&\quad + \frac{1}{2} \int \bar q_{s+1:T+1,\phi}(z_{s+1:T}) \left|\log \frac{q_{s|s+1, \phi'_s}(z_{s+1}, z_s)}{q_{s|s+1, \phi_s}(z_{s+1}, z_s)}\right| q_{s|s+1, \phi'_s}(z_{s+1}, z_s) \, \rmd z_{s:T} \\
&\leq C_{\infty} a^{\max\{N_t\} - 1} \left\| \phi_s - \phi'_s \right\| \eqsp,
\end{align*}
where we used the Lipschitz condition of $\phi_s \mapsto \log q_{s|s+1, \phi_s}$.
We then derive the following bounds:
\begin{align*}
A^{1}_t &\leq \frac{1}{(1-\rho) \tau_{m}^2} C_{\infty}^2 \left( C_{G} + C_{\infty}\right) a^{N_m - 1} a^{\max\{N_t\} - 1} \left\| \phi - \phi' \right\|\eqsp, \\
A^{2}_t &\leq \frac{1}{(1-\rho) \tau_{g}^2} C_{\infty}^2 \left( C_{G} + \left\|x_{0:T}\right\|\right) a^{N_g - 1} a^{\max\{N_t\} - 1} \left\| \phi - \phi' \right\|\eqsp.
\end{align*}
Therefore, 
\begin{multline*}
\left\| \nabla_{\theta} \mathcal{L}_{T}(\theta, \phi; \mathbf{x}_{0:T}) - \nabla_{\theta} \mathcal{L}_{T}(\theta, \phi'; \mathbf{x}_{0:T}) \right\|\\ \leq \left(T+1\right) C_{\infty}^2 a^{\max\{N_t\} - 1} \left( \frac{C_{G} + C_{\infty}}{(1-\rho) \tau_{m}^2} a^{N_m - 1} + \frac{C_{G} + \left\|x_{0:T}\right\|}{(1-\rho) \tau_{g}^2} a^{N_g - 1} \right) \left\| \phi - \phi' \right\|\eqsp,
\end{multline*}
which concludes the proof of the Lipschitz condition for $\nabla_{\theta} \mathcal{L}_{T}(\theta, \phi)$ by taking the expectation over $\mathbf{x}_{0:T}$.

\paragraph{Lipschitz condition of $\nabla_{\phi_t} \mathcal{L}_{T}(\theta, \phi)$. }

Given an observation sequence $x_{0:T}$, the ELBO for a sequential VAE in this setting, along with the score function gradient of the ELBO with respect to $\phi$ is defined as:
\begin{align*}
\mathcal{L}_{T}(\theta, \phi; \mathbf{x}_{0:T}) &= \sum_{t=0}^{T} \mathbb{E}_{q_{\phi}(\cdot|x_{0:T})} \left[ \log m_{\theta}(z_t, z_{t+1}) + \log g_{\theta}(z_t, x_t) - \log q_{t|t+1, \phi_t}(z_{t+1}, z_t)\right]\eqsp, \\
\nabla_{\phi_t} \mathcal{L}_{T}(\theta, \phi; \mathbf{x}_{0:T}) &= \mathbb{E}_{q_{\phi}(\cdot|x_{0:T})} \left[ \log \frac{m_{\theta}(z_t, z_{t+1}) g_{\theta}(z_t, x_t)}{q_{t|t+1, \phi_t}(z_{t+1}, z_t)} \nabla_{\phi_t} \log q_{t|t+1, \phi_t}(z_{t+1}, z_t) \right] \eqsp.
\end{align*}

This gradient shares structural similarities with the independent case but differs due to the appearance of both the transition density $m$ and the emission density $g$. Following a similar procedure to that in Lemma \ref{lemma:ELBO_smooth_score}, we obtain:
\begin{align*}
&\left\| \nabla_{\phi_t} \mathcal{L}_{T}(\theta, \phi_t; \mathbf{x}_{0:T}) - \nabla_{\phi_t} \mathcal{L}_{T}(\theta', \phi_t; \mathbf{x}_{0:T}) \right\| \\
&\leq \mathbb{E}_{q_{\phi}(\cdot|x_{0:T})} \left[ \left\| \nabla_{\phi_t} \log q_{t|t+1, \phi_t}(z_{t+1}, z_t) \right\| \left\| \log m_{\theta}(z_t, z_{t+1}) - \log m_{\theta'}(z_t, z_{t+1}) \right\| \right] \\
&\quad + \mathbb{E}_{q_{\phi}(\cdot|x_{0:T})} \left[ \left\| \nabla_{\phi_t} \log q_{t|t+1, \phi_t}(z_{t+1}, z_t) \right\| \left\| \log g_{\theta}(z_t, x_t) - \log g_{\theta'}(z_t, x_t) \right\| \right] \\
&\leq C_{\infty}^2 a^{N_t-1} \left(\left(C_G + C_{\infty}\right) a^{N_m-1} + \left(C_G + \left\|x_{0:T}\right\|\right) a^{N_g-1}\right)\eqsp.
\end{align*}

Now, let the unnormalized weights be denoted as follows: 
$$w_{\theta, \phi_t}(x_t, z_t, z_{t+1}) = \frac{m_{\theta}(z_t, z_{t+1}) g_{\theta}(z_t, x_t)}{q_{t|t+1, \phi_t}(z_{t+1}, z_t)}\eqsp.$$
Following the same approach as in the Lipschitz condition for $\nabla^{\SF}_{\phi} \mathcal{L}(\theta, \phi)$ in the independent case (Lemma \ref{lemma:ELBO_smooth_score}), we have:
\begin{align*}
\left\| \nabla_{\phi_t} \mathcal{L}_{T}(\theta, \phi_t; \mathbf{x}_{0:T}) - \nabla_{\phi_t} \mathcal{L}_{T}(\theta, \phi'_t; \mathbf{x}_{0:T}) \right\| &\leq A^{1}_t + A^{2}_t\eqsp,
\end{align*}
where
\begin{align*}
    A^{1}_t &= \left\| \mathbb{E}_{\phi}\left[ \log w_{\theta, \phi_t}(x_t, z_t, z_{t+1}) \nabla_{\phi_t} \log q_{t|t+1, \phi_t}(z_{t+1}, z_t) \right] - \mathbb{E}_{\phi}\left[ \log w_{\theta, \phi'_t}(x_t, z_t, z_{t+1}) \nabla_{\phi_t} \log q_{t|t+1, \phi'_t}(z_{t+1}, z_t) \right] \right\|,\\
    A^{2}_t &= \left\| \mathbb{E}_{\phi_t}\left[ \log w_{\theta, \phi_t}(x_t, z_t, z_{t+1}) \nabla_{\phi_t} \log q_{t|t+1, \phi_t}(z_{t+1}, z_t) \right] - \mathbb{E}_{\phi'_t}\left[ \log w_{\theta, \phi_t}(x_t, z_t, z_{t+1}) \nabla_{\phi_t} \log q_{t|t+1, \phi_t}(z_{t+1}, z_t) \right] \right\|.
\end{align*}

Furthermore, denoting $x=x_{0:T}$ and $z=z_{0:T}$, we also obtain the following bounds:
\begin{align*}
A^{1}_t &\leq 2\mathbb{E}_{\pi, \phi}\left[ \alpha(x,z) L_2(x,z) \right] \left\|\phi - \phi'\right\| + \mathbb{E}_{\pi, \phi}\left[ M_{q}(x,z)^2 \right] \left\|\phi - \phi'\right\|\eqsp, \\
A^{2}_t &\leq 2 \left( \mathbb{E}_{\pi, \phi} \left[ \alpha(x,z) M_{q}(x,z)^2 \right] + \mathbb{E}_{\pi, \phi'} \left[ \alpha(x,z) M_{q}(x,z)^2 \right] \right) \left\|\phi - \phi'\right\|\eqsp, 
\end{align*}
where
\begin{align*}
\alpha(x,z) &= \max \left\{\frac{d_z}{2} \log(2\pi C_{\Sigma}) + \frac{1}{c_{\Sigma}} \left(C_{\infty}^2 + C_{\mu}^2 \right), \frac{d_z}{2} \log(2\pi c^2) + \frac{1}{c^2} \left(\|x_{0:T}\|^2 + C_{G}^2 \right)
\right\}\eqsp, \\
L_2(x,z) &= \frac{N_{t}}{c_{\Sigma}} \left\|x_{0:T}\right\|^2 a^{2(N_{t}-1)} \left( 4C_{\infty}^2 + 4C_{\mu}^2 + 4C_{\infty} + 4C_{\mu} + c_{\Sigma} + 1\right)\eqsp, \\
M_{q}(x,z) &= C_{\infty} a^{N_t-1}\eqsp.
\end{align*}
We establish the Lipschitz condition for $\nabla_{\phi} \mathcal{L}_{T}(\theta, \phi)$ by combining all the derived inequalities and then taking the expectation over $x_{0:T}$.
The proof is then completed by applying Theorem \ref{th:rate_general_Adam}.
\end{proof}

\section{ADDITIONAL EXPERIMENTS}
\label{supp:sec:exp}

\subsection{Additional Experiments details on CelebA}
\label{supp:sec:add_exp_celeba}

In this section, we provide further details regarding the experiments conducted on the CelebA dataset. Specifically, we examine how architectural modifications impact the convergence rate. Figure \ref{layers_CelebA} illustrates the effect of adding layers to the model on the squared norm of the gradients,
Figure \ref{layers_CelebA} illustrates the impact of adding layers to the model on the squared norm of the gradients $\| \nabla \mathcal{L}(\theta_n, \phi_n) \|^{2}$. 
The figure compares the baseline model, which has 22,607,435 parameters, with two variants: one that includes an additional fully connected layer (24,706,635 parameters) and another that adds a convolutional layer (30,993,483 parameters). The results indicate that adding an extra layer generally slows down convergence. Notably, the difference in convergence rates between the fully connected layer and the convolutional layer is relatively small, primarily due to the variation in the number of parameters introduced by each layer type.

\begin{figure}[ht]
\begin{center}
\centerline{
    \includegraphics[width=0.5\textwidth]
    {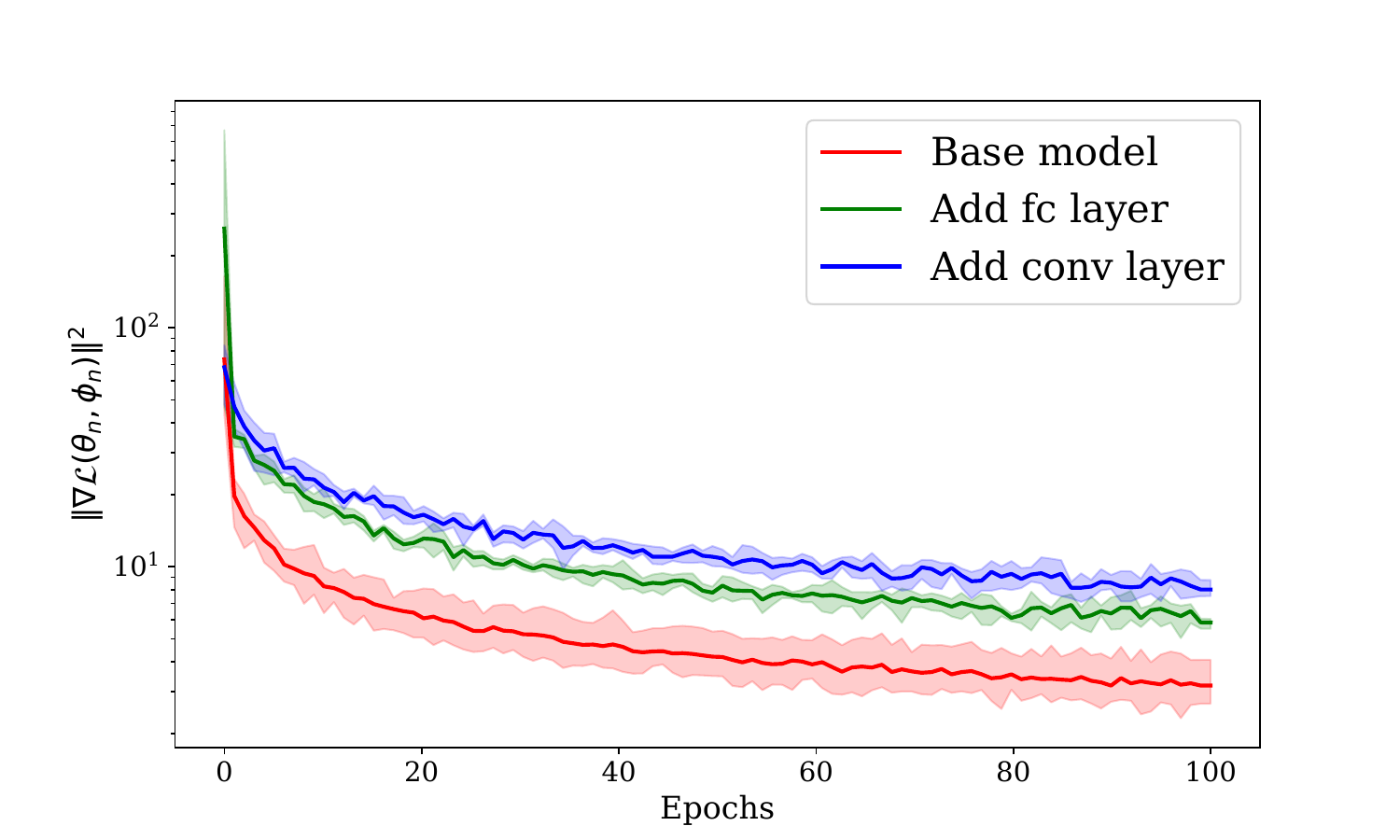}
    }
\caption{$\| \nabla \mathcal{L}(\theta_n, \phi_n) \|^{2}$ in VAE trained with Adam for the baseline model, a model with an additional fully connected layer, and a model with an additional convolutional layer. Bold lines represent the mean over 5 independent runs. Figures are plotted on a logarithmic scale for better visualization.}
\label{layers_CelebA}
\end{center}
\vskip -0.2in
\end{figure}

\subsection{Experiments on CIFAR-100}
\label{supp:sec:exp_cifar100}

\textbf{Dataset and Model. } We conduct our experiments on the CIFAR-100 dataset \citep{krizhevsky2009learning} and use a Convolutional Neural Network (CNN) architecture with ReLU and generalized soft-clipping activation functions for both the encoder and decoder networks. All other model, optimizer, and training parameters are consistent with those used for the CelebA dataset.

In the first experiment, we illustrate the convergence results of the standard VAE using our choice of activation functions, similar to those applied in the CelebA dataset.
Figure \ref{all_CIFAR100} shows the squared norm of the gradients $\| \nabla \mathcal{L}(\theta_n, \phi_n) \|^{2}$ and the negative log-likelihood on the test dataset for both ReLU and the generalized soft-clipping activation function with various values of $s$.
We observe a comparable convergence rate for all values of $s$.
Notably, $s=5$ emerges as a reasonable choice, balancing convergence rate and effective gradient flow, consistent with observations from the CelebA case. However, $s=10$ also performs adequately, contrary to the results from the CelebA dataset.

Next, we estimate the squared norm of the gradients and the negative log-likelihood using the $\beta$-VAE and IWAE models.
Figure \ref{beta_CIFAR100} shows both the squared gradient norm and the negative log-likelihood for the $\beta$-VAE across different values of $\beta$. Additionally, Figure \ref{iwae_CIFAR100} displays the same quantities for the IWAE, evaluated with different values of $K$.
As with the CelebA dataset, we observe that smaller values of $\beta$ for the $\beta$-VAE, lead to faster convergence in both cases. Similarly, increasing the value of $K$ for the IWAE results in faster convergence.

\begin{figure*}[ht!]
\begin{center}
\centerline{
    \includegraphics[width=0.5\textwidth]
    {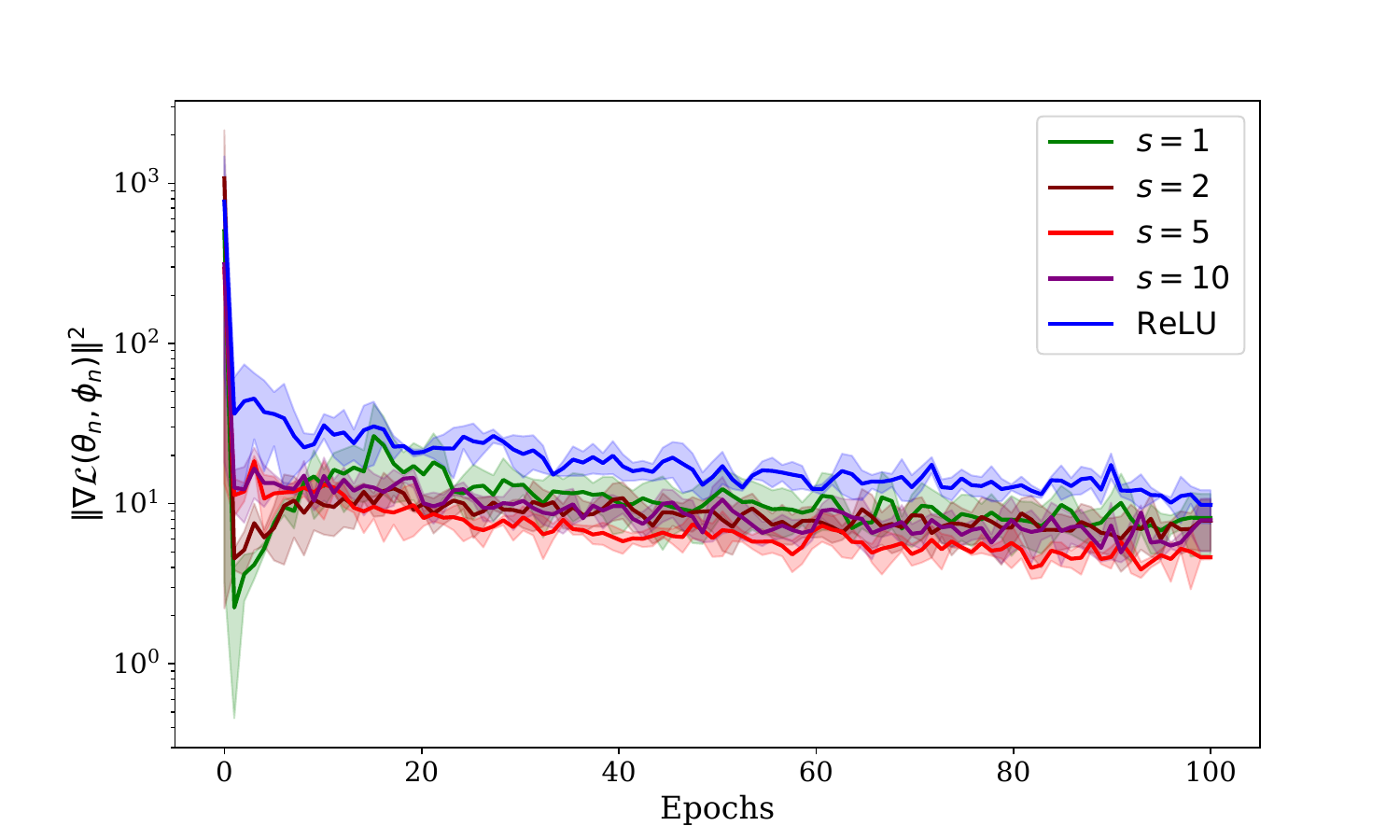}
    \includegraphics[width=0.5\textwidth]{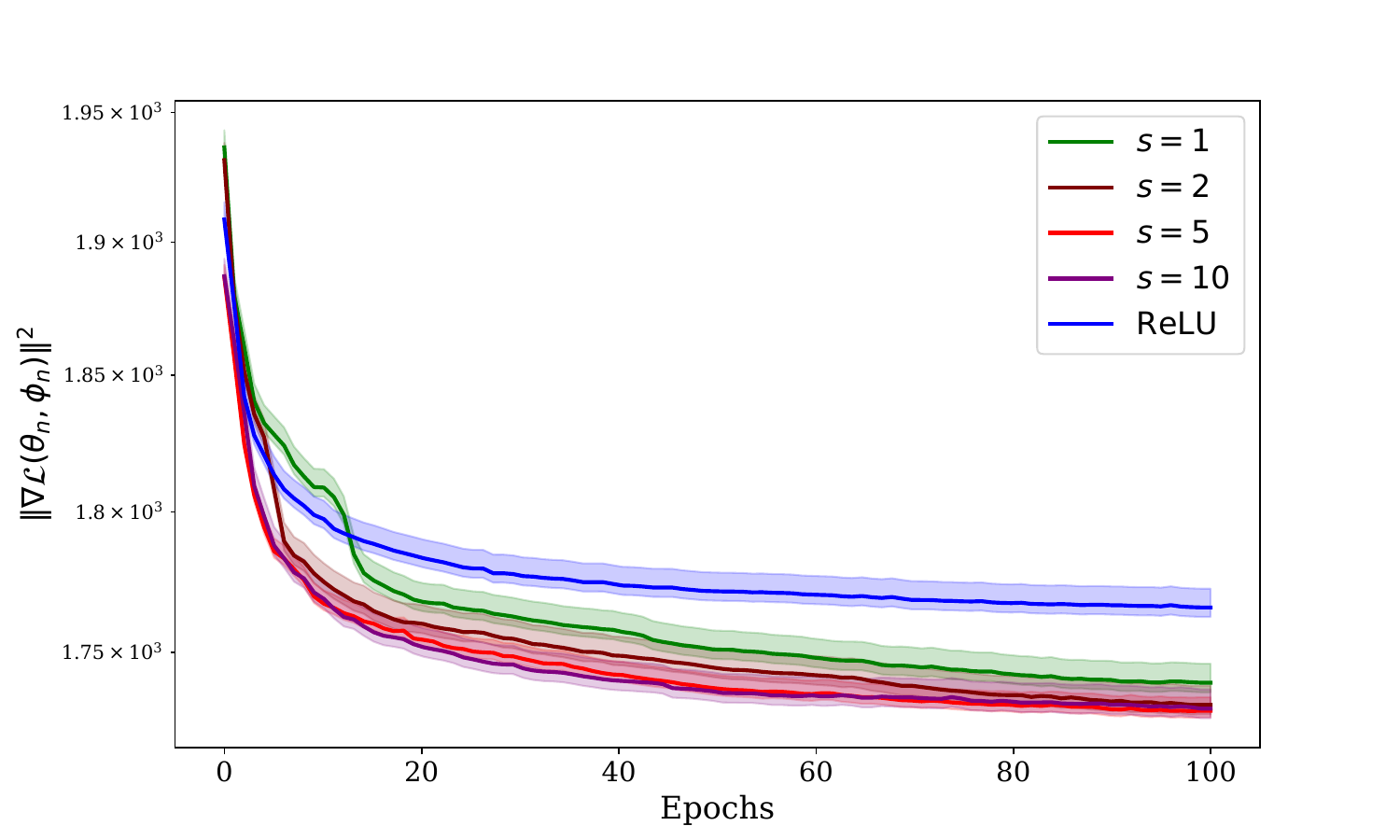}
    }
\caption{Squared norm of gradients and Negative ELBO on the test set of the CelebA for VAE trained with Adam and generalized soft-clipping activation function. Bold lines represent the mean over 5 independent runs.}
\label{all_CIFAR100}
\end{center}
\vskip -0.2in
\end{figure*}

\begin{figure*}[ht!]
\vskip -0.2in
\begin{center}
\centerline{
    \includegraphics[width=0.5\textwidth]
    {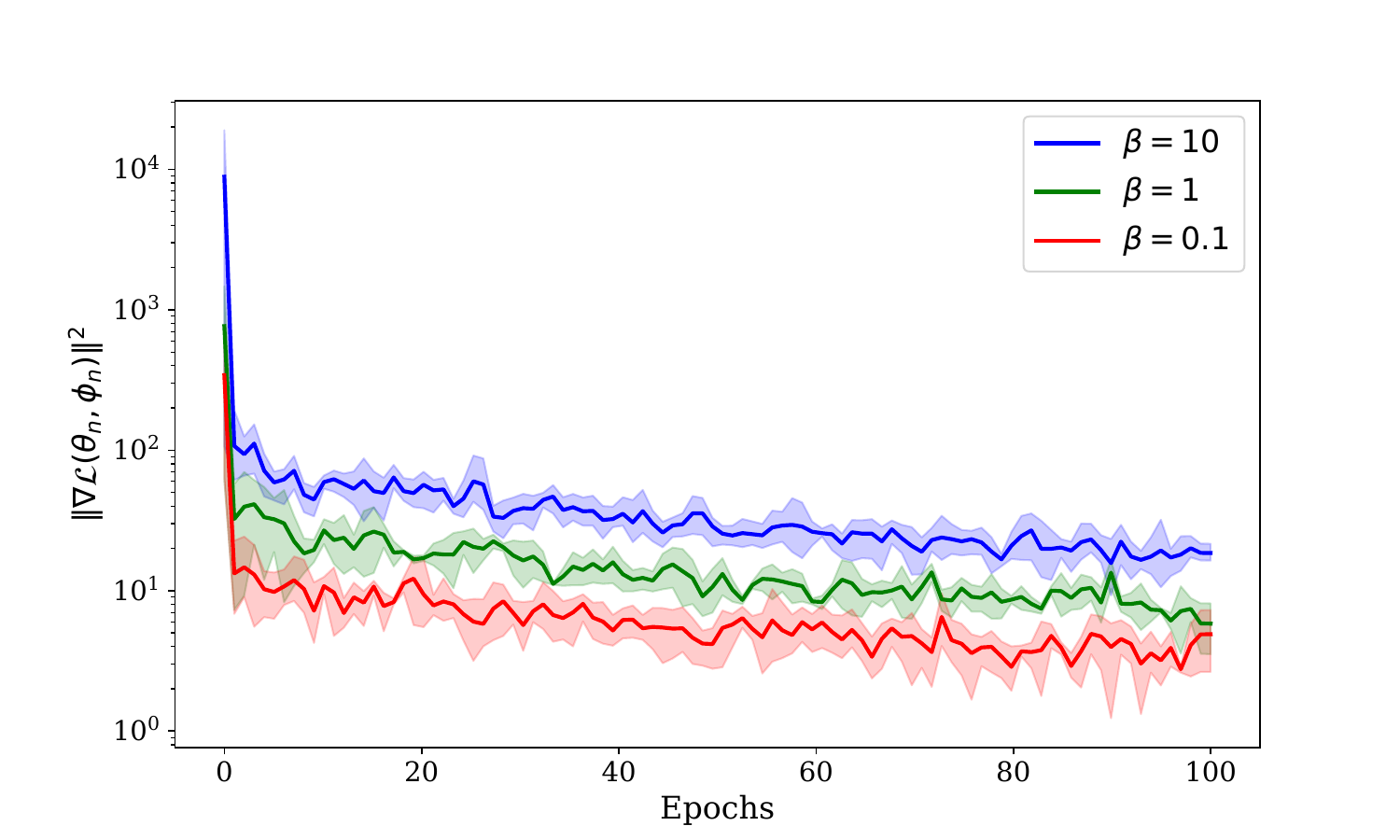}
    \includegraphics[width=0.5\textwidth]{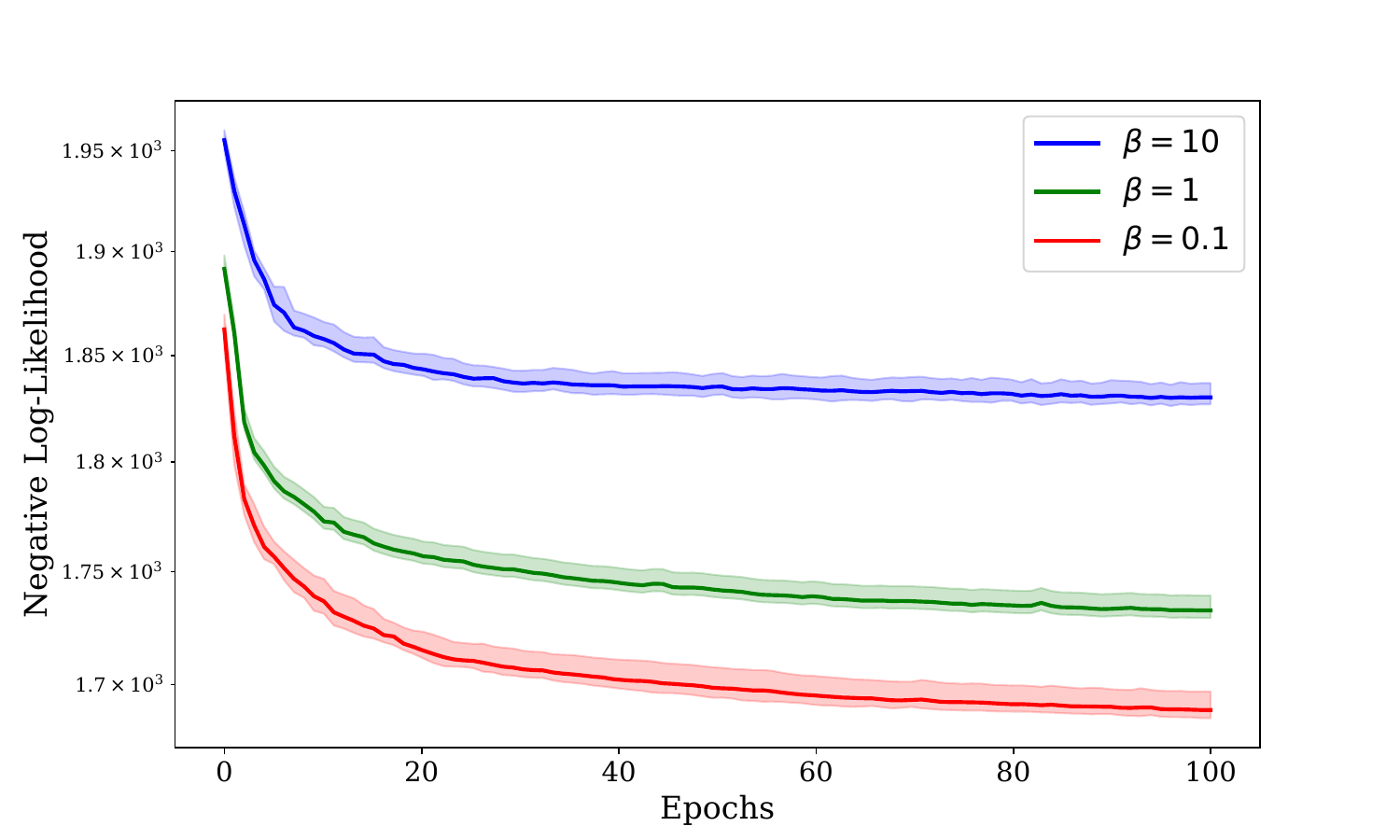}
    }
\caption{Squared norm of gradients (on the left) and Negative Log-Likelihood (on the right) for $\beta$-VAE trained on CIFAR-100 dataset using Adam. Bold lines represent the mean over 5 independent runs.}
\label{beta_CIFAR100}
\end{center}
\vskip -0.2in
\end{figure*}

\begin{figure*}[ht!]
\vskip -0.2in
\begin{center}
\centerline{
    \includegraphics[width=0.5\textwidth]
    {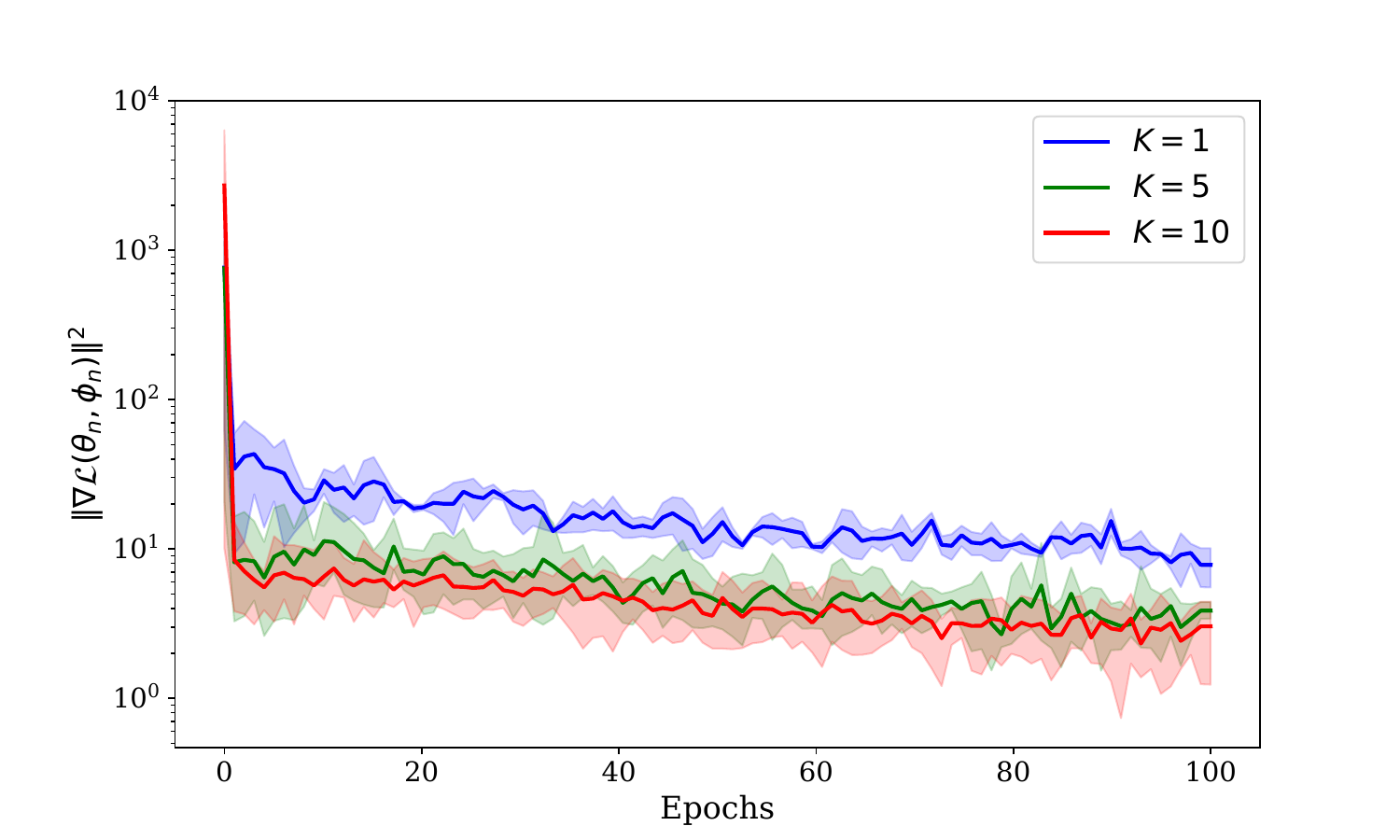}
    \includegraphics[width=0.5\textwidth]{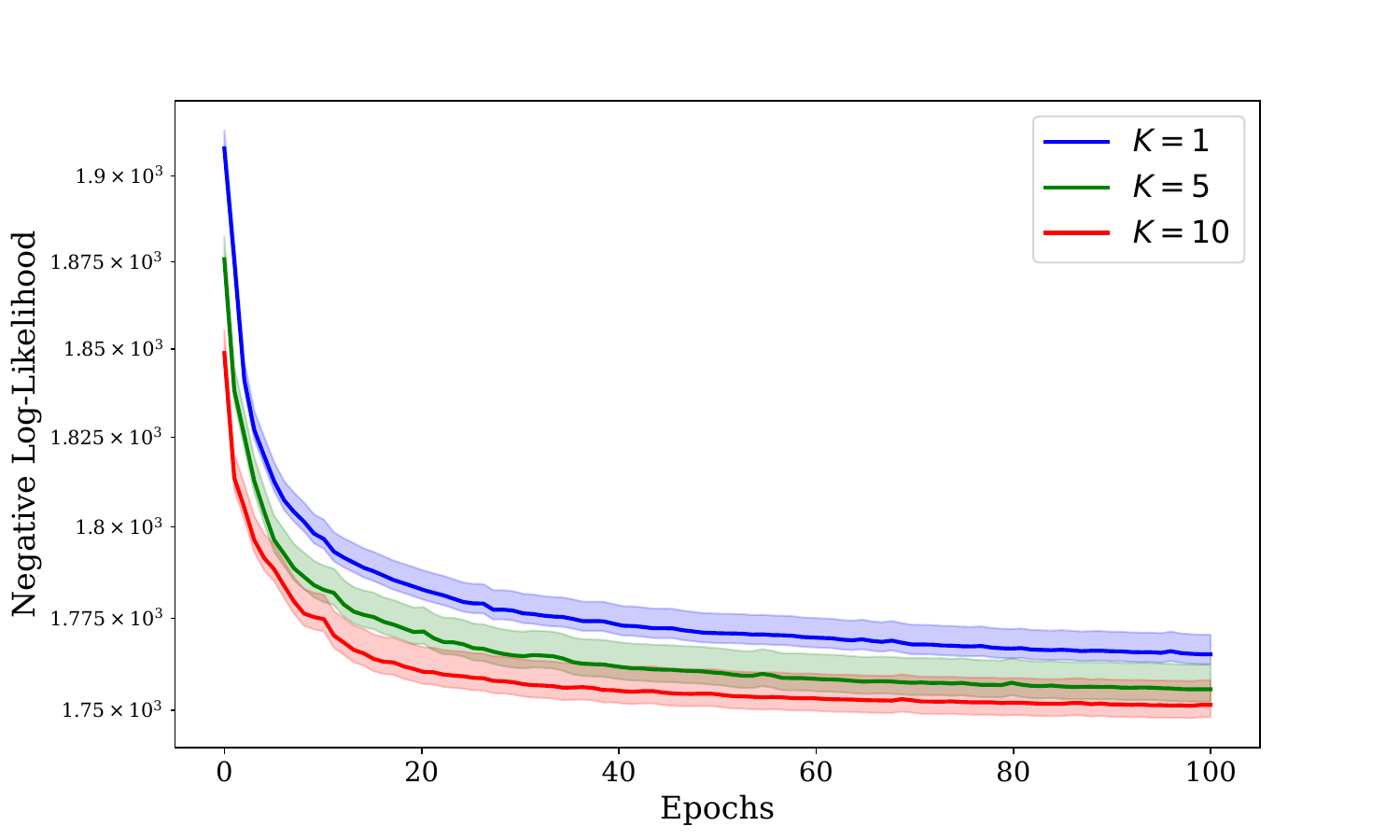}
    }
\caption{Squared norm of gradients (on the left) and Negative Log-Likelihood (on the right) for IWAE trained on CIFAR-100 dataset using Adam. Bold lines represent the mean over 5 independent runs.}
\label{iwae_CIFAR100}
\end{center}
\vskip -0.2in
\end{figure*}

The simulations in this paper were conducted using the NVIDIA RTX 6000 GPUs with 48GB of VRAM. The total computing hours required for the results presented in this paper are estimated to be around 100 to 200 hours of GPU usage.

\section{TECHNICAL LEMMAS}

\begin{lemma}\label{lem:grad_NN}
Let $G_{\theta}: z \mapsto\NN(z; \theta, f, N)$ denote a neural network with $N$ layers, where the parameters are $\theta = \{W_i, b_i\}_{i=1}^{N}$, and activation functions are $f = \{f_i\}_{i=1}^{N}$.
For all $\theta \in \Theta$ and $z \in \Zset$, the gradient of $G_{\theta}(z)$ with respect to $W_i$ for all $1 \leq i \leq N$, can be expressed as:
$$\nabla_{W_i} G_{\theta}(z) = \left( \prod_{j=i+1}^{N} f_j'(u_j) \cdot W_j \right) f_i'(u_i) f_{i-1}(u_{i-1})^\top \; \text{and} \; \nabla_{b_i} G_{\theta}(z) = \left( \prod_{j=i+1}^{N} f_j'(u_j) \cdot W_j \right) f_i'(u_i)\eqsp,
$$
where $f_0(z) = u_0 = z, u_1 = W_1 z + b_1$ and $u_j = W_{j} f_{j-1}(u_{j-1}) + b_{j}$ for all $2 \leq j \leq N$.
\end{lemma}

\begin{proof}
To derive the gradient of $G_{\theta}(z)$ with respect to $W_i$ and $b_i$, we use the chain rule.
First, for the final layer, i.e., $i=N$, the gradient of $G_{\theta}(z)$ with respect to $W_N$ and $b_N$ are given by:
$$\nabla_{W_N} G_{\theta}(z) = \frac{\partial G_{\theta}(z)}{\partial u_N} \cdot \frac{\partial u_N}{\partial W_N} \quad \text{and} \quad \nabla_{b_N} G_{\theta}(z) = \frac{\partial G_{\theta}(z)}{\partial u_N} \cdot \frac{\partial u_N}{\partial b_N}\eqsp.
$$
Note that $\nabla_{W_N} G_{\theta}(z) \in \rset^{d_{N} \times d_{N} \times d_{N-1}}$, where $d_i$ denotes the dimension of $u_i$.
For all $p, q, r \in \mathbb{N}$, we observe that  $(\nabla_{W_N} G_{\theta}(z))_{p,q,r} = 0$ if $p \neq q$. Thus, we can conventionally simplify this tensor to a matrix of size $\rset^{d_{N} \times d_{N-1}}$ by ignoring the second dimension.
Since $u_N = W_N f_{N-1}(u_{N-1}) + b_N$ and $G_{\theta}(z) = f_{N}(u_{N})$, we have:
$$\nabla_{W_N} G_{\theta}(z) =  f_N'(u_N) f_{N-1}(u_{N-1})^\top\eqsp,
$$
$$\nabla_{b_N} G_{\theta}(z) = \text{Diag}(f_N'(u_N)) \eqsp.
$$
For $i<N$, by recursively applying the chain rule, the gradient of $G_{\theta}(z)$ with respect to $W_i$ for all $1 \leq i \leq N$, can be expressed as:
$$\nabla_{W_i} G_{\theta}(z) = \left( \prod_{j=i+1}^{N} f_j'(u_j) \cdot W_j \right) f_i'(u_i) f_{i-1}(u_{i-1})^\top \; \text{and} \; \nabla_{b_i} G_{\theta}(z) = \left( \prod_{j=i+1}^{N} f_j'(u_j) \cdot W_j \right) f_i'(u_i)\eqsp,
$$
where $f_0(z) = u_0 = z, u_1 = W_1 z + b_1$ and $u_j = W_{j} f_{j-1}(u_{j-1}) + b_{j}$ for all $2 \leq j \leq N$.
\end{proof}

\begin{lemma}\label{lem:smooth_z_NN}
Let $G_{\theta}:z\mapsto \NN(z; \theta, f, N)$ denote a neural network with $N$ layers, where the parameters are $\theta = \{W_i, b_i\}_{i=1}^{N}$, and activation functions are $f = \{f_i\}_{i=1}^{N}$ such that $f_i \in \mathcal{F}_{\text{SL}}$ for all $1 \leq i \leq N$.
Let $M_{f_i}$ and $L_{f_i}$ represent the Lipschitz constant and the smoothness parameter of the activation function $f_i$ in the $i$-th layer, respectively.
Assume that there exists some constant $a$ such that for any $\theta \in \Theta$, $\lVert \theta \rVert \leq a$.  Then, for all $z, z' \in \Zset$, and  $\theta \in \Theta$:
\begin{align*}
\left\|\nabla_{z} G_{\theta}(z)\right\| &\leq a^N \prod_{j=1}^{N} M_{f_j}\eqsp,\\
\left\| \nabla_{z} G_{\theta}(z) - \nabla_{z} G_{\theta}(z') \right\| &\leq \sum_{k=1}^{N} L_{f_k} a^{N+k} \prod_{j=1}^{k-1} M_{f_j}^2 \prod_{j=k+1}^{N} M_{f_j} \left\| z - z'\right\|\eqsp.
\end{align*}
\end{lemma}

\begin{proof}
Similar to Lemma \ref{lem:grad_NN}, the gradient of $G_{\theta}(z)$ with respect to $z$ is defined as:
$$\nabla_z G_{\theta}(z) = \left( \prod_{j=1}^{N} f_j'(u_j) \cdot W_j \right),$$
where $u_1 = W_1 z + b_1$ and $u_j = W_{j} f_{j-1}(u_{j-1}) + b_{j}$ for all $2 \leq j \leq N$.
We will now show by induction on the number of layers $N$ that: 
$$\left\| \nabla_{z} G_{\theta}(z) - \nabla_{z} G_{\theta}(z') \right\| \leq \sum_{k=1}^{N} L_{f_k} a^{N+k} \prod_{j=1}^{k-1} M_{f_j}^2 \prod_{j=k+1}^{N} M_{f_j} \left\| z - z'\right\|\eqsp.
$$
For the base case, consider the neural network with a single layer, where $G_{\theta}(z) = f_1(W_1 z + b_1)$. The gradient with respect to $z$ is given by:
$$\nabla_{z} G_{\theta}(z) = f_1'(W_1 z + b_1) \cdot W_1\eqsp.$$
Note that
\begin{align*}
\left\| \nabla_{z} G_{\theta}(z) - \nabla_{z} G_{\theta}(z') \right\| &= \left\| f_1'(W_1 z + b_1) \cdot W_1 - f_1'(W_1 z' + b_1) \cdot W_1 \right\|, \\
&\leq \left\|W_1\right\| \left\| f_1'(W_1 z + b_1) - f_1'(W_1 z' + b_1) \right\|, \\
&\leq a^2 L_{f_1} \left\| z - z'\right\|\eqsp,
\end{align*}
which completes the base case.

Now, for the inductive step, assume that the Lipschitz constant holds for a network with $N-1$ layers. We  show that it also holds for a neural network with $N$ layer. Let $u$ and $u'$ be sequences defined as follows: $u_0 = z, u'_0 = z', u_1 = W_1 z' + b_1, u'_1 = W_1 z' + b_1, u_j = W_{j} f_{j-1}(u_{j-1}) + b_{j}$ and $u'_j = W_{j} f_{j-1}(u'_{j-1}) + b_{j}$ for all $2 \leq j \leq N$. We have:
\begin{align*}
\left\| \prod_{j=1}^{N} f_j'(u_{j}) \cdot W_j - \prod_{j=1}^{N} f_j'(u_{j}') \cdot W_j \right\|
&\leq A_1 + A_2\eqsp,
\end{align*}
where
\begin{align*}
A_1 &= \left\| f_N'(u_{N}) \cdot W_N \left( \prod_{j=1}^{N-1} f_j'(u_{j}) \cdot W_j \right) - f_N'(u_{N}') \cdot W_N \left( \prod_{j=1}^{N-1} f_j'(u_{j}) \cdot W_j \right) \right\|\eqsp, \\
A_2 &= \left\| f_N'(u_{N}') \cdot W_N \left( \prod_{j=1}^{N-1} f_j'(u_{j}) \cdot W_j \right) - f_N'(u_{N}') \cdot W_N \left( \prod_{j=1}^{N-1} f_j'(u_{j}') \cdot W_j \right) \right\|\eqsp.
\end{align*}
First, we have:
\begin{align*}
A_1 &= \left\| f_N'(u_{N}) \cdot W_N \left( \prod_{j=1}^{N-1} f_j'(u_{j}) \cdot W_j \right) - f_N'(u_{N}') \cdot W_N \left( \prod_{j=1}^{N-1} f_j'(u_{j}) \cdot W_j \right) \right\| \\
&\leq \left\|\prod_{j=1}^{N-1} f_j'(u_{j}) \cdot W_j\right\| \left\|W_N\right\| \left\| f_N'(u_{N}) - f_N'(u_{N}') \right\| \\
&\leq a^{N-1} a \prod_{j=1}^{N-1} M_{f_j} \left\| f_N'(u_{N}) - f_N'(u_{N}') \right\| \\
&\leq a^{N} \prod_{j=1}^{N-1} M_{f_j} a^{N} L_{f_N} \prod_{j=1}^{N-1} M_{f_j} \left\| z - z'\right\| \\
&\leq a^{2N} L_{f_N} \prod_{j=1}^{N-1} M_{f_j}^2 \left\| z - z'\right\|\eqsp,
\end{align*}
where we used the Lipschitz and smoothness conditions of the activation functions in the second last inequality.
For A2, using the induction hypothesis:
\begin{align*}
A_2 &= \left\| f_N'(u_{N}') \cdot W_N \left( \prod_{j=1}^{N-1} f_j'(u_{j}) \cdot W_j \right) - f_N'(u_{N}') \cdot W_N \left( \prod_{j=1}^{N-1} f_j'(u_{j}') \cdot W_j \right) \right\| \\
&\leq \left\| f_N'(u_{N}') \right\| \left\| W_N \right\| \left\| \prod_{j=1}^{N-1} f_j'(u_{j}) \cdot W_j - \prod_{j=1}^{N-1} f_j'(u_{j}') \cdot W_j \right\| \\ 
&\leq M_{f_N} a \sum_{k=1}^{N-1} L_{f_k} a^{N-1+k} \prod_{j=1}^{k-1} M_{f_j}^2 \prod_{j=k+1}^{N-1} M_{f_j} \left\| z - z'\right\| \\ 
&\leq \sum_{k=1}^{N-1} L_{f_k} a^{N+k} \prod_{j=1}^{k-1} M_{f_j}^2 \prod_{j=k+1}^{N} M_{f_j} \left\| z - z'\right\|\eqsp.
\end{align*}
By combining these two terms, we obtain:
$$\left\| \nabla_{z} G_{\theta}(z) - \nabla_{z} G_{\theta}(z') \right\| \leq \sum_{k=1}^{N} L_{f_k} a^{N+k} \prod_{j=1}^{k-1} M_{f_j}^2 \prod_{j=k+1}^{N} M_{f_j} \left\| z - z'\right\|\eqsp, 
$$
which concludes the proof of the smoothness condition.
For the boundedness of the gradient, we have:
$$
\left\|\nabla_{z} G_{\theta}(z)\right\| = \left\| \prod_{j=1}^{N} f_j'(u_{j}) \cdot W_j \right\| \leq a^N \prod_{j=1}^{N} M_{f_j}\eqsp,
$$
which completes the proof.
\end{proof}
Lemma \ref{lem:smooth_z_NN} establishes the boundedness and smoothness of a neural network with respect to its input. The following lemmas extends these properties to the parameters, proving their boundedness and smoothness. Specifically, the boundedness and smoothness with respect to $W_1$ are addressed in Lemma \ref{lem:smooth_NN_W1}, $W_i$ in Lemma \ref{lem:smooth_NN_Wj} and $\theta$ in Lemma \ref{lem:smooth_NN}.

\begin{lemma}\label{lem:smooth_NN_W1}
Let $G_{\theta}: z\mapsto \NN(z; \theta, f, N)$ denote a neural network with $N$ layers, where the parameters are $\theta = \{W_i, b_i\}_{i=1}^{N}$, and activation functions are $f = \{f_i\}_{i=1}^{N}$ such that $f_i \in \mathcal{F}_{\text{SL}}$ for all $1 \leq i \leq N$.
Let $M_{f_i}$ and $L_{f_i}$ represent the Lipschitz constant and the smoothness parameter of the activation function $f_i$ in the $i$-th layer, respectively.
Assume that there exists some constant $a$ such that for any $\theta \in \Theta$, $ \lVert \theta \rVert \leq a$.  
Then, for all $\theta, \theta' \in \Theta$, and $z \in \Zset$:
\begin{align*}
\left\|\nabla_{W_1} G_{\theta}(z)\right\| &\leq \left\|z\right\| a^{N-1} \prod_{j=1}^{N} M_{f_j}\eqsp,\\
\left\| \nabla_{W_1} G_{W_1}(z) - \nabla_{W_1} G_{W_1'}(z) \right\| &\leq \sum_{k=1}^{N} \left\|z\right\|^2 L_{f_k} a^{N-2+k} \prod_{j=1}^{k-1} M_{f_j}^2 \prod_{j=k+1}^{N} M_{f_j} \left\| W_1 - W_1'\right\|\eqsp.
\end{align*}
\end{lemma}

\begin{proof}
Using Lemma \ref{lem:grad_NN}, the gradient of $G_{\theta}(z)$ with respect to $W_1$ and $b_1$ are defined as:
$$\nabla_{W_1} G_{\theta}(z) = \left( \prod_{j=2}^{N} f_j'(u_{j}) \cdot W_i \right) f_1'(u_1) u_0^\top \quad \text{and} \quad \nabla_{b_1} G_{\theta}(z) = \left( \prod_{j=2}^{N} f_j'(u_j) \cdot W_j \right) f_1'(u_1)\eqsp,
$$
where $u_0 = z, u_1 = W_1 z + b_1$ and $u_j = W_{j} f_{j-1}(u_{j-1}) + b_{j}$ for all $2 \leq j \leq N$.
We now show by induction on the number of layers that the Lipschitz constant with respect to $W_1$ is $\sum_{k=1}^{N} \|z\|^2 L_{f_k} a^{N-2+k} \prod_{j=1}^{k-1} M_{f_j}^2 \prod_{j=k+1}^{N} M_{f_j}$. 
For the base case, consider the neural network with a single layer, where $G_{\theta}(z) = f_1(W_1 z + b_1)$. The gradient with respect to $W_1$ is given by:
$$\nabla_{W_1} G_{\theta}(z) = f_1'(W_1 z + b_1) z^\top\eqsp.$$
Note that
\begin{align*}
\left\| \nabla_{W_1} G_{W_1}(z) - \nabla_{W_1} G_{W_1'}(z) \right\| &= \left\| f_1'(W_1 z + b_1) z^\top - f_1'(W_1' z + b_1) z^\top \right\|, \\
&\leq \left\|z\right\| \left\| f_1'(W_1 z + b_1) - f_1'(W_1' z + b_1) \right\|, \\
&\leq \left\|z\right\|^2 L_{f_1} \left\| W_1 - W_1'\right\|\eqsp,
\end{align*}
which completes the base case.

Now, for the inductive step, assume that the Lipschitz constant holds for a network with $N-1$ layer. We show that it also holds for a neural network with $N$ layers. Let $u$ and $u'$ be sequences defined as follows: $u_0 = u'_0 = z, u_1 = W_1 z + b_1, u'_1 = W'_1 z + b_1, u_j = W_{j} f_{j-1}(u_{j-1}) + b_{j}$ and $u'_j = W_{j} f_{j-1}(u'_{j-1}) + b_{j}$ for all $2 \leq j \leq N$. We have: 
\begin{align*}
\left\| \nabla_{W_1} G_{W_1}(z) - \nabla_{W_1} G_{W_1'}(z) \right\|  &= \left\| \left( \prod_{j=2}^{N} f_j'(u_{j}) \cdot W_j \right) f_1'(u_1) z^\top - \left( \prod_{j=2}^{N} f_j'(u_{j}') \cdot W_j \right) f_1'(u_1') z^\top \right\| \\ 
&\leq A_1 + A_2\eqsp,
\end{align*}
where
\begin{align*}
A_1 &= \left\| f_N'(u_{N}) \cdot W_N \left( \prod_{j=2}^{N-1} f_j'(u_{j}) \cdot W_j \right) f_1'(u_1) z^\top - f_N'(u_{N}') \cdot W_N \left( \prod_{j=2}^{N-1} f_j'(u_{j}) \cdot W_j \right) f_1'(u_1) z^\top \right\|\eqsp, \\
A_2 &= \left\| f_N'(u_{N}') \cdot W_N \left( \prod_{j=2}^{N-1} f_j'(u_{j}) \cdot W_j \right) f_1'(u_1) z^\top - f_N'(u_{N}') \cdot W_N \left( \prod_{j=2}^{N-1} f_j'(u_{j}') \cdot W_j \right) f_1'(u_1') z^\top \right\|\eqsp.
\end{align*}
First, we have:
\begin{align*}
A_1 &= \left\| f_N'(u_{N}) \cdot W_N \left( \prod_{j=2}^{N-1} f_j'(u_{j}) \cdot W_j \right) f_1'(u_1) z^\top - f_N'(u_{N}') \cdot W_N \left( \prod_{j=2}^{N-1} f_j'(u_{j}) \cdot W_j \right) f_1'(u_1) z^\top \right\| \\
&\leq \left\|\left( \prod_{j=2}^{N-1} f_j'(u_{j}) \cdot W_j \right) f_1'(u_1) z^\top\right\| \left\|W_N\right\| \left\| f_N'(u_{N}) - f_N'(u_{N}') \right\| \\
&\leq a^{N-2} a \prod_{j=1}^{N-1} M_{f_j} \left\|z\right\| \left\| f_N'(u_{N}) - f_N'(u_{N}') \right\| \\
&\leq a^{N-1} \prod_{j=1}^{N-1} M_{f_j}^2 \left\|z\right\|^2 a^{N-1} L_{f_N} \left\| W_1 - W_1'\right\| \\
&\leq \left\|z\right\|^2 a^{2(N-1)} L_{f_N} \prod_{j=1}^{N-1} M_{f_j}^2 \left\| W_1 - W_1'\right\|\eqsp,  
\end{align*}
where we used the Lipschitz and smoothness conditions of the activation functions in the second last inequality.
For A2, using the induction hypothesis, 
\begin{align*}
A_2 &= \left\| f_N'(u_{N}') \cdot W_N \left( \prod_{j=2}^{N-1} f_j'(u_{j}) \cdot W_j \right) f_1'(u_1) z^\top - f_N'(u_{N}') \cdot W_N \left( \prod_{j=2}^{N-1} f_j'(u_{j}') \cdot W_j \right) f_1'(u_1') z^\top \right\| \\
&\leq \left\| f_N'(u_{N}') \right\| \left\| W_N \right\| \left\| \left( \prod_{j=2}^{N-1} f_j'(u_{j}) \cdot W_j \right) f_1'(u_1) z^\top - \left( \prod_{j=2}^{N-1} f_j'(u_{j}') \cdot W_j \right) f_1'(u_1') z^\top \right\| \\ 
&\leq M_{f_N} a \left\|z\right\|^2 \sum_{k=1}^{N-1} L_{f_k} a^{N-3+k} \prod_{j=1}^{k-1} M_{f_j}^2 \prod_{j=k+1}^{N} M_{f_j} \left\| W_1 - W_1'\right\| \\ 
&\leq \left\|z\right\|^2 \sum_{k=1}^{N-1} L_{f_k} a^{N-2+k} \prod_{j=1}^{k-1} M_{f_j}^2 \prod_{j=k+1}^{N} M_{f_j} \left\| W_1 - W_1'\right\|\eqsp.
\end{align*}
By combining these two terms, we obtain:
$$\left\| \nabla_{W_1} G_{W_1}(z) - \nabla_{W_1} G_{W_1'}(z) \right\| \leq \left\|z\right\|^2 \sum_{k=1}^{N} L_{f_k} a^{N-2+k} \prod_{j=1}^{k-1} M_{f_j}^2 \prod_{j=k+1}^{N} M_{f_j} \left\| W_1 - W_1'\right\|\eqsp, \\ 
$$
which concludes the proof of the smoothness condition.
For the boundedness of the gradient, we have:
$$
\left\|\nabla_{W_1} G_{\theta}(z)\right\| = \left\|\left( \prod_{j=2}^{N} f_j'(u_{j}) \cdot W_j \right) f_1'(u_1) z^\top\right\| \leq \left\|z\right\| a^{N-1} \prod_{j=1}^{N} M_{f_j}\eqsp,
$$
which completes the proof.
\end{proof}

\begin{lemma}\label{lem:smooth_NN_Wj}
Let $G_{\theta}:z\mapsto \NN(z; \theta, f, N)$ denote a neural network with $N$ layers, where the parameters are $\theta = \{W_i, b_i\}_{i=1}^{N}$, and activation functions are $f = \{f_i\}_{i=1}^{N}$ such that $f_i \in \mathcal{F}_{\text{SL}}$ for all $1 \leq i \leq N$.
Let $M_{f_i}$ and $L_{f_i}$ represent the Lipschitz constant and the smoothness parameter of the activation function $f_i$ in the $i$-th layer, respectively.
Assume that there exists some constant $a$ such that for any $\theta \in \Theta$, $\lVert \theta \rVert \leq a$. 
Then, for all $\theta, \theta' \in \Theta$, and $z \in \Zset$:
\begin{align*}
\left\|\nabla_{W_i} G_{\theta}(z)\right\| &\leq \left\|z\right\| a^{N-1} \prod_{j=1}^{N} M_{f_j}\eqsp,\\
\left\| \nabla_{W_i} G_{W_i}(z) - \nabla_{W_i} G_{W_i'}(z) \right\| &\leq \sum_{k=i}^{N} \left\|z\right\|^2 L_{f_k} a^{N-2i+k} \prod_{j=i}^{k-1} M_{f_j}^2 \prod_{j=k+1}^{N} M_{f_j} \left\| W_i - W_i'\right\|\eqsp.
\end{align*}
\end{lemma}

\begin{proof}
For the boundedness of the gradient of $G_{\theta}(z)$ with respect to $W_i$, we have:
\begin{align*}
\left\|\nabla_{W_i} G_{\theta}(z)\right\| \leq \left\|\left( \prod_{j=i+1}^{N} f_j'(u_j) \cdot W_j \right) f_i'(u_i) f_{i-1}(u_{i-1})^\top\right\| &\leq \left\|f_{i-1}(u_{i-1})\right\| a^{N-i} \prod_{j=i}^{N} M_{f_j} \\
&\leq \left\|z\right\| a^{N-1} \prod_{j=1}^{N} M_{f_j}\eqsp,
\end{align*}
where we used the fact that $\left\|f_{i-1}(u_{i-1})\right\| \leq \left\|z\right\| a^{i-1} \prod_{j=1}^{i-1} M_{f_j}$. For the smoothness condition, using a similar argument as in Lemma \ref{lem:smooth_NN_W1}, we define the sequences $u$ and $u'$ as follows: $u_0 = u'_0 = z, u_1 = W_1 z + b_1, u'_1 = W'_1 z + b_1, u_j = W_{j} f_{j-1}(u_{j-1}) + b_{j}$ and $u'_j = W'_{j} f_{j-1}(u'_{j-1}) + b_{j}$ for all $2 \leq j \leq N$. Thus, we obtain:
\begin{align*}
\left\| \nabla_{W_i} G_{W_i}(z) - \nabla_{W_i} G_{W_i'}(z) \right\|
&\leq \left\|\left( \prod_{j=i+1}^{N} f_j'(u_j) \cdot W_j \right) f_i'(u_i) f_{i-1}(u_{i-1})^\top - \left( \prod_{j=i+1}^{N} f_j'(u_j') \cdot W_j \right) f_i'(u_i') f_{i-1}(u_{i-1})^\top\right\| \\
&\leq \left\|f_{i-1}(u_{i-1})\right\|^2 \sum_{k=i}^{N} L_{f_k} a^{N-2i+k} \prod_{j=i}^{k-1} M_{f_j}^2 \prod_{j=k+1}^{N} M_{f_j} \left\| W_i - W_i'\right\| \\
&\leq \sum_{k=i}^{N} L_{f_k} a^{N-2+k} \prod_{j=1}^{k-1} M_{f_j}^2 \prod_{j=k+1}^{N} M_{f_j} \left\| W_i - W_i'\right\|\eqsp.
\end{align*}
which concludes the proof.
\end{proof}

\begin{lemma}\label{lem:smooth_NN}
Let $G_{\theta}:z\mapsto \NN(z; \theta, f, N)$ denote a neural network with $N$ layers, where the parameters are $\theta = \{W_i, b_i\}_{i=1}^{N}$, and activation functions are $f = \{f_i\}_{i=1}^{N}$ such that $f_i \in \mathcal{F}_{\text{SL}}$ for all $1 \leq i \leq N$.
Let $M_{f_i}$ and $L_{f_i}$ represent the Lipschitz constant and the smoothness parameter of the activation function $f_i$ in the $i$-th layer, respectively.
Assume that there exists some constant $a$ such that for any $\theta \in \Theta$, $\lVert \theta \rVert \leq a$. 
Then, for all $\theta, \theta' \in \Theta$, and $z \in \Zset$:
\begin{align*}
\left\|\nabla_{\theta} G_{\theta}(z)\right\| &\leq \left(\|z\| + 1\right) a^{N-1} \prod_{j=1}^{N} M_{f_j}\eqsp,\\
\left\| \nabla_{\theta} G_{\theta}(z) - \nabla_{\theta} G_{\theta'}(z) \right\| &\leq N \left(\|z\|^2 + 1\right) \sum_{k=1}^{N} L_{f_k} a^{N-2+k} \prod_{j=1}^{k-1} M_{f_j}^2 \prod_{j=k+1}^{N} M_{f_j} \left\| \theta - \theta'\right\|\eqsp.
\end{align*}
\end{lemma}

\begin{proof}
For the gradient of $G_{\theta}(z)$ with respect to $\theta$, we have:
\begin{align*}
\left\|\nabla_{\theta} G_{\theta}(z)\right\| &= \max_{1 \leq i \leq N}\left\{ \left\| \nabla_{W_i} G_{\theta}(z) \right\|, \left\| \nabla_{b_i} G_{\theta}(z) \right\| \right\} \\
&\leq \max_{1 \leq i \leq N} \left\| \nabla_{W_i} G_{\theta}(z) \right\| + \max_{1 \leq i \leq N} \left\| \nabla_{b_i} G_{\theta}(z) \right\| \\
&\leq \|z\| a^{N-1} \prod_{j=1}^{N} M_{f_j} + a^{N-1} \prod_{j=1}^{N} M_{f_j}\eqsp,
\end{align*}
where we used Lemma \ref{lem:smooth_NN_Wj}.
Using Lemma \ref{lem:smooth_NN_Wj}, we obtain the following smoothness condition:
\begin{align*}
\left\| \nabla_{\theta} G_{\theta}(z) - \nabla_{\theta} G_{\theta'}(z) \right\| &= \max_{1 \leq i \leq N} \left\| \nabla_{W_i} G_{\theta}(z) - \nabla_{W_i} G_{\theta'}(z) \right\| + \max_{1 \leq i \leq N} \left\| \nabla_{b_i} G_{\theta}(z) - \nabla_{b_i} G_{\theta'}(z) \right\|\eqsp.
\end{align*}
For each weight $W_i$, we have:
\begin{align*}
\left\| \nabla_{W_i} G_{\theta}(z) - \nabla_{W_i} G_{\theta'}(z) \right\|
&\leq \sum_{j=1}^{N} \left\| \nabla_{W_i} G_{W_j}(z) - \nabla_{W_i} G_{W_j'}(z) \right\| \\
&\leq \sum_{j=1}^{N} L_{W_j}\left\| W_j - W_j' \right\| \\
&\leq L_{\text{max}} \sum_{j=1}^{N} \left\| W_j - W_j' \right\| \\
&\leq L_{\text{max}} N \left\| \theta - \theta' \right\|\eqsp,
\end{align*}
where $L_{\text{max}} = \|z\|^2 \max_{1 \leq i \leq N} \left\{ \sum_{k=i}^{N} L_{f_k} a^{N-2+k} \prod_{j=1}^{k-1} M_{f_j}^2 \prod_{j=k+1}^{N} M_{f_j} \right\}$.
Similarly, for the bias terms $b_i$, we can use the same reasoning, obtaining an analogous bound. Thus, combining both the weight and bias terms, and noting that all terms in the sum for $L_{\text{max}}$ are positive, we conclude:
$$\left\| \nabla_{\theta} G_{\theta}(z) - \nabla_{\theta} G_{\theta'}(z) \right\| \leq N \left(\|z\|^2 + 1\right) \sum_{k=1}^{N} \left\|z\right\|^2 L_{f_k} a^{N-2+k} \prod_{j=1}^{k-1} M_{f_j}^2 \prod_{j=k+1}^{N} M_{f_j} \left\| \theta - \theta'\right\|\eqsp.
$$
\end{proof}

\begin{lemma}\label{lem:grad_z_smooth_NN}
Let $G_{\theta}:z\mapsto \NN(z; \theta, f, N)$ denote a neural network with $N$ layers, where the parameters are $\theta = \{W_i, b_i\}_{i=1}^{N}$, and activation functions are $f = \{f_i\}_{i=1}^{N}$ such that $f_i \in \mathcal{F}_{\text{SL}}$ for all $1 \leq i \leq N$.
Let $M_{f_i}$ and $L_{f_i}$ represent the Lipschitz constant and the smoothness parameter of the activation function $f_i$ in the $i$-th layer, respectively.
Assume that there exists some constant $a$ such that for any $\theta \in \Theta$, $\lVert \theta \rVert \leq a$. Then, for all $1 \leq i \leq N$, $W_i, W_i' \in \Theta$, and $z \in \Zset$,
$$\left\| \nabla_{z} G_{W_i}(z) - \nabla_{z} G_{W_i'}(z) \right\| \leq \prod_{j=1}^{i-1} M_{f_j} \left(a^{N-1} \prod_{j=i}^{N} M_{f_j} + \sum_{k=i}^{N} \left\|z\right\| L_{f_k} a^{N-i+k} \prod_{j=i}^{k-1} M_{f_j}^2 \prod_{j=k+1}^{N} M_{f_j}\right) \left\| W_i - W_i'\right\|\eqsp.
$$
\end{lemma}

\begin{proof}
The gradient of $G_{\theta}(z)$ with respect to $z$ is defined as:
$$\nabla_z G_{\theta}(z) = \left( \prod_{j=1}^{N} f_j'(u_j) \cdot W_j \right),$$
where $u_1 = W_1 z + b_1$ and $u_j = W_{j} f_{j-1}(u_{j-1}) + b_{j}$ for all $2 \leq j \leq N$. 
Let $u$ and $u'$ be sequences defined as follows: $u_0 = u'_0 = z, u_i = W_i z + b_i, u'_i = W'_i z + b_i$ and $u_j = u'_j = W_{j} f_{j-1}(u_{j-1}) + b_{j}$ for all $j \neq i$. We have: 
We have:
\begin{align*}
\left\| \nabla_{z} G_{W_i}(z) - \nabla_{z} G_{W_i'}(z) \right\| &= \left\| \prod_{j=1}^{N} f_j'(u_j) \cdot W_j - \prod_{j=1}^{N} f_j'(u_j') \cdot W_j' \right\| \\
&\leq \left\| \left(\prod_{j=1}^{i-1} f_j'(u_j) \cdot W_i\right)\left(\prod_{j=i}^{N} f_j'(u_j) \cdot W_j\right) - \left(\prod_{j=1}^{i-1} f_j'(u_j) \cdot W_j\right)\left(\prod_{j=i}^{N} f_j'(u_j') \cdot W_j'\right) \right\| \\
&\leq \left\| \prod_{j=1}^{i-1} f_j'(u_j) \cdot W_j \right\| \left\|\prod_{j=i}^{N} f_j'(u_j) \cdot W_j - \prod_{j=i}^{N} f_j'(u_j') \cdot W_j' \right\| \\
&\leq a^{i-1} \prod_{j=1}^{i-1} M_{f_j} \left\|\prod_{j=i}^{N} f_j'(u_j) W_j - \prod_{j=i}^{N} f_j'(u_j') W_j' \right\|\eqsp.
\end{align*}
We now show by induction on the number of layers that:
$$
\left\|\prod_{j=i}^{N} f_j'(u_j) \cdot W_j - \prod_{j=i}^{N} f_j'(u_j') \cdot W_j' \right\| \leq \left(a^{N-i} \prod_{j=i}^{N} M_{f_j} + \sum_{k=i}^{N} \left\|z\right\| L_{f_k} a^{N-2i+k+1} \prod_{j=i}^{k-1} M_{f_j}^2 \prod_{j=k+1}^{N} M_{f_j}\right) \left\| W_i - W_i'\right\|\eqsp.
$$
For the base case, consider the neural network with a single layer, where $G_{\theta}(z) = f_1(W_1 z + b_1)$. The gradient with respect to $z$ is given by:
$$\nabla_{z} G_{\theta}(z) = f_1'(W_1 z + b_1) \cdot W_1\eqsp.$$
We have:
\begin{align*}
\left\| \nabla_{z} G_{W_1}(z) - \nabla_{z} G_{W_1'}(z) \right\| &= \left\| f_1'(W_1 z + b_1) \cdot W_1 - f_1'(W_1' z + b_1) \cdot W_1' \right\| \\
&\leq \left\| f_1'(W_1 z + b_1) \cdot W_1 - f_1'(W_1 z + b_1) \cdot W_1' \right\| + \left\| f_1'(W_1 z + b_1) \cdot W_1' - f_1'(W_1' z + b_1) \cdot W_1' \right\| \\
&\leq M_{f_1} \left\| W_1 - W_1' \right\| + \left\| z \right\| a L_{f_1} \left\| W_1 - W_1' \right\|\eqsp,
\end{align*}
which completes the base case.

Now, for the inductive step, assume that the Lipschitz constant holds for a network with $N-1$ layers. We show that it also holds for a neural network with $N$ layers.
Let $u$ and $u'$ be sequences defined as follows: $u_0 = u'_0 = z, u_i = W_i z + b_i, u'_i = W'_i z + b_i$ and $u_j = u'_j = W_{j} f_{j-1}(u_{j-1}) + b_{j}$ for all $j \neq i$. We have:
\begin{equation*}
\left\| \prod_{j=i}^{N} f_j'(u_{j}) \cdot W_j - \prod_{j=i}^{N} f_j'(u_{j}') \cdot W_j' \right\|
\leq A_1 + A_2\eqsp,
\end{equation*}
where
\begin{align*}
A_1 &= \left\| f_N'(u_{N}) \cdot W_N \left( \prod_{j=i}^{N-1} f_j'(u_{j}) \cdot W_j \right) - f_N'(u_{N}') \cdot W_N \left( \prod_{j=i}^{N-1} f_j'(u_{j})\cdot  W_j \right) \right\|\eqsp, \\
A_2 &= \left\| f_N'(u_{N}') \cdot W_N \left( \prod_{j=i}^{N-1} f_j'(u_{j}) \cdot W_j \right) - f_N'(u_{N}') \cdot W_N \left( \prod_{j=i}^{N-1} f_j'(u_{j}') \cdot W_j' \right) \right\|\eqsp.
\end{align*}
First, we have:
\begin{align*}
A_1 &= \left\| f_N'(u_{N}) \cdot W_N \left( \prod_{j=i}^{N-1} f_j'(u_{j}) \cdot W_j \right) - f_N'(u_{N}') \cdot W_N \left( \prod_{j=i}^{N-1} f_j'(u_{j}) \cdot W_j \right) \right\| \\
&\leq \left\|\prod_{j=i}^{N-1} f_j'(u_{j}) \cdot W_j\right\| \left\|W_N\right\| \left\| f_N'(u_{N}) - f_N'(u_{N}') \right\| \\
&\leq a^{N-j} a \prod_{j=i}^{N-1} M_{f_j} \left\| f_N'(u_{N}) - f_N'(u_{N}') \right\| \\
&\leq a^{N-i+1} \prod_{j=i}^{N-1} M_{f_j} \left\|z\right\| a^{N-i} L_{f_N} \prod_{j=i}^{N-1} M_{f_j} \left\| W_i - W_i'\right\| \\
&\leq a^{2(N-i)+1} \left\|z\right\| L_{f_N} \prod_{j=i}^{N-1} M_{f_j}^2 \left\| W_i - W_i'\right\|\eqsp,
\end{align*}
where we used the Lipschitz and smoothness conditions of the activation functions in the second last inequality.
For A2, using the induction hypothesis:
\begin{align*}
A_2 &= \left\| f_N'(u_{N}') \cdot W_N \left( \prod_{j=i}^{N-1} f_j'(u_{j}) \cdot W_j \right) - f_N'(u_{N}') \cdot W_N \left( \prod_{j=i}^{N-1} f_j'(u_{j}') \cdot W_j' \right) \right\| \\
&\leq \left\| f_N'(u_{N}') \right\| \left\| W_N \right\| \left\| \prod_{j=i}^{N-1} f_j'(u_{j}) \cdot W_j - \prod_{j=i}^{N-1} f_j'(u_{j}') \cdot W_j' \right\| \\ 
&\leq M_{f_N} a \left(a^{N-1-i} \prod_{j=i}^{N-1} M_{f_j} + \sum_{k=i}^{N-1} \left\|z\right\| L_{f_k} a^{N-2i+k} \prod_{j=i}^{k-1} M_{f_j}^2 \prod_{j=k+1}^{N-1} M_{f_j}\right) \left\| W_i - W_i'\right\| \\ 
&\leq \left(a^{N-i} \prod_{j=i}^{N} M_{f_j} + \sum_{k=i}^{N-1} \left\|z\right\| L_{f_k} a^{N-2i+k+1} \prod_{j=i}^{k-1} M_{f_j}^2 \prod_{j=k+1}^{N} M_{f_j}\right) \left\| W_i - W_i'\right\|\eqsp.
\end{align*}
By combining these two terms, we obtain:
$$
\left\|\prod_{j=i}^{N} f_j'(u_j) \cdot W_j - \prod_{j=i}^{N} f_j'(u_j') \cdot W_j' \right\| \leq \left(a^{N-i} \prod_{j=i}^{N} M_{f_j} + \sum_{k=i}^{N} \left\|z\right\| L_{f_k} a^{N-2i+k+1} \prod_{j=i}^{k-1} M_{f_j}^2 \prod_{j=k+1}^{N} M_{f_j}\right) \left\| W_i - W_i'\right\|\eqsp,
$$
which concludes the proof.
\end{proof}

\begin{lemma}\label{lemma:tech_eq}
Assume that there exists a constant $c_{\Sigma} > 0$ such that for all $\phi \in \Phi$ and $x \in \Xset$, $\lambda_{\min}\left(\Sigma_{\phi}(x)\right) \geq c_{\Sigma}$.
Then, for all $\phi, \phi' \in \Phi$, and $x \in \Xset$:
$$\left\| \Sigma_{\phi}(x)^{-1/2} - \Sigma_{\phi'}(x)^{-1/2} \right\|
\leq c_{\Sigma}^{-3/2} \left\| \Sigma_{\phi}(x) - \Sigma_{\phi'}(x) \right\|\eqsp.$$
\end{lemma}

\begin{proof}
Using the mean value theorem for matrix functions, we get:
$$\Sigma_{\phi}(x)^{-1/2} - \Sigma_{\phi'}(x)^{-1/2} 
= \int_0^1 \left( (1-t)\Sigma_{\phi}(x) + t\Sigma_{\phi'}(x) \right)^{-3/2} 
\left( \Sigma_{\phi'}(x) - \Sigma_{\phi}(x) \right) \, \rmd t\eqsp.$$
Therefore,
$$\left\| \Sigma_{\phi}(x)^{-1/2} - \Sigma_{\phi'}(x)^{-1/2} \right\|
\leq \left\| \Sigma_{\phi}(x) - \Sigma_{\phi'}(x) \right\| \int_0^1 \left\| \left( (1-t)\Sigma_{\phi}(x) + t\Sigma_{\phi'}(x) \right)^{-3/2} \right\| \, \rmd t\eqsp.$$
Since $\lambda_{\min}\left(\Sigma_{\phi}(x)\right) \geq c_{\Sigma}$, for all $t \in [0,1]$, it follows that:
$$\lambda_{\min}\left( (1-t)\Sigma_{\phi}(x) + t\Sigma_{\phi'}(x) \right) 
\geq (1-t)\lambda_{\min}\left( \Sigma_{\phi}(x) \right) + t\lambda_{\min}\left( \Sigma_{\phi'}(x) \right) 
\geq c_{\Sigma}\eqsp.$$
Then,
$$\left\| \left( (1-t)\Sigma_{\phi}(x) + t\Sigma_{\phi'}(x) \right)^{-3/2} \right\| 
\leq c_{\Sigma}^{-3/2}\eqsp,$$
and
$$\left\| \Sigma_{\phi}(x)^{-1/2} - \Sigma_{\phi'}(x)^{-1/2} \right\|
\leq c_{\Sigma}^{-3/2} \left\| \Sigma_{\phi}(x) - \Sigma_{\phi'}(x) \right\|\eqsp.$$
\end{proof}

\end{document}